\documentclass{article}

\usepackage{microtype}
\usepackage{graphicx}
\usepackage{algorithmic}
\usepackage{booktabs} 
\usepackage{etex}
\usepackage{amsmath}
\usepackage{amsthm}
\usepackage{amssymb} 
\usepackage{color, colortbl}
\usepackage{xcolor}
\usepackage{hyperref}
\usepackage{cleveref}

\usepackage{bm}
\usepackage{bbm}
\usepackage{graphicx}
\usepackage{thmtools}
\usepackage{thm-restate}
\usepackage{mathtools}
\usepackage[skip=0pt]{caption}
\usepackage[skip=0pt]{subcaption}
\usepackage{booktabs} 
\usepackage{fontawesome}
\usepackage[algo2e,ruled,vlined, linesnumbered]{algorithm2e}

\definecolor{red}{HTML}{E51400}  
\definecolor{blue}{HTML}{0050EF} 
\definecolor{green}{HTML}{008A00} 
\definecolor{purple}{HTML}{AA00FF} 
\definecolor{dark-red}{rgb}{0.4, 0.15, 0.15}
\definecolor{dark-blue}{rgb}{0.15, 0.15, 0.4}
\definecolor{medium-red}{rgb}{0.5, 0, 0}
\definecolor{medium-blue}{rgb}{0, 0, 0.5}
\definecolor{light-red}{rgb}{0.7, 0, 0}
\definecolor{light-blue}{rgb}{0, 0, 0.7}

\hypersetup{
    colorlinks=true,
    linkcolor=blue,
    filecolor=magenta,      
    urlcolor=blue,
}

\newtheorem{theorem}{\bf Theorem}
\newtheorem{lemma}{\bf Lemma}
\newtheorem{fact}{\bf Fact}
\newtheorem{condition}{\bf Condition}

\newtheorem{definition}{\bf Definition}

\definecolor{red}{HTML}{E51400} 
\definecolor{blue}{HTML}{0050EF} 
\definecolor{green}{HTML}{008A00} 
\definecolor{purple}{HTML}{AA00FF} 
\definecolor{orange}{HTML}{FF7F00}
\definecolor{gray}{HTML}{848482}

\newcommand{\mulane}{{MuLaNE}}

\DeclareMathOperator*{\argmax}{arg\,max}

\newcommand{\norm}[1]{\left\lVert#1\right\rVert}
\newcommand{\prts}[1]{\left(#1\right)}
\newcommand{\brkt}[1]{\left[#1\right]}

\newcommand{\cS}{\mathcal{S}}
\newcommand{\abs}[1]{\left| #1 \right|}

\newcommand{\pl}{\left(}
\newcommand{\pr}{\right)}

\newcommand{\R}{\mathbb{R}}
\newcommand{\Z}{\mathbb{Z}}
\newcommand{\E}{\mathbb{E}}

\newcommand{\ba}{\boldsymbol{a}}
\newcommand{\bb}{\boldsymbol{b}}

\newcommand{\bch}{\boldsymbol{\chi}}
\newcommand{\bc}{\boldsymbol{c}}

\newcommand{\bk}{\boldsymbol{k}}

\newcommand{\bP}{\boldsymbol{P}}

\newcommand{\bx}{\boldsymbol{x}}
\newcommand{\by}{\boldsymbol{y}}

\newcommand{\cA}{\mathcal{A}}
\newcommand{\cB}{\mathcal{B}}

\newcommand{\cE}{\mathcal{E}}

\newcommand{\cG}{\mathcal{G}}

\newcommand{\cL}{\mathcal{L}}

\newcommand{\cN}{\mathcal{N}}

\newcommand{\cV}{\mathcal{V}}

\newcommand{\cQ}{{\mathcal{Q}}}
\newcommand{\cW}{{\mathcal{W}}}

\newcommand{\bbZ}{\mathbb{Z}}
\newcommand{\bbR}{\mathbb{R}}

\newcommand{\defeq}{\vcentcolon=}

\newcommand{\bsig}{\bm{\sigma}}

\newcommand{\compilefullversion}{true}
\ifthenelse{\equal{\compilefullversion}{false}}{%
	\newcommand{\OnlyInFull}[1]{}
	\newcommand{\OnlyInShort}[1]{#1}
}{%
	\newcommand{\OnlyInFull}[1]{#1}%
	\newcommand{\OnlyInShort}[1]{}%
}%

\newcommand{\compilehidecomments}{false}
\ifthenelse{ \equal{\compilehidecomments}{true} }{%
	\newcommand{\wei}[1]{}
	\newcommand{\xutong}[1]{}
	\newcommand{\jinhang}[1]{}
}{
\newcommand{\wei}[1]{{\color{blue}{\small{\bf [Wei: #1]}}}}
\newcommand{\xutong}[1]{{\color{green} [#1]}}
\newcommand{\jinhang}[1]{{\color{orange} [\text{Jinhang:} #1]}}
}

\makeatletter
\newcommand{\algrule}{\par\vskip.2\baselineskip{\color{black!30}\hrule}\par\vskip.2\baselineskip}
\makeatother
\renewcommand{\algorithmiccomment}[1]{\hfill$\triangleright$\textit{\textcolor{medium-blue}{#1}}}

\usepackage[accepted]{icml2021}

\icmltitlerunning{Multi-layered Network Exploration via Random Walks}

\begin{document}

\twocolumn[
\icmltitle{Multi-layered Network Exploration via Random Walks: \\
           From Offline Optimization to Online Learning}





\begin{icmlauthorlist}
\icmlauthor{Xutong Liu}{cuhk}
\icmlauthor{Jinhang Zuo}{cmu}
\icmlauthor{Xiaowei Chen}{bytedance}
\icmlauthor{Wei Chen}{msra}
\icmlauthor{John C.S. Lui}{cuhk}
\end{icmlauthorlist}

\icmlaffiliation{cuhk}{Department of Computer Science and Engineering, The Chinese University of Hong Kong, Hong Kong SAR, China}
\icmlaffiliation{cmu}{Department of Electrical and Computer Engineering, Carnegie Mellon University, Pittsburgh, PA, USA}
\icmlaffiliation{bytedance}{Bytedance, Mountain View, CA, USA}
\icmlaffiliation{msra}{Microsoft Research, Beijing, China}
\icmlcorrespondingauthor{Xutong Liu}{liuxt@cse.cuhk.edu.hk}
\icmlcorrespondingauthor{Wei Chen}{weic@microsoft.com}
\icmlcorrespondingauthor{John C.S. Lui}{cslui@cse.cuhk.edu.hk}
\icmlkeywords{Machine Learning, ICML}

\vskip 0.3in
]



\printAffiliationsAndNotice{}  

\begin{abstract}
Multi-layered network exploration (MuLaNE) problem is an important problem abstracted from many applications. In MuLaNE, there are multiple network layers where each node has an importance weight and each layer is explored by a random walk. The MuLaNE task is to allocate total random walk budget $B$ into each network layer so that the total weights of the unique nodes visited by random walks are maximized. We systematically study this problem from offline optimization to online learning. For the offline optimization setting where the network structure and node weights are known, we provide greedy based constant-ratio approximation algorithms for overlapping networks, and greedy or dynamic-programming based optimal solutions for non-overlapping networks. For the online learning setting, neither the network structure nor the node weights are known initially. We adapt the combinatorial multi-armed bandit framework and design algorithms to learn random walk related parameters and node weights while optimizing the budget allocation in multiple rounds, and prove that they achieve logarithmic regret bounds. Finally, we conduct experiments on a real-world social network dataset to validate our theoretical results.

\end{abstract}

\section{Introduction}
Network exploration is a fundamental paradigm of searching/exploring
in order to discover information and resources available at nodes in a network, 
and random walk is often used as an effective tool for network exploration \cite{lv2002search,gleich2015pagerank,wilder2018maximizing}.
In this paper, we study the multi-layered network exploration via random walks problem, which can model many real-world applications, including resource searching in peer-to-peer (P2P) networks and web surfing in online social networks (OSNs).

In P2P networks, a user wants to find resources that are scattered on nodes (i.e. peers) in different P2P networks via some platform-specified strategies.
A commonly used search strategy is based on multiple random walks \cite{lv2002search}.
Since different resources have different importance,
the search quality of different random walkers varies.
Moreover, the resource-search process typically has a life span, i.e., a total time-limit or hop-limit for random walks.
So the user's goal is to decide how to allocate limited budgets to different random walkers to find as many important resources as possible.
Another application is the web surfing, where users want to find information in different OSNs, by looking at posts via others' home-pages~\cite{lerman2006social}.
Once the user arrives at one of his friends' home-pages in a particular OSN, he could find some information and continues to browse the home-pages of his friends' friends, which can be regarded as a random walk process.
Since the user only has a finite duration in web surfing,
the similar question is how to allocate his time in exploring different OSNs so to get the maximum amount of useful information.

We abstract the above application scenarios as the \textbf{Mu}lti-\textbf{La}yered \textbf{N}etwork \textbf{E}xploration (MuLaNE) problem.
In \mulane, we model the overall network to be explored (e.g., combining different OSNs) as a~multi-layered network $\cG$ that consists of $m$ layers $L_1, ..., L_m$.
Each layer $L_i$ (e.g., a single OSN) is represented by a weighted directed graph $\cG_i(\cV_i,\cE_i,w_i)$,
where $\cV_i$ is the set of nodes to be explored (e.g., users' home-pages with importance weight $\sigma_u$ for $u\in \cV_i$), $\cE_i \subseteq \cV_i \times \cV_i$ is the set of directed edges (e.g., social links), and $w_i$ is the edge weight function on edges $\cE_i$.
Each layer $L_i$ is associated with an explorer (or random walker) $W_i$ and a fixed starting distribution  $\bm{\alpha}_i$ on nodes $\cV_i$.
Explorer $W_i$ starts on a node in $\cG_i$ following the distribution $\bm{\alpha}_i$, and then walks on network $\cG_i$ following outgoing edges with probability proportional to edge weights.
We assume that each random walk step will cost one unit of the budget\footnote{Our model can be extended to move multiple steps with one unit of budget.}.
Given a total budget constraint $B$ and individual layer budget constraint $c_i$, the MuLaNE task is to find the optimal budget allocation 
$\bk{=}(k_1, ..., k_m)$ with ${\sum_{i=1}^mk_i} \le B$, $0 {\le} k_i {\le} c_i$ such that the
	expected total weights of {\it unique} nodes visited are maximized.

In real applications, the network structure $\cG$, the node weights $\bsig$ and the starting distributions $\bm{\alpha}_i$'s may not be known in advance, so we consider both the
	\textit{offline optimization} cases when $\cG$, $\bsig$ and $\bm{\alpha}_i$'s are known and the \textit{online learning} case when $\cG$, $\bsig$ and $\bm{\alpha}_i$'s are unknown.
Moreover, different layers may have overlapping vertices (e.g., home-page of the same user may appear in different OSNs), 
and the starting distributions $\bm{\alpha}_i$'s may or may not be stationary distributions.
In this paper, we provide a systematic study of all these case combinations.

For the offline optimization, we first consider the general overlapping setting and provide constant approximation algorithms based on
	the lattice submodularity property for non-stationary  $\bm{\alpha}_i$'s and the diminishing-return (DR) submodularity property for stationary  $\bm{\alpha}_i$'s.
For the special non-overlapping setting, we design a dynamic programming algorithm that finds the exact optimal solution for MuLaNE.

In the online learning, we conduct multiple rounds of exploration, each of which has the same budget constraints $B$ and $c_i$'s. 
After each round of exploration, the total weights of unique nodes visited in this round are the reward for this round, and
	the trajectory of every explorer in every layer and the importance weights of visited nodes are observed as the feedback, which could be used to learn information about the network for the benefit of future explorations.

We adapt the combinatorial multi-armed bandit (CMAB) framework and the CUCB algorithm \cite{chen2016combinatorial} to our setting, and design online learning algorithms to minimize the regret, which
	is the difference between the cumulative reward achieved by the exact or approximate offline algorithm and that achieved by our learning algorithm,  over $T$ rounds.
We show that directly learning the graph structure $\cG$ is inefficient.
Instead, we define intermediate random variables in MuLaNE as the base arms in CMAB and learn these intermediate parameters, which can sufficiently determine the rewards to guide our budget allocation.
Moreover, the node weight is not revealed until this node is first visited, which further complicates the design of online exploration algorithms.
For the overlapping case, we adapt the CUCB algorithm to address the unrevealed node weights and  the extra constraint of monotonicity for the intermediate parameters.
We further improve the analysis by leveraging on special properties in our setting and show logarithmic regret bounds in this case.
For the non-overlapping case, we define more efficient intermediate parameters and use the exact offline algorithm, and thus achieve a better regret bound. 
Finally, we conduct experiments on a real-world multi-layered social network dataset to validate the effectiveness of both our offline and online algorithms.

Our contributions can be summarized as follows: (1) We are the first to model the multi-layered network exploration via random walks problem (\mulane) as an abstraction for many real-world applications (2) We provide a systematic study of the \mulane{} problem via theoretical analysis and empirical validation by considering offline and online settings for both overlapping and non-overlapping multi-layered networks. 
Due to the space limit, proofs are included in the supplementary material.

\vspace{-9pt}
\subsection{Related Work}\label{sec: related work}
\vspace{-5pt}
Network exploration via random walks has been studied in various application contexts such as community detection~\cite{pons2005computing}, centrality measuring~\cite{gleich2015pagerank},
	large-scale network sampling~\cite{li2019walking}, and influence maximization~\cite{wilder2018maximizing}.
Multiple random walks has also been used, e.g. in~\cite{lv2002search} for query resolution in peer-to-peer networks.
However, none of these studies address the budgeted MuLaNE problem.
\mulane{} is also related to the influence maximization (IM) problem~\cite{kempe2003maximizing, chen2009efficient} and can be viewed as a special variant of IM. 
Different from the standard Independent Cascade (IC) model~\cite{wang2012scalable} in IM, which randomly broadcasts to all its neighbors, the propagation process of \mulane{} is a random walk that selects one neighbor. Moreover, each unit of budget in \mulane{} is used to propagate one random-walk step, not to select one seed node as in IC. 

The offline budget allocation problem has been studied by \citet{alon2012optimizing,soma2014optimal}, the latter of which
propose the lattice submodularity and a constant approximation algorithm based on this property.
Our offline overlapping MuLaNE setting is based on a similar approach, but we provide a better approximation ratio compared to the original analysis in~\cite{alon2012optimizing}.
Moreover, we further show that the stationary setting enjoys DR-submodularity, leading to an efficient algorithm with a better approximation ratio.

The multi-armed bandit (MAB) problem is first studied by \citet{robbins1952some} and then extended by many studies
(cf. \cite{BCB12}).
Our online MuLaNE setting fits into the general Combinatorial MAB (CMAB) framework of~\cite{gai2012combinatorial,chen2016combinatorial}.
CUCB is proposed as a general algorithm for CMAB~\cite{chen2016combinatorial}.
Our study includes several adaptations, such as handling unknown node weights, defining intermediate random variables as base arms and so on.

\citet{chen2018community} study the community exploration problem, which is essentially a special case of MuLaNE with
	non-overlapping complete graphs.
As a result, their technique is different and much simpler than ours. 

The rest of the paper is organized as follows. 
Section \ref{sec: problem setting} states the settings of \mulane;
Section \ref{sec: equivalent model} states the equivalent bipartite coverage model to derive the explicit form for our reward function;
Section \ref{sec: offline opt} states offline algorithms for four offline settings with provable approximation guarantee and running time analysis;
Section \ref{sec: online learning} states two online learning algorithms with regret analysis;
Empirical results are shown in Section \ref{sec: experiments};
and Section \ref{sec: conclusion} concludes the paper.

\section{Problem Settings}\label{sec: problem setting}

\textbf{Basic Notations}.
In this paper, we use $\R$ and $\Z$ to denote the sets of real numbers and integers, respectively, and
	the associated subscript $\ge 0$ and $>0$ denote their non-negative and positive subsets, respectively.
Suppose we want to explore a multi-layered network $\cG(\cV, \cE, \bsig)$, consisting of $m$ layers $L_1, ..., L_m$ and $N$ unique nodes $\cV$ with fixed node weights
$\bsig\in [0,1]^{\cV}$.
Let $[m]$ denote the set $\{1,2,...,m\}$.
Each layer $L_i, i \in[m]$, represents a subgraph of $\cG$ we could explore and is modeled as a weighted digraph $\cG_i(\cV_i, \cE_i;w_i)$,
where $\cV_i \subseteq \cV$ is the set of vertices in layer $L_i$, $\cE_i \subseteq \cV_i\times \cV_i$ is the set of edges in $\cG_i$, 
and $w_i: \cE_i \rightarrow \R_{> 0}$ is the edge weight function associating each edge of $\cG_i$.
For any two layers $L_i$ and $L_j$, we say they are \textit{non-overlapping} if  $\cV_i \cap \cV_j = \emptyset$. 
$\cG$ is called a \textit{non-overlapping} multi-layered network if any two layers are non-overlapping; 
otherwise $\cG$ is an \textit{overlapping} multi-layered network. 
Let $n_i=|\cV_i|$ denote the size of layer $L_i$, and 
the adjacency matrix of $\cG_i$ is defined as $\bm{A_i} \in \R^{n_i \times n_i}_{\ge 0}$, where
$\bm{A_i}[u,v]=w_i((u,v))$ if $(u,v)\in \cE_i$ and 0 otherwise. 

\noindent\textbf{Exploration Rule.}
Each layer $L_i$ is associated with an explorer $W_i$, a budget $k_i \in \Z_{\ge 0}$ and an initial starting distribution 
$\bm{\alpha}_i = (\alpha_{i,u})_{u\in \cV_i}\in \R^{n_i}_{\ge 0}$ with $\sum_{u\in \cV_i} \alpha_{i,u}=1$.
The exploration process of $W_i$ is identical to applying \textit{weighted random walks} within layer $L_i$. 
Specifically, the explorer $W_i$ starts the exploration from a random node $v_1\in \cV_i$ with probability $\alpha_{i,v_1}$; 
it consumes one unit of budget and continues the exploration by walking
from $v_1$ to one of its out-neighbors, say $v_2$, with probability $\bm{P}_i[v_1, v_2]$, where $\bm{P}_i \in \R^{n_i \times n_i}_{\ge 0}$ is the transition probability matrix, and $\bm{P}_i{[u,v]} {\defeq} \bm{A_i}{[u,v]}/(\sum_{w:{(u,w)} \in \cE_i} \bm{A_i}{[u,w]} )$. 
Consider visiting the initial node $v_1$ as the first step, the process is repeated until the explorer $W_i$ walks $k_i$ steps
 and visits $k_i$ nodes (with possibly duplicated nodes).
Note that one can easily generalize our results by moving $\lambda_i \in \bbZ_{>0}$ steps with one unit of budget for $W_i$, 
and for simplicity, we set $\lambda_i=1$.

\noindent\textbf{Reward Function.}
We define the exploration trajectory for explorer $W_i$ after exploring $k_i$ nodes as $\Phi(i,k_i) {\defeq}(X_{i,1}$, $\ldots, X_{i,k_i})$, where $X_{i,j} \in \cV_i$ denotes the node visited at $j$-th step, and $\Pr(X_{i, 1}{=}u) = \alpha_{i,u}, u\in \cV_i$.
The reward for $\Phi(i, k_i)$ is defined as the total weights of \textit{unique} nodes visited by $W_i$, i.e., $\sum_{v \in \bigcup_{j=1}^{k_i}\{X_{i,j}\}}\sigma_v$. 
Considering trajectories of all $W_i$'s, the total reward is the total weights of \textit{unique} nodes visited by all random walkers, i.e., $\sum_{v \in \bigcup_{i=1}^{m}\ \bigcup_{j=1}^{k_i}\{X_{i,j}\}}\sigma_v$.
For notational simplicity, let $\cG {\defeq} (\cG_1, \ldots,
\cG_m)$, $\bm{\alpha} {\defeq} (\bm{\alpha}_1, \ldots, \bm{\alpha}_m)$ and $\bsig=(\sigma_1, ..., \sigma_{|\cV|})$ be the parameters of a problem instance and $\bm{k} {\defeq} (k_1, \ldots, k_m)$ be the allocation vector.
For overlapping multi-layered network $\cG$,
the total expected reward is 
	\begin{equation}\label{eq: over opt}\textstyle
	r_{\cG,\bm{\alpha},\bsig} (\bm{k}) {\defeq} {\E_{\Phi(1,k_1), ..., \Phi(m,k_m)}\left[ \sum_{v \in \bigcup_{i=1}^{m}\ \bigcup_{j=1}^{k_i}\{X_{i,j}\}}\sigma_v \right]},
	\end{equation}
where its explicit formula will be discussed later in Sec.~\ref{sec: equivalent model}.
For non-overlapping $\cG$, 
the reward function can be simplified and written as the summation over separated layers, 
	\begin{equation}\label{eq: non-over opt}\textstyle
	r_{\cG,\bm{\alpha},\bsig} (\bm{k}) \defeq {\sum_{i=1}^{m}}{\E_{\Phi(i,k_i)}\left[\sum_{v \in \bigcup_{j=1}^{k_i}\{X_{i,j}\}}\sigma_v \right]}.
	\end{equation}
\noindent\textbf{Problem Formulation.}
We are interested in the \textit{budget allocation problem}: how to allocate the total budget $B \in \Z_{\ge 0}$ to the $m$ explorers so as to maximize the total weights of unique nodes visited. 
Also, we assume there is a budget constraint $\bc \defeq (c_1, ..., c_m)$ such that the allocated budget $k_i$ should not exceed $c_i$, i.e., $k_i \le c_i$, for $i \in [m]$. 
\setlength{\abovedisplayskip}{5pt}
\setlength{\belowdisplayskip}{-5pt}
\begin{definition}
Given graph structures $\cG \defeq (\cG_1, \ldots,
\cG_m)$, starting distributions $\bm{\alpha}\defeq(\bm{\alpha}_1, \ldots, \bm{\alpha}_m)$, total budget $B$ and budget constraints $\bc \defeq(c_1, ..., c_m)$, the \textbf{Mu}lti-\textbf{La}yered \textbf{N}etwork \textbf{E}xploration problem, denoted as \textbf{\mulane{}}, is formulated as the following optimization problem,
\noindent
\begin{equation}\label{eq: opt budget}
    \text{Maximize} \,\, r_{\cG,\bm{\alpha},\bsig} (\bm{k}) \, s.t. \,\, \bk \in \mathbb{Z}^m_{\ge 0} \le \bc,\, {\sum_{i=1}^{m}}{k_i}\le B,
\end{equation}
\end{definition}
\noindent where $r_{\cG,\bm{\alpha},\bsig} (\bm{k})$ is given by Eq.~(\ref{eq: over opt}) or Eq.~(\ref{eq: non-over opt}).
We also need to consider the following settings:
\setlength{\abovedisplayskip}{7pt}
\setlength{\belowdisplayskip}{7pt}

{\noindent\textbf{\textit{Offline Setting.}} For the offline setting, all problem instance parameters $(\cG, \bm{\alpha}, \bsig, \bc, B)$ are given, and we aim to find the optimal budget
allocation $\bm{k^*}$ determined by Eq.~(\ref{eq:
opt budget}).}

\noindent\textbf{\textit{Online Setting.}} 
For the online setting, we consider $T$-round explorations. 
	Before the exploration, we only know an upper bound of the total number of nodes in $\cG$ and
		the number of layers $m$,\footnote{More precisely we only need to know
		the number of independent explorers. 
		If two explorers explore on the same layer, it is equivalent as two layers with identical graph structures.} but we do not know about the network structure $\cG$
		, the starting distributions $\bm{\alpha}$ or node weights $\bsig$.
	In round $t \in [T]$, we choose the budget allocation $\bm{k}_t {\defeq} (k_{t,1}, \ldots, k_{t,m})$ only based on observations from previous rounds, 
		where $k_{t,i}\in \Z_{\ge 0}$ is the budget allocated to the $i$-th layer and $\sum_{i=1}^{m}{k_{t,i}}\le B, \bk_t \le \bc$. 
	Define $\bm{k}_t$ as the {\em action} taken in round $t$. 
	By taking the action $\bm{k}_t$ we mean that we interact with the environment and the random explorer $W_i$, which is part of the environment, would explore $k_{t,i}$ steps and generate
	 the exploration trajectory $\Phi(i, k_{t,i}){=}(X_{i,1}, ..., X_{i,k_{t,i}})$.
	{After we take
	the action $\bm{k}_t$, the exploration trajectory $\Phi(i, k_{t,i})$ for each
	layer $L_i$ as well as the fixed importance weight $\sigma_u$ of $u \in \Phi(i,k_{t,i})$ is revealed as the \textit{feedback}\footnote{We further consider random node weights with unknown mean vector $\bsig$, see the \OnlyInFull{Appendix~\ref{appendix: over extend}}\OnlyInShort{supplementary material} for details.}, which  
	we leverage on to
	learn parameters related to the graph structures $\cG$, the starting distribution $\bm{\alpha}$ and the node weights $\bsig$,
		so that we can select better actions in future rounds}.
The reward we gain in round $t$ is the total weights of unique nodes visited by all random explorers in round $t$.
Our goal is to design an efficient online learning algorithm $A$ to give us guidance on
taking actions and gain as much cumulative reward as possible in $T$ rounds. 

In general, the online
	algorithm has to deal with the exploration-exploitation tradeoff.
	The cumulative~regret is a commonly used metric to evaluate
	the performance of an online learning algorithm $A$.
	{
	Formally, the $T$-round ($(\xi,\beta)$-approximation) regret of $A$ is:
			\begin{equation}\textstyle \label{eq: alpha beta regret}
			Reg_{\cG,\bm{\alpha},\bsig}(T) {=} \xi \beta T\cdot r_{\cG,\bm{\alpha},\bsig}( \bm{k^*}){-}\E\big[ \sum_{t=1}^Tr_{\cG,\bm{\alpha},\sigma}(\bm{k_t}^A)\big],
			\end{equation} 
			}
	where $r_{\cG,\bm{\alpha},\bsig} (\bm{k^*})$ is the reward value for the optimal budget
	allocation $\bm{k^*}$, $\bm{k}_t^A$ is budget allocation selected by the learning algorithm
	$A$ in round $t$, the expectation is taken over the randomness of the
	algorithm and the exploration trajectories in all $T$ rounds,
	and $(\xi,\beta)$ is the approximation guarantee of the offline oracle as explained below.
Similar to other online learning frameworks~\cite{chen2016combinatorial, wang2017improving}, we assume that the online learning algorithm has access to an offline $(\xi, \beta)$-approximation oracle, which for problem 
instance $(\cG, \bm{\alpha}, \bsig, \bc, B)$ outputs an action $\bk$ such that $\Pr\left(r_{{\cG},\bm{\alpha}, \bsig} (\bk) \ge \xi \cdot r_{{\cG},\bm{\alpha},\bsig}( \bk^*)\right) \ge \beta$.
We also remark that the actual oracle we use takes certain intermediate parameters as inputs instead of $\cG$ and $\bm{\alpha}$.

\textbf{{Submodularity and DR-Submodularity Over Integer lattices.}}
To solve the \mulane{} problem, 
we leverage on the submodular and DR-submodular properties of the reward function.
For any $\bx, \by \in \bbZ^m_{\ge 0}$, we denote $\bx \vee \by, \bx \wedge \by \in \bbZ^m_{\ge 0} $ as the coordinate-wise maximum and minimum of these two vectors, i.e., $(\bx \vee \by)_i=\max\{x_i, y_i\}, (\bx \wedge \by)_i=\min\{x_i, y_i\}$.
We define a function $f: \bbZ_{\ge 0}^m \rightarrow \bbR$ over the integer lattice $\bbZ^m_{\ge 0 }$ as a \textit{ submodular} function if the following inequality holds for any $\bx, \by \in \bbZ^m_{\ge 0}$:
\begin{equation}\label{ieq: lattice def}\textstyle
f(\bx \vee \by) + f(\bx \wedge \by) \le f(\bx) + f(\by).
\end{equation}
Let $\text{supp}^+(\bx - \by)$ denote the set $I = \{i\in [m]: x_i > y_i\}$,
$\bch_i = (0, ..., 1, ..., 0)$ be the one-hot vector whose $i$-th element is 1 and 0 otherwise, and $\bx \ge \by$ means $x_i \ge y_i$ for all $i \in [m]$.
We define a function $f: \bbZ_{\ge 0}^m \rightarrow \bbR$ as a \textit{DR-submodular} (diminishing return submodular) function if the following inequality holds for any $\bx \le \by$ and $i \in [m]$, 
\begin{equation}\textstyle 
f(\by + \bch_i) - f(\by) \le f(\bx + \bch_i) - f(\bx).
\end{equation}
We say a function $f$ is \textit{monotone} if for any $\bx \le \by$, $f(\bx) \le f(\by)$. 
Note that for a function $f: \bbZ_{\ge 0}^m \rightarrow \R$, submodularity does not imply DR-submodularity over integer lattices.
In fact, the former is weaker than the latter, that is, a DR-submodular function is always a submodular function, but not vice versa.
However, for a typical submodular function $f: \{0,1\}^m \rightarrow \bbR$ defined on a set, they are equivalent.

\section{Equivalent Bipartite Coverage Model}\label{sec: equivalent model}
In order to derive the explicit formulation of the reward function $r_{\cG, \bm{\alpha},\bsig}( \bk)$ in Eq.~(\ref{eq: over opt}),
we construct an undirected bipartite coverage graph $\cB(\cW,\cV,\cE')$, where $\cW=\{W_1, ..., W_m\}$ denotes $m$ random explorers,
$\cV=\bigcup_{i \in [m]}{\cV_i}$ denotes all possible distinct nodes to be explored in $\cG$, and the edge set $\cE'{=}\{(W_i,u)|u \in \cV_i, i \in [m]\}$ indicating whether node $u$ could be visited by $W_i$. 
For each edge $(W_i, u) \in \cE'$, we associate it with $c_i+1$ visiting probabilities denoted as $P_{i,u}(k_i)$ for $k_i \in \{0\} \cup  [c_i]$.
$P_{i,u}(k_i)$ represents the probability that the node $u$ is visited by the random walker $W_i$ given the budget $k_i$, i.e., $P_{i,u}(k_i)=\Pr(u \in \Phi(i,k_i))$.
Since $W_i$ can never visit the node outside the $i$-th layer (i.e., $u \notin \cG_i$), we set $P_{i,u}(k_i)=0$ if $(W_i, u) \not\in \cE'$. 
Then, given budget allocation $\bk$, the probability of a node $u$ visited by at least one random explorer is $\Pr(u, \bk){=} 1-\Pi_{i\in[m]}(1-P_{i,u}(k_i))$ because each $W_i$ has the independent probability $P_{i,u}(k_i)$ to visit $u$. 

By summing over all possible nodes, the reward function is 
\begin{equation}\label{eq: over reward}\textstyle
r_{\cG,\bm{\alpha}, \bsig}(\bm{k})=\sum_{u\in\cV}\sigma_u\left(1-\Pi_{i\in[m]}\left(1-P_{i,u}(k_i)\right)\right)
\end{equation}
According to Eq.~(\ref{eq: non-over opt}), for non-overlapping multi-layered network, we can rewrite the reward function as:
\begin{equation}\label{eq: non-over reward}\textstyle
    r_{\cG,\bm{\alpha},\bsig} (\bm{k})=\sum_{i\in[m]}\sum_{u \in \cV_i}\sigma_u P_{i,u}(k_i).
\end{equation}

\noindent\textbf{Remark.}
The bipartite coverage model is needed to integrate the graph structure and the random walk exploration mechanisms into $P_{i,u}(k_i)$ to determine the reward function. It also handles the scenario where each layer is explored by multiple random walkers \OnlyInFull{(see Appendix~\ref{appendix: multiple})}\OnlyInShort{(see supplementary material)}. 

\subsection{Properties of the visiting probability $P_{i,u}(k_i)$}\label{sec: explicit formula}

The quantities $P_{i,u}(k_i)$'s in Eq.~\eqref{eq: over reward} and~\eqref{eq: non-over reward} are crucial for both our
	offline and online algorithms, and thus we provide their analytical formulas and properties here.
In order to analyze the property of the reward function, we apply the absorbing Markov Chain technique to calculate $P_{i,u}(k_i)$.
For $u \notin \cG_i$, $P_{i,u}(k_i)=0$;
For $u \in \cG_i$, we create an absorbing Markov Chain $\bm{P_i(u)} \in \R^{n_i \times n_i}$ by setting the target node $u$ as the absorbing node. 
We derive the corresponding transition matrix $\bm{P_i(u)}$ as 
$\bm{P_i(u)}[v,\cdot] = \bch_u^{\top}$ if $v$ = $u$ and $\bm{P_i(u)}[v,\cdot]
=\bm{P_i}{[v,\cdot]}$ otherwise,
where $\bm{P_i}{[v,\cdot]}$ denotes the row vector corresponding to the node $v$ of $\bm{P_i}$,
 and $\bch_u= (0, ..., 0, 1, 0, ..., 0 )^\top$ denotes the one-hot vector with $1$ at the $u$-th entry and $0$ 
 elsewhere.\footnote{When related to matrix operations, we treat all vectors as column vectors by default.}
Intuitively, $\bm{P_i(u)}$ corresponds to the transition probability matrix of $\cG_i$ after removing all out-edges of $u$ and adding a self loop to itself in the original graph $\cG_i$. 
The random walker $W_i$ will be trapped in $u$ if it ever visits $u$.
We observe that $P_{i,u}(k_i)$ equals to the probability the random walker stays in
the absorbing node $u$ at step $k_i \in \bbZ_{> 0}$ (trivially, $P_{i,u}(0)=0
$), 
\begin{equation}\label{eq: simple form}\textstyle
P_{i,u}(k_i)=\bm{\alpha_i}^{\top} \bm{P_i(u)}^{k_i-1} \bch_u.
\end{equation}
Then define the marginal gain of $P_{i,u}(\cdot)$ at step $k_i\ge1$ as,
\begin{equation}\label{eq: marginal gain}\textstyle
g_{i,u}(k_i) = P_{i,u}(k_i) - P_{i,u}(k_i-1).
\end{equation}
The physical meaning of $g_{i,u}(k_i)$ is the probability that node $u$ is visited exactly at the $k_i$-th step and not visited before $k_i$. 
Now, we can show $g_{i,u}(k_i)$ is non-negative,
\begin{restatable}{lemma}{lemNonDecreasing}\label{lem: non-decreasing}
$g_{i,u}(k_i) \ge 0$ for any $i\in [m]$, $u \in \cV_i$, $k_i\in \Z_{> 0}$.
\end{restatable}
This means
$P_{i,u}(k_i)$ is non-decreasing with respect to step $k_i$.
In other words, the more budgets we allocate to $W_i$, the higher probability $W_i$ can visit $u$ in $k_i$ steps.

\noindent{\textbf{Starting from arbitrary distributions:}}
Although $P_{i,u}(k_i)$ is monotone (non-decreasing) w.r.t $k_i$, 
 $g_{i,u}(k_i)$ may not be monotonic non-increasing under arbitrary staring distributions.
Intuitively, there may exist one critical step $k_i$ such that $P_{i,u}(k_i)$ suddenly increases by a large value \OnlyInFull{(see Appendix~\ref{appendix: explct form stationary})}\OnlyInShort{(see supplementary material)}.
This means that $P_{i,u}(k_i)$ lacks the diminishing return (or ``discretely concave") property, which many problems rely on to provide good optimization~\cite{kapralov2013online,soma2015generalization}.
In other words, we are dealing with a more challenging non-concave optimization problem for the discrete budget allocation.

\noindent{\textbf{Starting from the stationary distribution:}}
If our walkers start with the stationary distributions, the following property holds: $g_{i,u}(k_i)$ is non-increasing w.r.t $k_i$.
\begin{restatable}{lemma}{thmStationary}\label{lem: stationary}
	$g_{i,u}(k_i+1) - g_{i,u}(k_i) \le 0$ for any $i\in [m]$, $u \in \cV_i$, $k_i\in \Z_{> 0}$, if $\bm{\alpha}_i=\bm{\pi_i}$, where $\bm{\pi_i}^{\top}\bm{P_i}=\bm{\pi_i}^{\top}$.
\end{restatable}


\section{Offline Optimization for \mulane}\label{sec: offline opt}
In this section, we first consider the general case, where layers are overlapping and starting distributions are arbitrary.
Next, we consider the starting distribution is the stationary distribution, and give solutions with better solution quality and time complexity.
Then we analyze special cases where layers are non-overlapping and give optimal solutions for arbitrary distributions.
The summary for offline models and algorithmic results are presented in Table~\ref{tab: offline models}.
\begin{table*}[t]
	\centering
	\caption{Summary of the offline models and algorithms.}	\resizebox{2.05\columnwidth}{!}{
	\label{tab: offline models}
	\begin{tabular}{ccccc}
		\toprule
		\textbf{Overlapping?}& \textbf{Starting distribution}& \textbf{Algorithm} & \textbf{Apprx ratio} & \textbf{Time complexity}\\
		\midrule
		\checkmark & \textit{Arbitrary} & Budget Effective Greedy & 
		$(1-e^{-\eta})$\footnote{$1-e^{-\eta}\approx 0.357$, where $\eta$ is the solution of $e^\eta=2-\eta$.}
		&$O(B\norm{\bc}_{\infty}mn_{max}+\norm{\bc}_{\infty}mn_{max}^3)$ \\
		\checkmark & \textit{Stationary} & Myopic Greedy & $(1-1/e)$ & $O(Bmn_{max}+\norm{\bc}_{\infty}mn_{max}^3)$ \\
		$\times$ & \textit{Arbitrary} & Dynamic Programming & 1 & $O(B\norm{\bc}_{\infty}m+\norm{\bc}_{\infty}mn_{max}^3)$ \\
		$\times$ & \textit{Stationary} & Myopic Greedy & 1 & $O(B\log m + \norm{\bc}_{\infty}mn_{max}^3)$ \\
		\bottomrule
	\end{tabular}
}
\end{table*}

\subsection{Overlapping MuLaNE}
\textbf{Starting from arbitrary distributions.}
Based on the equivalent bipartite coverage model, one can observe that our problem formulation is a generalization of the Probabilistic Maximum Coverage (PMC)~\cite{chen2016combinatorial} problem, which is NP-hard and has a $(1-1/e)$ approximation based on submodular set function maximization.
However, our problem is more general in that we want to select multi-sets from $\cW$ with budget constraints, and the reward
	function in general does not have the DR-submodular property for an arbitrary starting distribution.
Nevertheless, we have the following lemma to solve our problem.

\begin{restatable}{lemma}{thmSubmodular}\label{thm: lattice submodular}
For any network $\cG$, distribution $\bm{\alpha}$ and weights $\bsig$, $r_{\cG, \bm{\alpha},\bsig}(\cdot): \bbZ^m_{\ge 0} \rightarrow \R$ is monotone and submodular.
\end{restatable}

Leveraging on the monotone submodular property, we design a Budget Effective Greedy algorithm (Alg.~\ref{alg: layer traversal greedy}).
The core of Alg.~\ref{alg: layer traversal greedy} is the BEG procedure.
Let $\delta(i,b, \bk)$ be the \textit{per-unit marginal gain} $(r_{\cG,\bm{\alpha},\bsig}( \bm{k}+b\bm{\chi}_i) - r_{\cG,\bm{\alpha},\bsig}( 
\bm{k}))/b$ for allocating $b$ more budgets to layer $i$, which equals to
\begin{equation}\textstyle\label{eq: marginal gain delta}
 \sum_{u\in \cV}\sigma_u \Pi_{j\neq i}(1-P_{j,u}(k_j))(P_{i,u}(k_i+b)-P_{i,u}(k_i))/b.   
\end{equation}
BEG procedure consists of two parts and maintains a queue $\cQ$, where any pair $(i,b) \in \cQ$ represents a tentative plan of allocating $b$ more budgets to layer $i$.
The first part is built around the while loop (line~\ref{line: alg3 while begin}-\ref{line: remove}), where each iteration greedily selects the $(i,b)$ pair in $\cQ$ such that the \textit{per-unit marginal gain} $\delta(i,b, \bk)$ is maximized.
The second part is a for loop (line~\ref{line: alg3 for begin}-\ref{line: alg3 for end}), where in the $i$-th round we attempt to allocate all $c_i$ budgets to layer $i$ and replace the current best budget allocation if we have a larger reward.
\begin{restatable}{theorem}{thmOverSol}\label{thm: overlapping solution}
    Algorithm~\ref{alg: layer traversal greedy} obtains a $(1-e^{-\eta})\approx 0.357$-approximate solution, where $\eta$ is the solution of equation $e^\eta=2-\eta$, to the overlapping MuLaNE problem.
\end{restatable}
Line 1 uses $O(m\norm{\bc}_{\infty}n_{max}^3)$ time to pre-calculate visiting probabilities based on Eq.~(\ref{eq: simple form}), 
where $n_{max}=\max_{i}|\cV_i|$.
In the BEG procedure, the while loop contains $B$ iterations, the size of queue $\cQ$ is $O(m\norm{\bc}_{\infty})$, and line~\ref{line: largest margin} uses $O(n_{max})$ to calculate $\delta(i,b,\bk)$ by proper pre-computation and update\OnlyInFull{ (see Appendix~\ref{appendix: efficient eval})}, 
thus the time complexity of Algorithm~\ref{alg: layer traversal greedy} is $O(B\norm{\bc}_{\infty}mn_{max} + m\norm{\bc}_{\infty}n_{max}^3)$.

\noindent\textbf{Remark 1.} The idea of combining the greedy algorithm with enumerating solutions on one layer is also adopted in previous works~\cite{khuller1999budgeted, alon2012optimizing}, but
they only give a $\frac{1}{2}(1-e^{-1}) \approx 0.316$-approximation analysis. 
In this paper, we provide a novel analysis with a better $(1-e^{-\eta})\approx 0.357$-approximation, \OnlyInFull{see Appendix~\ref{appendix: offline proof arbitrary distribution} for details}.

\noindent\textbf{Remark 2.} Another algorithm~\cite{alon2012optimizing} with a better approximation ratio is to use partial enumeration techniques (i.e., BEGE), which can achieve $(1-1/e)$ approximation ratio.
This is the best possible solution in polynomial time unless P=NP.
But the time complexity is prohibitively high in $O(B^4m^4\norm{\bc}_{\infty}n_{max}+ m\norm{\bc}_{\infty}n_{max}^3)$.
\OnlyInFull{The algorithm and the analysis are provided in the Appendix~\ref{appendix: enum begreedy analysis}.}

\begin{algorithm}[t]
	\caption{Budget Effective Greedy (BEG) Algorithm for the Overlapping MuLaNE}\label{alg: layer traversal greedy}
	\resizebox{.97\columnwidth}{!}{
\begin{minipage}{\columnwidth}
	\begin{algorithmic}[1]
	    \INPUT Network $\cG$, starting distributions $\bm{\alpha}$, node weights $\bsig$, budget $B$, constraints $\bc$.
	    \OUTPUT Budget allocation $\bk$.
	    \STATE Compute visiting probabilities $(P_{i,u}(b))_{i\in[m], u\in\cV, b \in [c_i]}$ according to Eq.~(\ref{eq: simple form}). \label{line:computePiub}
	    \STATE $\bk \leftarrow$ BEG($(P_{i,u}(b))_{i\in[m], u\in\cV, b \in [c_i]}$,$\bsig$, $B$, $\bc$). \label{line:return}
	    \algrule
		\FUNCTION {BEG($(P_{i,u}(b))_{i\in[m], u\in\cV, b \in [c_i]}$, $\bsig$, $B$, $\bc$)}
		\STATE Let $\bm{k}\defeq(k_1, ..., k_m) \leftarrow \bm{0}$, $K\leftarrow B$. \label{line:test} 
		\STATE Let $\cQ \leftarrow \{(i,b_i) \, | \, i \in [m], 1\le b_i \le  c_i\}$.
		\WHILE {$K > 0$ and $\cQ \neq \emptyset$ \label{line: alg3 while begin}}
		    \STATE $(i^*,b^*) \leftarrow \argmax_{(i,b)\in \cQ} \delta(i,b,\bk)/b$ \label{line: largest margin} \algorithmiccomment{Eq.~(\ref{eq: marginal gain delta})}
		    \STATE $k_{i^*} \leftarrow k_{i^*} + b^*$, $K \leftarrow K - b^*$.
		    \STATE Modify all pairs $(i,b) \in \cQ$ to $(i, b - b^*)$.
		    \STATE Remove all paris $(i,b) \in \cQ$ such that $b \le 0$.\label{line: remove}
		\ENDWHILE
		\FOR{$i \in [m]$} \label{line: alg3 for begin}
		\STATE \textbf{if} $r_{\cG,\bm{\alpha},\bsig}( c_i\bm{\chi}_i) > r_{\cG,\bm{\alpha},\bsig}( \bm{k})$, \textbf{then} $\bk \leftarrow c_i \bch_i.$ \label{line: alg3 for end}
		\ENDFOR
		\STATE \textbf{return} $\bk\defeq(k_1, ..., k_m)$.
		\ENDFUNCTION
		\end{algorithmic}
		\end{minipage}}
\end{algorithm}

\noindent\textbf{{Starting from the stationary distribution.}}
We also consider the special case where each random explorer $W_i$ starts from the stationary distribution $\bm{\pi_i}$
	with $\bm{\pi_i}^{\top} \bP = \bm{\pi_i}^{\top}$.
In this case, we have the following stronger DR-submodularity.

\begin{restatable}{lemma}{thmDRSubmodular}
	For any network $\cG$, stationary distributions $\bm{\pi}$ and node weights $\bsig$, function $r_{\cG,\bm{\pi},\bsig,}(\cdot):\bbZ^m_{\ge 0} \rightarrow \R$ is monotone and DR-submodular.
\end{restatable}
Since the reward function is DR-submodular, 
the BEG procedure can be replaced by 
the simple MG procedure in Alg.~\ref{alg: pure greedy} with a better approximation ratio.
The time complexity is also improved to $O(Bmn_{max}+m\norm{\bc}_{\infty}n_{max}^3)$.
\begin{restatable}{theorem}{thmOverStationarySol}
    Algorithm ~\ref{alg: pure greedy} obtains a $(1-1/e)$-approximate solution to the overlapping
    MulaNE with the stationary starting distributions.
\end{restatable}
%

\begin{algorithm}[t]
\caption{Myopic Greedy (MG) Algorithm for MuLaNE }\label{alg: pure greedy}
\resizebox{.95\columnwidth}{!}{
\begin{minipage}{\columnwidth}
\begin{algorithmic}[1]
\STATE Same input, output and line~\ref{line:computePiub}-\ref{line:return} as in Alg.~\ref{alg: layer traversal greedy}, except replacing BEG with MG procedure below.
\algrule
\FUNCTION{MG($(P_{i,u}(b))_{i\in[m], u\in\cV, b \in [c_i]}$, $\bsig$, $B$, $\bc$)}
\STATE Let $\bm{k}\defeq(k_1, ..., k_m) \leftarrow \bm{0}$, $K \leftarrow B$.\;
\WHILE{$K  > 0$} 
    \STATE $i^*\leftarrow\argmax_{i \in [m], k_i + 1 \le c_i} \delta(i,1,\bk).$  \algorithmiccomment{Eq.~(\ref{eq: marginal gain delta})}
    \label{line: greedy select}
    \STATE $k_{i^*} \leftarrow k_{i^*} + 1$, $K \leftarrow K - 1$.\;
    \ENDWHILE
    \STATE \textbf{return} {$\bk=(k_1, ..., k_m)$.}
\ENDFUNCTION
\end{algorithmic}
\end{minipage}}
\end{algorithm}

\subsection{Non-overlapping MuLaNE}
For non-overlapping MuLaNE, we are able to achieve the exact optimal solution: for the stationary starting distribution, a slight modification of the greedy algorithm Alg.~\ref{alg: pure greedy} gives the optimal solution, while for an arbitrary starting distribution,
we design a dynamic programming algorithm to compute the optimal solution \OnlyInShort{(see supplementary material)}.
\OnlyInFull{Due to the space constraint, the details are included in Appendix~\ref{appendix: offline non over analysis}.}

%

\section{Online Learning for \mulane}\label{sec: online learning}
In the online setting, we continue to study both overlapping and non-overlapping \mulane{} problem.
However, the network structure $\cG$, the distribution $\bm{\alpha}$ and node weights $\bsig$ \textit{are not known a priori}. 
Instead, the only information revealed to the decision maker includes total budget $B$, the number of layers $m$, the number of the target nodes $|\cV|$ (or an upper bound of it) and the budget constraint $\bc$.

\subsection{Online Algorithm for Overlapping \mulane}

For the unknown network structure and starting distributions, we bypass the transition matrices $\bm{P}_i$ and directly estimate the visiting probabilities $P_{i,u}(b) \in [0,1]$.
This avoids the analysis for how the estimated $\bm{{P}}_i$ affects the performance of online algorithms, which could be unbounded since we consider general graph structures and the reward function given by Eq.~(\ref{eq: over reward}) is highly non-linear in $\bm{{P}}_i$.
Moreover, we can save the matrix calculation by directly using $P_{i,u}(b)$, which is efficient even for large networks.

Specifically, we maintain a set of base arms $\cA = \{(i,u,b)| i \in [m], u \in \cV, b \in [c_i]\}$, where the total number $|\cA| = \sum_{i \in [m]}c_i|\cV|$.
For each base arm $(i,u,b) \in \cA$, we denote $\mu_{i,u,b}$ as the true value of each base arm, i.e., 
$\mu_{i,u,b}=P_{i,u}(b)$.
For the unknown node weights, we maintain the optimistic weight $\bar{\sigma}_u=1$ if $u \in \cV$ has not been visited.
After $u$ is first visited and its true value $\sigma_u$ is revealed, we replace $\bar{\sigma}_u$ with $\sigma_u$.
With a little abuse of the notation, we use $\bm{\mu}$ to denote the unknown intermediate parameters and $r_{\bm{\mu}, \bsig}(\bk)$ to denote the reward $r_{\cG, \bm{\alpha}, \bsig}(\bk)$.

We present our algorithm in Alg.~\ref{alg: over online}, which is an adaptation of the CUCB algorithm for the general
	Combinatorial Multi-arm Bandit (CMAB) framework~\cite{chen2016combinatorial} to our setting.
Notice that we use $\cA$ as defined above in the algorithm, and $\cA$ is defined using $\cV$, which is the set of node ids and 
	should not be known before the learning process starts.
This is not an issue, because at the beginning we can create $|\cV|$ (which is known) placeholders for the node ids, and once a new node is visited,we 
	immediately replace one of the placeholders with the new node id.
Thus $\cV$ appearing in the above definition of $\cA$ is just for notational convenience.

In Alg.~\ref{alg: over online}, we maintain an unbiased estimation of the visiting probability $P_{i,u}(b)$, denoted as $\hat{\mu}_{i,u,b}$.
Let $T_{i,u,b}$ record the number of times arm $(i,u,b)$ is played so far and $\bar{\sigma}_v$ denote the optimistic importance weight.
In each round $t \ge 1$, we compute the confidence radius $\rho_{i,u,b}$ in line~\ref{line: over online radius}, which controls the level of exploration.
The confidence radius is larger when the arm $(i,u,b)$ is not explored often (i.e. $T_{i,u,b}$ is small), and thus motivates more exploration. 
Due to the randomness of the exploration process, the upper condifence bound (UCB) value $\tilde{\mu}_{i,u,b}$ could be decreasing w.r.t $b$, but our offline oracle BEG can only accept non-decreasing UCB values (otherwise the $1-e^{-\eta}$ approximation is not guaranteed).
Therefore, in line~\ref{line: over online validate}, we increase the UCB value $\tilde{\mu}_{i,u,b}$ and set it to be $\max_{j \in [b]} {\tilde{\mu}_{i,u,j}}$. 
This is the adaption of the CUCB algorithm to fix our case, and thus we name our algorithm as CUCB-MAX.
After we apply the $1-e^{-\eta}$ approximate solution $\bk$ given by the BEG oracle (i.e., Alg.~\ref{alg: layer traversal greedy}), we get $m$ trajectories as feedbacks.
In line~\ref{line: update weight}, we can update the unknown weights of visited nodes. In line~\ref{line: over online update}, for base arms $(i,u,b)$ with $b \le k_i$, we update corresponding statistics by the Bernoulli random variable $Y_{i,u,b}\in\{0,1\}$ indicating whether node $u$ is visited by $W_i$ in first $b$ steps.

\begin{algorithm}[t]
	\caption{CUCB-MAX Algorithm for the \mulane}\label{alg: over online}
	\resizebox{.97\columnwidth}{!}{
\begin{minipage}{\columnwidth}
	\begin{algorithmic}[1]
	    \INPUT{Budget $B$, number of layers $m$, number of nodes $|\cV|$, constraints $\bc$, offline oracle BEG.}
		\STATE For each arm $(i,u,b) \in \cA$, $T_{i,u,b} \leftarrow 0$, $\hat{\mu}_{i,u,b} \leftarrow 0$. 
		\STATE {For each node $v \in \cV$, $\bar{\sigma}_v \leftarrow 1$.}
		\FOR {$t=1,2,3,...,T $}
		    \FOR {$(i,u,b) \in \cA$}
		    \STATE $\rho_{i,u,b} \leftarrow \sqrt{3\ln t/(2T_{i,u,b})}$.  \label{line: over online radius}
		    \STATE $\tilde{\mu}_{i,u,b} \leftarrow \min \{\hat{\mu}_{i,u,b} + \rho_{i,u,b}, 1\}$. 
		    \label{line: over online ucb}
            \ENDFOR
        \STATE For $(i,u,b)\in \cA$, $\bar{\mu}_{i,u,b} \leftarrow \max_{j \in [b]} {\tilde{\mu}_{i,u,j}}$.		
		\label{line: over online validate}\label{alg: valid}
		\STATE $\bm{k} \leftarrow$ BEG($(\bar{\mu}_{i,u,b})_{(i,u,b)\in\cA}$, $(\bar{\sigma}_v)_{v \in \cV}$, $B$, $\bc$).
		\label{line: over online oracle}
		\STATE Apply budget allocation $\bk$, which gives trajectories $\bm{X}\defeq(X_{i,1}, ..., X_{i,k_i})_{i\in [m]}$ as feedbacks. \label{line: over online feedback}
		\STATE { For any visited node $v \in \bigcup_{i\in[m]}\{X_{i,1}, ..., X_{i,k_{i}}\}$, receive its node weight $\sigma_v$ and set $\bar{\sigma}_v \leftarrow \sigma_v$. }\label{line: update weight}
		\STATE For any $(i,u,b) \in \tau \defeq \{(i,u,b) \in \cA \vert \, b \le k_i\}$, \\$Y_{i,u,b} \leftarrow 1$ if $u \in \{X_{i,1}, ..., X_{i, b}\}$ and $0$ otherwise.\label{line: over bernoulli rv}
		\STATE For $(i,u,b) \in \tau$, update $T_{i,u,b}$ and $\hat{\mu}_{i,u,b}$:\\ 
		$T_{i,u,b} \leftarrow T_{i,u,b} + 1$,  $\hat{\mu}_{i,u,b} \leftarrow \hat{\mu}_{i,u,b} + (Y_{i,u,b} - \hat{\mu}_{i,u,b})/T_{i,u,b}.$\label{line: over online update}
		\ENDFOR
		\end{algorithmic}
    		\end{minipage}}
\end{algorithm}

\noindent\textbf{{Regret Analysis.}}
We define the reward gap $\Delta_{\bm{k}}{=}\max(0, \xi r_{\bm{\mu}, \bsig}(\bm{k^*}) - r_{\bm{\mu}, \bsig}(\bm{k}))$ for all feasible action $\bm{k}$ satisfying $\sum_{i=1}^{m}k_i = B$, $0 \le k_i \le c_i$, where $\bk^*$ is the optimal solution for parameters $\bm{\mu}$, $\bsig$ and
	$\xi=1-e^{-\eta}$ is the approximation ratio of the $(\xi,1)$-approximate offline oracle. For each base arm $(i,u,b)$, 
we define  $\Delta_{\min}^{i,u,b}{=}\min_{\Delta_{\bm{k}}>0, k_i = b}\Delta_{\bm{k}}$ and $\Delta_{\max}^{i,u,b}{=}\max_{\Delta_{\bm{k}}>0, k_i=b}\Delta_{\bm{k}}$. 
As a convention, if there is no action $\bm{k}$ with $k_i=b$ such that $\Delta_{\bm{k}} {>} 0$, we define $\Delta_{\min}^{i,u,b}{=}\infty$ and $\Delta_{\max}^{i,u,b}{=}0$. 
Let $\Delta_{\min} {=} \min_{(i,u,b) \in \cA}\Delta_{\min}^{i,u,b}$ and $\Delta_{\max} {=} \max_{(i,u,b) \in \cA}\Delta_{\max}^{i,u,b}$.
The following theorem summarizes the regret bound for Alg.~\ref{alg: over online}.
\vspace{-0pt}
\begin{restatable}{theorem}{thmOver} \label{thm: over regret}
Algorithm~\ref{alg: over online} has the following distribution-dependent $(1-e^{-\eta}, 1)$ approximation regret,
\begin{equation*}\textstyle
Reg_{\bm{\mu}, \bsig}(T) \le \sum_{(i,u,b) \in \cA}\frac{108m|\cV| \ln T}{\Delta_{\min}^{i,u,b}} + 2|\cA| + \frac{\pi^2}{3}|\cA|\Delta_{\max}.
\end{equation*}
\end{restatable}
\noindent\textbf{Remark 1.}
Looking at the above distribution dependent bound, we have the $O(\log T)$ approximation regret, which is asymptotically tight.
Coefficient $m|\cV|$ in the leading term corresponds to the number of edges in the complete bipartite coverage graph. Notice that we cannot use the true edge number $\sum_{i \in [m]}{|\cV_i|}$, 
	because the learning algorithm does not know which nodes are contained in each layer, and has to explore all visiting
	possibilities given by the default complete bipartite graph.
The set of base arms $\cA$ has some redundancy due to the correlation between these base arms, and thus it is unclear 
	if the summation over all base arms in the regret bound is tight.
For the non-overlapping case, we further reduce the number of base arms to achieve better regret bounds, but for the overlapping case,
	how to further reduce base arms to achieve a tighter regret bound is a challenging open question left for the future work.

\noindent\textbf{Remark 2.}
The $(1-e^{-\eta}, 1)$ approximate regret is determined by the offline oracle BEG that we plug in Line~\ref{line: over online oracle} and can be replaced by $(1-1/e, 1)$ regret using BEGE or even by the exact regret if the oracle can obtain the optimal budget allocation.
The usage of BEG is a trade-off we make between computational efficiency and learning efficiency and empirically, it performs well as we shall see in Section~\ref{sec: experiments}. 

\noindent\textbf{Remark 3.} 
The full proof of the above theorem is included in the
\OnlyInFull{Appendix~\ref{appendix: over online analysis}}\OnlyInShort{supplementary material}, where we rely on the following properties of $r_{\bm{\mu}, \bsig}(\bk)$ to bound the regret.

\begin{restatable}{property}{condOverMono}(Monotonicity). \label{cond: over mono}
	The reward $r_{\bm{\mu}, \bsig}(\bm{k})$ is monotonically increasing, i.e., 
	for any budget allocation $\bm{k}$, any two vectors $\bm{\mu} {=} (\mu_{i,u,b})_{(i,u,b)\in \cA}$, $\bm{\mu'}{=}(\mu'_{i,u,b})_{(i,u,b) \in \cA}$ and any node weights $\bsig$, $\bsig'$, we have $r_{\bm{\mu}, \bsig}(\bm{k}) \le r_{\bm{\mu'}, \bsig'}(\bm{k})$, if $\mu_{i,u,b} \le \mu'_{i,u,b}$ and $\sigma_v \le \sigma'_v$, $\forall (i,u,b) \in \cA, v \in \cV$.
\end{restatable}
\begin{restatable}{property}{condOverOneNorm}(1-Norm Bounded Smoothness).\label{cond: over 1-norm}
	The reward function $r_{\bm{\mu}, \bsig}(\bm{k})$ satisfies the 1-norm bounded smoothness condition, i.e., 
	for any budget allocation $\bm{k}$, any two vectors $\bm{\mu} {=} (\mu_{i,u,b})_{(i,u,b)\in \cA}$, $\bm{\mu'}{=}(\mu'_{i,u,b})_{(i,u,b) \in \cA}$ and any node weights $\bsig$, $\bsig'$, we have $|r_{\bm{\mu}, \bsig}(\bm{k}) - r_{\bm{\mu'}, \bsig'}(\bm{k})| \le \sum_{i \in [m], u \in \cV, b = k_i}(\sigma_u|\mu_{i,u,b} - \mu'_{i,u, b}|+\abs{\sigma_u-\sigma_u'}\mu'_{i,u,b})$.
\end{restatable}
We emphasize that our algorithm and analysis differ from the original CUCB algorithm~\cite{, chen2016combinatorial} as follows.
First,  we have the additional regret caused by the over-estimated weights for unvisited nodes, i.e., $\abs{\sigma_u-\sigma_u'}\mu'_{i,u,b}$ term in~ property \ref{cond: over 1-norm}.
We carefully bound this term based on the observation that $\mu'_{i,u,b}$ is small and decreasing quickly before $u$ is first visited.
Next, we have to take the max (line~\ref{line: over online validate}) to guarantee the UCB value is monotone w.r.t $b$ since our BEG oracle can only output $(1-e^{-\eta},1)$-approximation with monotone inputs. 
Due to the above operation, $(i,u,b)$'s UCB value depends on the feedback from all arms $(i,u,j)$ for $j \le b$
	(set $\tau$ in line~\ref{line: over bernoulli rv}).
	 So we should update 
	all these arms (line~\ref{line: over online update}) to guarantee that the estimates to all these arms are accurate enough.
Finally, directly following the standard CMAB result would have a larger regret, because arms in $\tau$ are defined as triggered arms, but only arms in $\tau'=\{(i,u,b)\in \cA| k_i=b\}$ affect the rewards.
So we conceptually view arms in $\tau'$
 as triggered arms and use a tighter 1-Norm Bounded Smoothness condition as given above 
	to derive a tighter regret bound.
This improves the coefficient of the leading $\ln T$ term in the distribution dependent regret by a factor of $|\tau|/|\tau'|= O(B/m)$, and the $1/\Delta_{\min}^{i,u,b}$ term is smaller since the original definition would have $\Delta_{\min}^{i,u,b}=\min_{\Delta_{\bm{k}}>0, b \le k_i}\Delta_{\bm{k}}$.

\subsection{Online Algorithm for Non-overlapping Case} \label{sec:cucbmg}
For the non-overlapping case, we set \textit{layer-wise marginal gains} as our base arms.
Concretely, we maintain a set of base arms $\cA = \{(i,b) | i \in [m], b \in [c_i]\}$.
For each base arm $(i,b) \in \cA$., let $\mu_{i,b}=\sum_{u \in \cV}\sigma_u(P_{i,u}(b)-P_{i,u}(b-1))$
	be the true marginal gain of assigning budget $b$ in layer $i$.
We apply the standard CUCB algorithm to this setting 
	and call the resulting algorithm CUCB-MG \OnlyInFull{(Algorithm~\ref{alg: non-over online} in Appendix~\ref{appendix: non over online})}\OnlyInShort{(see supplementary material)}.
Note that in the non-overlapping setting we can solve the offline problem exactly, so we can use the exact offline oracle to
	solve the online problem and achieve an exact regret bound.
This is the major advantage over the overlapping setting where we can only achieve an approximate bound.
Define $\Delta_{\bm{k}}=r_{\bm{\mu}, \bsig}(\bm{k^*}) - r_{\bm{\mu}, \bsig}(\bm{k})$ for all feasible action $\bm{k}$,
and $\Delta_{\min}^{i,b}=\min_{\Delta_{\bm{k}}>0, k_i \ge b}\Delta_{\bm{k}}$, 
CUCB-MG has a $O(\sum_{(i,b) \in \cA}{48B \ln T}/{\Delta_{\min}^{i,b}})$ regret bound.
\section{Experiments}\label{sec: experiments}
\begin{figure}[t]
	\centering
	\begin{subfigure}[b]{0.235\textwidth}
		\centering
		\includegraphics[width=\textwidth]{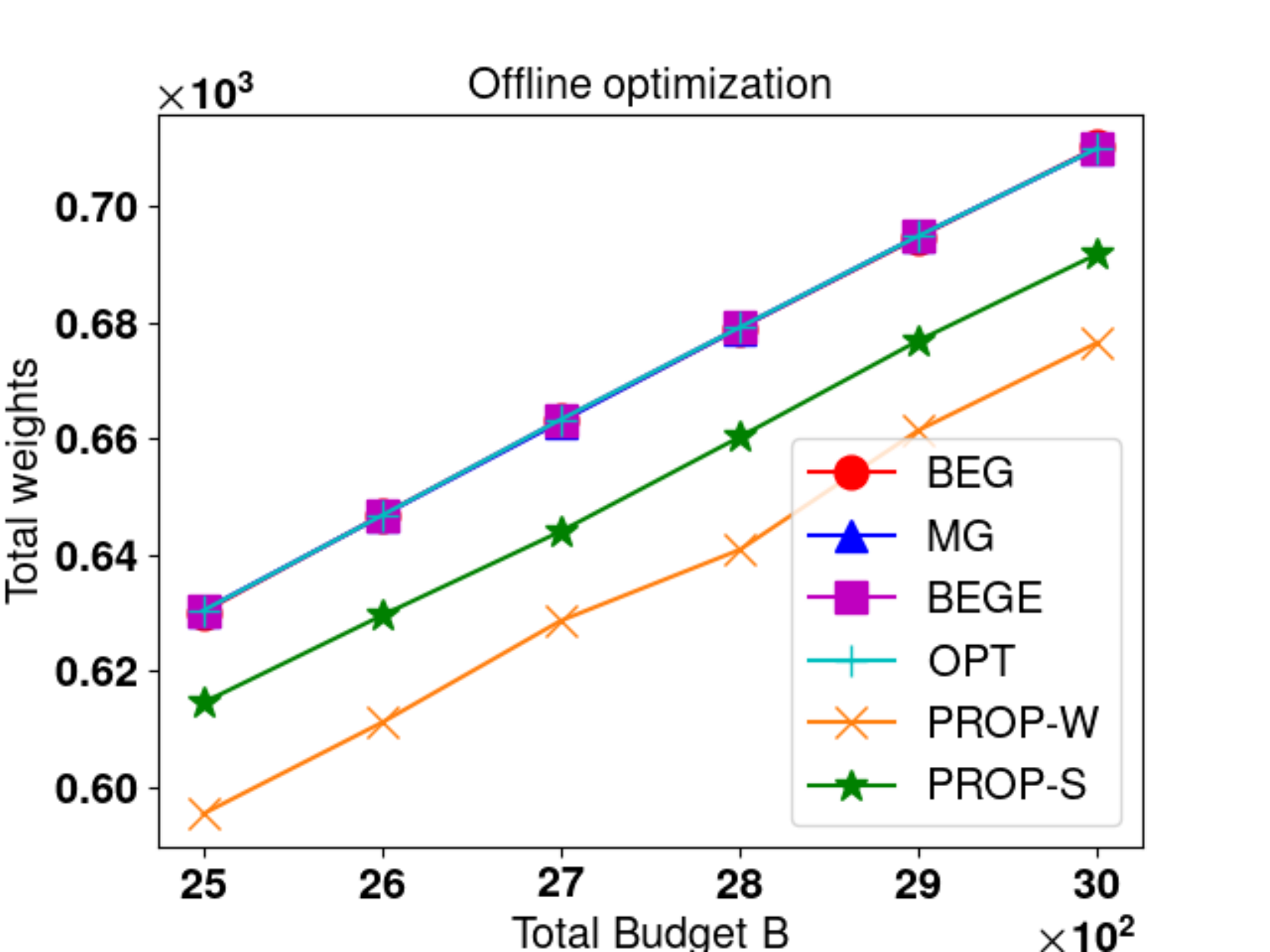}
		\caption{Offline, overlapping.}
		\label{fig: offline over fix}
	\end{subfigure}
	\hfill
	\begin{subfigure}[b]{0.235\textwidth}
		\centering
		\includegraphics[width=\textwidth]{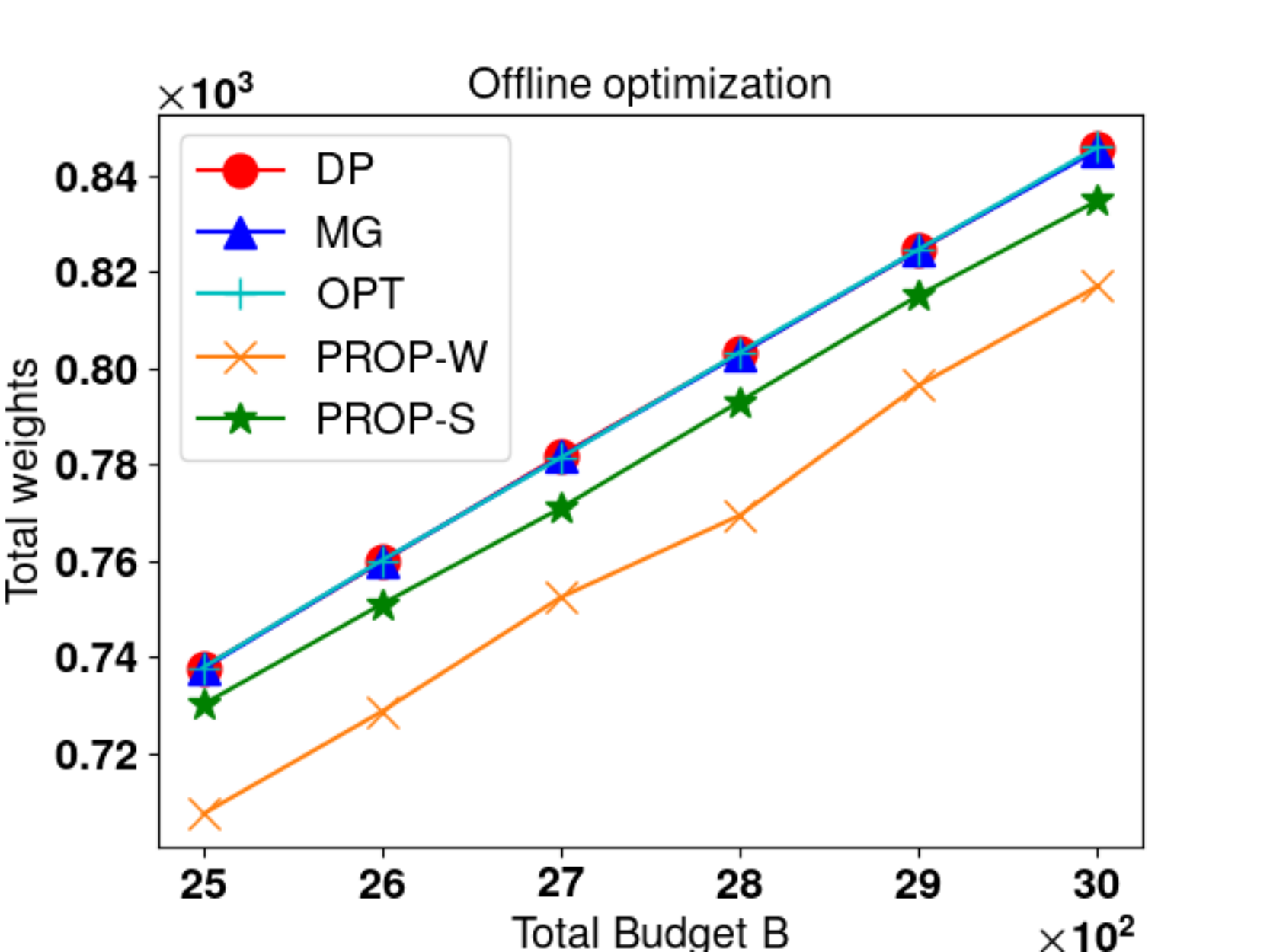}
		\caption{Offline, non-overlapping.}
		\label{fig: offline non over fix}
	\end{subfigure}
	\begin{subfigure}[b]{0.235\textwidth}
		\centering
		\includegraphics[width=\textwidth]{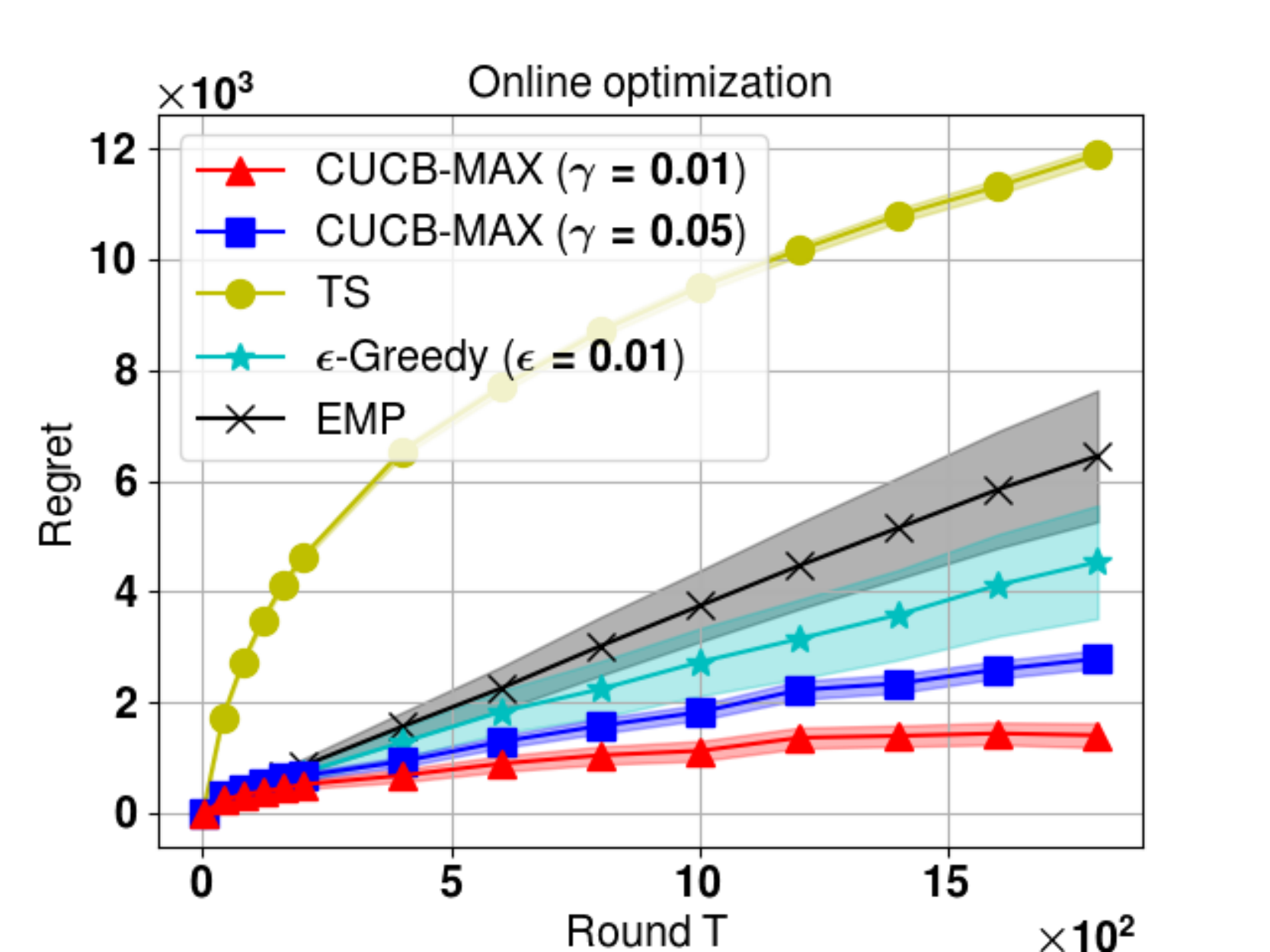}
		\caption{Online, overlapping.}
		\label{fig: online over fix 3000}
	\end{subfigure}
	\hfill
	\begin{subfigure}[b]{0.235\textwidth}
		\centering
		\includegraphics[width=\textwidth]{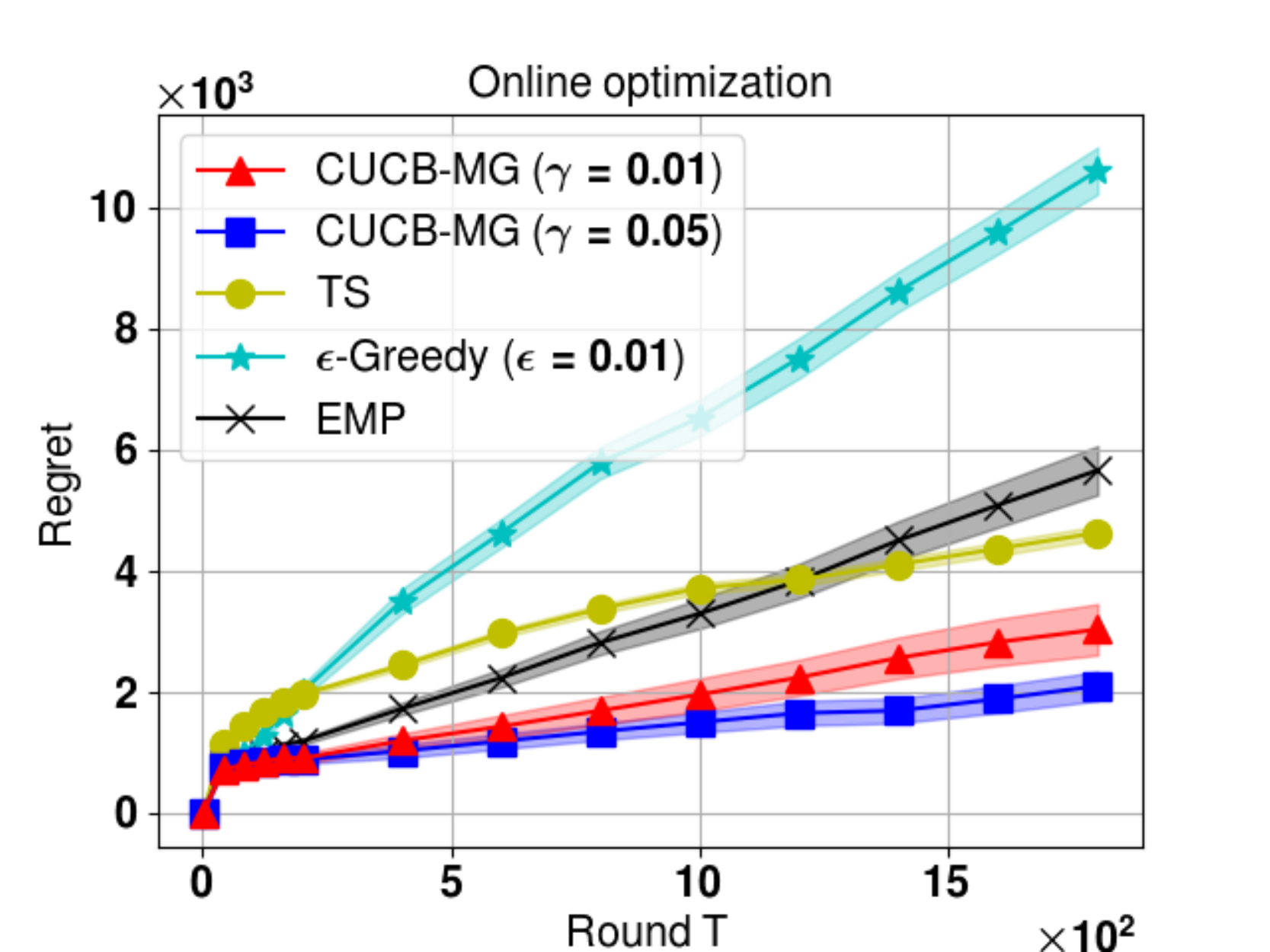}
		\caption{Online, non-overlapping.}
		\label{fig: online non over fix 3000}
	\end{subfigure}
	\caption{Above: total weights of unique nodes visited for offline algorithms. Below: regret for online algorithms when $B=3000$.
	}\label{fig: offline algorithms}
\end{figure}

\begin{table}
\centering
\caption{Statistics for FF-TW-YT network}
\label{tab: FF-TW-YT}
\begin{minipage}{0.90\columnwidth}
\begin{tabular}{cccc}
    \toprule
        Layer&FriendFeed & Twitter & YouTube  \\
     \midrule
     \# of vertices & 5,540 & 5,702 & 663 \\
     \# of edges & 31,921 & 42,327 & 614\\
     \bottomrule
\end{tabular}
\end{minipage}
\end{table}

\begin{table*}[h]
\caption{Running time (seconds) for offline and online algorithms.}
    \begin{subtable}[h]{0.3\textwidth}
        \centering
        \begin{minipage}{0.90\columnwidth}
        \begin{tabular}{cccccc}
    \toprule
     & B=2.6k & B=2.8k & B=3.0k  \\
     \midrule
     BEG & 0.274 & 0.316 & 0.363 \\
     BEGE & 34.37  & 45.51 &  59.20 \\
     OPT & 91.16 & 98.42 &  105.63 \\
     \bottomrule
\end{tabular}
\end{minipage}
       \caption{Running time of offline algorithms for the overlapping case with different budgets B.}
       \label{tab:offline, over}
    \end{subtable}
    \hfill
    \begin{subtable}[h]{0.3\textwidth}
        \centering
        \begin{tabular}{cccc}
    \toprule
      & B=2.6k & B=2.8k & B=3.0k  \\
     \midrule
     DP & 0.038  & 0.044 &  0.050 \\
     MG & 0.008  & 0.009 &0.010\\
     OPT & 70.10  & 75.78 & 81.11 \\
     \bottomrule
\end{tabular}
        \caption{Running time of offline algorithms for the non-overlapping case with different budgets B.}
        \label{tab:offline, non-over}
     \end{subtable}
     \hfill
         \begin{subtable}[h]{0.3\textwidth}
        \centering
        \begin{tabular}{ccc}
    \toprule
     Overlapping? & \checkmark  & $\times$  \\
     \midrule
     BEG & 1.22  & NA \\
     BEGE & 60.03  & NA\\
     DP  & NA  & 0.86  \\
     OPT & 107.33 & 82.01 \\
     \bottomrule
\end{tabular}
        \caption{Per-round running time for CUCB-MAX (or CUCB-MG) with different oracles when B=3.0k.
        }
        \label{tab:online}
     \end{subtable}
     \vspace{-10pt}
     \label{tab:runing time}
\end{table*}

\textbf{Dataset and settings.} We conduct experiments on a real-world multi-layered network FF-TW-YT, which contains $m=3$ layers representing users' social connections in FriendFeed (FF), Twitter (TW) and YouTube (YT)~\cite{dickison2016multilayer}.
In total, FF-TW-YT has $6,407$ distinct vertices representing users and $74,836$ directed edges representing connections (``who follows whom") among users.
The statistics for each layer is summarized in Table~\ref{tab: FF-TW-YT}.
We transform the FF-TW-YT network $\cG(\cV,\cE)$ into a symmetric directed
network (by adding a new edge $(v,u)$ if $(u,v)\in \cE$ but $(v,u) \notin \cE$) because a user can be visited via her followers or followees.
Each edge weight is set to be $1$ and the node weights are set to be $\sigma_u\in\{0,0.5,1\}$ uniformly at random. 
Each random walker always starts from the smallest node-id in each layer and we set constraints $c_i$ equal to the total budget $B$. 
Note that in order to test the non-overlapping case, we use the same network but relabel node-ids so that they do not overlap between different layers.
To handle the randomness, we repeat $10,000$ times and present the averaged total weights of unique nodes visited for offline optimization.
We calculate the regret by comparing with the \textit{optimal} solution, which is stronger than comparing with $(1-e^{-\eta},1)$-approximate solution as defined in Eq.~(\ref{eq: alpha beta regret}).
We average over 200 independent experiments to provide the mean regret with 95\% confidence interval.
To evaluate the computational efficiency, we also present the running time for both offline and online algorithms in Table~\ref{tab:runing time}.

\noindent\textbf{Algorithms in comparison.}
For the offline setting, we present the results for Alg.~\ref{alg: layer traversal greedy} (denoted as BEG), Alg.~\ref{alg: pure greedy} (denoted as MG), BEG with partial enumeration (denoted as BEGE) and \OnlyInFull{Alg.~\ref{alg: dynamic programming}}\OnlyInShort{the dynamic programming algorithm in the supplementary material} (denoted as DP).
We provide two baselines PROP-S and PROP-W, which allocates the budget proportional to the layer size and proportional to the total weights if we allocate $B/3$ budgets to that layer, respectively.
The optimal solution (denoted as OPT) is also provided by enumerating all possible budget allocations. 
For online settings, we consider CUCB-MAX (Alg.~\ref{alg: over online}) and CUCB-MG \OnlyInFull{(Alg.~\ref{alg: non-over online})}\OnlyInShort{(Sec.~\ref{sec:cucbmg})} algorithms.
We shrink the confidence interval by $\gamma$, i.e., $\rho_{i,u,b} \leftarrow \gamma\rho_{i,u,b}$, to speed up the learning, though our theoretical regret bound requires $\gamma = 1$.
For baselines, we consider the EMP algorithm which always allocates according to the empirical mean, and the $\epsilon$-Greedy algorithm which allocates budgets according to empirical mean with probability $1-\epsilon$ and allocates all $B$ budgets to the $i$-th layer with probability $\epsilon/m$.
We also compare with the Thompson sampling (TS) method~\cite{wang2018thompson}, which uses Beta distribution $\text{Beta}(\alpha,\beta)$ (where $\alpha=\beta=1$ initially) as prior distribution for each base arm.

\noindent\textbf{Experimental results.}
We show the results for \mulane{} problems in Figure~\ref{fig: offline algorithms}.
For the offline overlapping case, both BEG and MG outperform
two baselines PROP-W and PROP-S in receiving total weights.
Although not guaranteed by the theory, BEG are empirically close to BEGE and the optimal solution (OPT).
As for the computational efficiency, in Table~\ref{tab:offline, over}, the BEG is at least two orders of magnitude (e.g.,163 times when B=3.0k) faster than BEGE and OPT. 
Combining that the reward of BEG is empirically close to BEGE and the optimal solution, this shows that BEG is empirically better than BEGE and OPT.
For the offline non-overlapping case, the results are similar, but the difference is that we have the theoretical guarantee for the optimality of DP.
For the online setting, all CUCB-MAX/CUCB-MG curves outperform the baselines.
This demonstrates empirically that CUCB-MAX algorithm can effectively learn the unknown parameters while optimizing the objective.
For the computational efficiency of online learning algorithms, since the running time for algorithms with the same oracle is similar, we present the running time for CUCB-MAX with different oracles in Table~\ref{tab:online}.
CUCB-MAX with BEG is 50 times faster than BEGE, which is consistent with our theoretical analysis.
The results for different budgets $B$ are consistent with $B=3000$, which are included in the \OnlyInFull{Appendix~\ref{appendix: expriment for diff budgets}}\OnlyInShort{supplementary material}.
Results and analysis for stationary starting distributions are also in the\OnlyInFull{ Appendix~\ref{appendix: expriment for stationary distributions}}\OnlyInShort{supplementary material}.

\section{Conclusions and Future Work}\label{sec: conclusion}
This paper formulates the multi-layered network exploration via random walks (\mulane) as a budget allocation problem, requiring that the total weights of distinct nodes visited on the multi-layered network is maximized.
For the offline setting, we propose four algorithms for \mulane{} according to the specification of multi-layered network (overlapping or non-overlapping) and starting distributions (arbitrary or stationary), each of which has a provable guarantee on approximation factors and running time.
We further study the online setting where network structure and the node weights are not known a priori.
We propose the CUCB-MAX algorithm for overlapping \mulane{} and the CUCB-MG algorithm for the non-overlapping case, both of which are bounded by a $O(\log T)$ (approximate) regret.
Finally, we conduct experiments on a social network dataset to show the empirical performance of our algorithms.

There are many compelling directions for the future study.
For example, it would be interesting to extend our problem where the decision maker can jointly optimize the starting distribution and the budget allocation.
One could also study the adaptive \mulane{} by using the feedback from
the exploration results of the previous steps to determine the exploration strategy for future steps.

\section{Acknowledgement}
The work of John C.S. Lui was supported in part by the GRF 14200420.

\bibliography{main.bib}
\bibliographystyle{icml2021}

\appendix
\OnlyInFull{
\clearpage

\section*{Supplementary Material}
The supplementary material is organized as follows. 

We first discuss how we handle the multiple random walks in Section \ref{appendix: multiple}.
We provide proofs and examples for properties of the visiting probability in Section \ref{appendix: explct form}.
Proofs of offline optimization for overlapping \mulane{} are provided in Section \ref{apdx:offline_over_proof}.
We provide the detailed budget effective greedy algorithm with partial enumeration (BEGE) and its analysis in Section \ref{appendix: enum begreedy}.
Proofs of offline optimization for non-overlapping \mulane{} are provided in Section \ref{appendix: offline non over analysis}.
We state the detailed analysis of online learning for overlapping \mulane{} in Section \ref{appendix: over online analysis}.
We state the detailed analysis of online learning for non-overlapping \mulane{} in Section \ref{appendix: non over online}.
Supplemental experiments are provided in Section \ref{apdx:exp}.

\section{Handling the Multiple Random Walkers}\label{appendix: multiple}
\begin{figure}[h]
    \centering
    \begin{subfigure}[b]{0.23\textwidth}
        \centering
        \includegraphics[width=\textwidth]{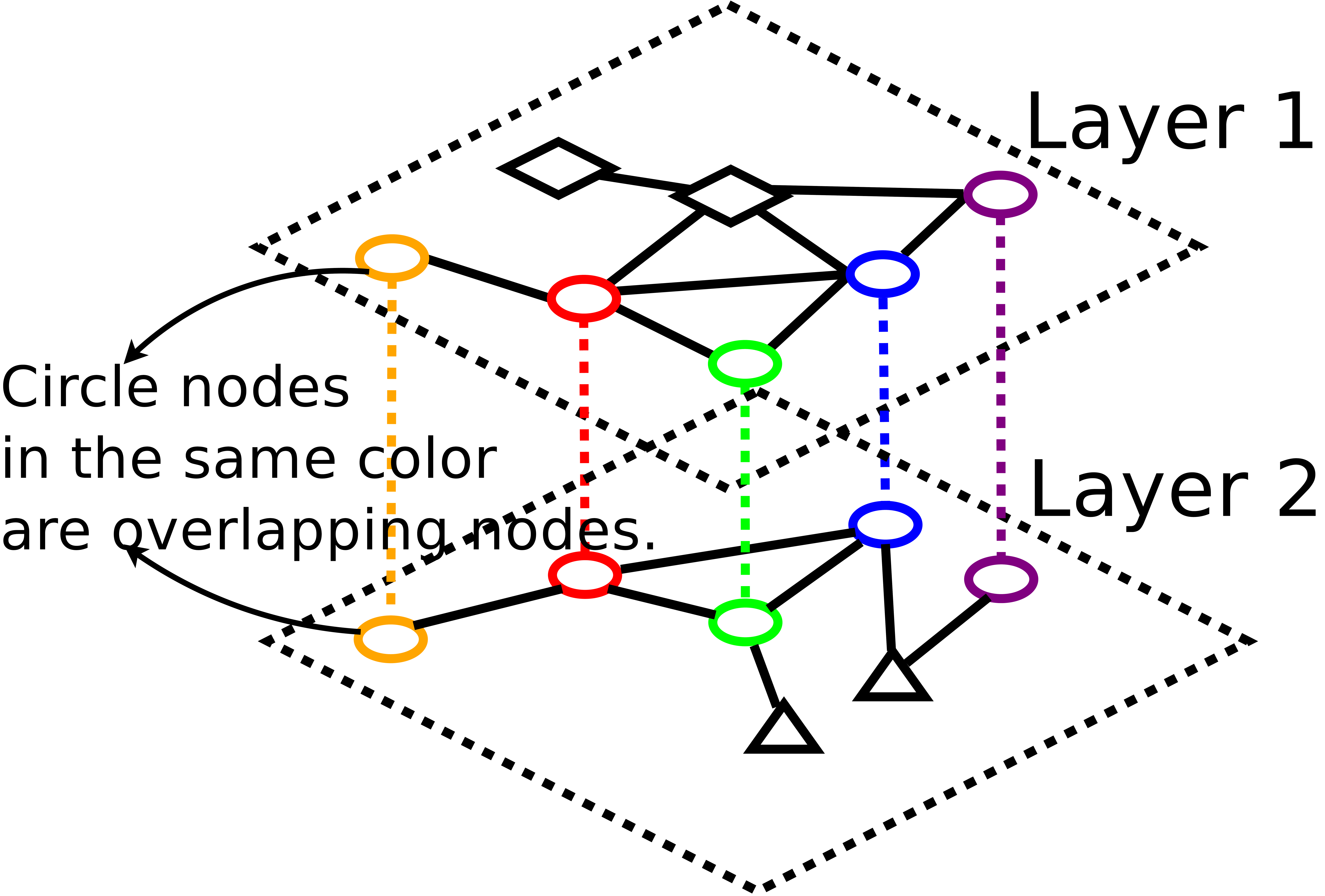}
        \caption{Overlapping.}
        \label{fig: over}
    \end{subfigure}
    \hfill
    \begin{subfigure}[b]{0.23\textwidth}
        \centering
        \includegraphics[width=\textwidth]{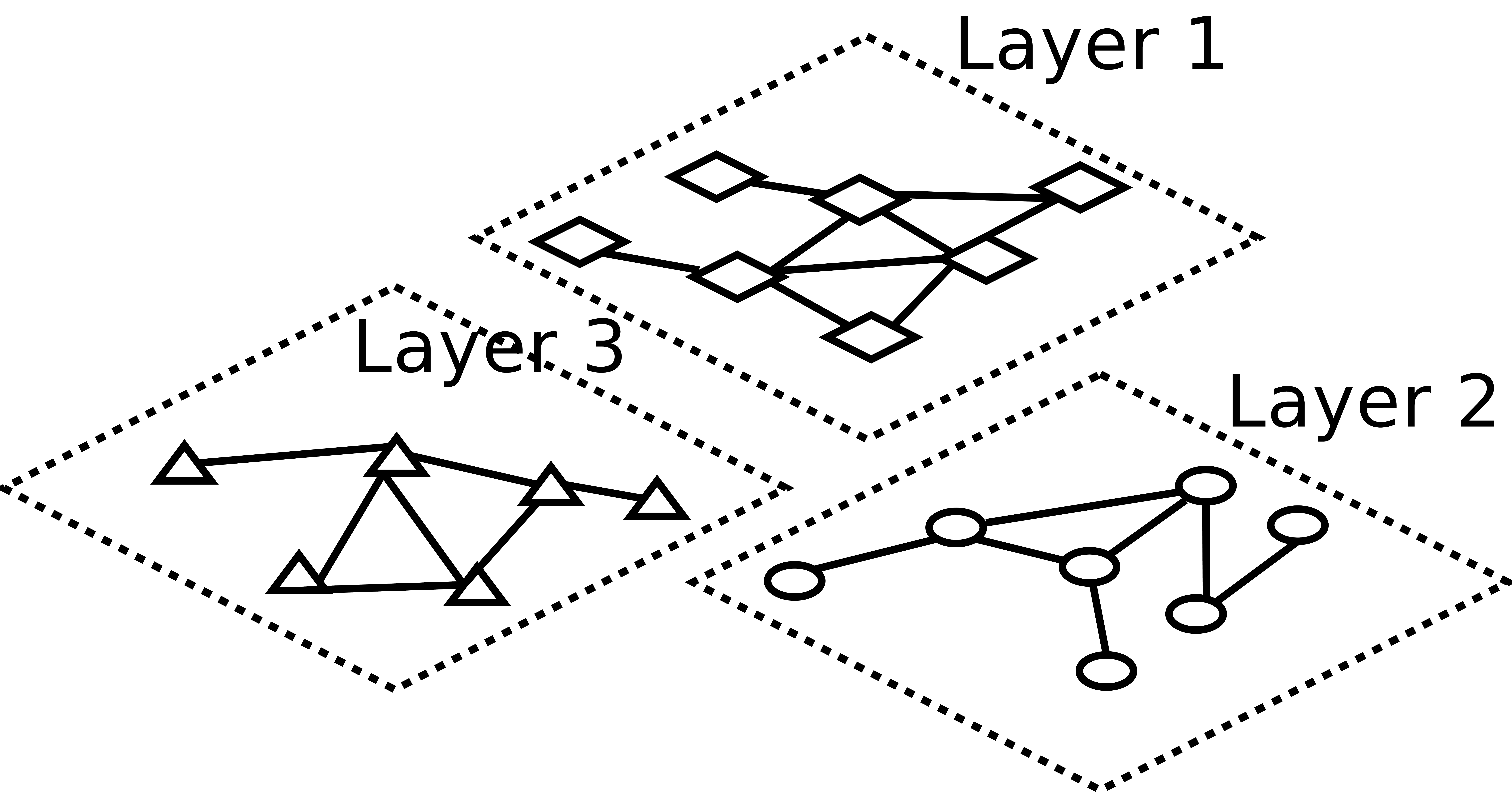}
        \caption{Non-overlapping.}
        \label{fig: non-over}
    \end{subfigure}
        \caption{Two types of multi-layered networks.}
        \label{fig: multi-layered}
\end{figure}
We can handle the scenario where each layer is explored by multiple random walkers using the bipartite coverage model. 
Without loss of generality, suppose we want to add a new random walker $W_1'$ to the layer $L_1$. 
A new node $W_1'$ can be easily added to $\cL$ representing the new random explorer, and new edges $\{(W_1', u)|u \in \cV_1\}$ are added to $\cE'$. 
Thus, we can use the same algorithms and analysis to solve the optimization problem on the newly constructed bipartite coverage graph.
\section{Proofs and Examples for Properties of the Visiting Probability
$P_{i,u}(k_i)$}\label{appendix: explct form}
\subsection{Starting From Arbitrary Distributions}\label{appendix: explct form stationary}
\lemNonDecreasing*
\begin{proof}
For analysis, we use the following equivalent formulation for $k_i \ge 2$,  
(trivially $P_{i,u}(0)=0$ and $P_{i,u}(1)=\alpha_{i,u}$),

\begin{equation}\label{eq: explicit form}\textstyle
P_{i,u}(k_i)=((\bm{\alpha_{i}})^{\top}_{-u} (\bm{\Gamma_i})_{uk_i} + \alpha_{i,u}),
\end{equation} 

where $(\bm{\Gamma_i})_{uk_i}=(\bm{I} + \bm{T_i(u)} + ... + \bm{T_i(u)}^{k_i-2})(\bm{p_i})_u$, $(\bm{p_i})_u = (\bm{P_i}{[\cdot, u]})_{-u}$, $\bm{T_i(u)}=(\bm{P_i})_{-u,-u}$,
(we use Eq. (2.67) in~\cite{kijima1997markov} to derive $\bm{\Gamma}_{uk_i}$). 
Note that $\bm{p}_{-u} \in \R^{n-1}$ is the vector obtained by deleting the $u$-th element from
$\bm{p} \in \R^n$, and $\bm{P}_{-u,-v} \in \R^{(n-1)\times (n-1)}$ is the matrix
obtained by deleting the $u$-th row and the $v$-th column from $\bm{P} \in \R^{n \times n}$.

We further derive the marginal gain of $P_{i,u}(\cdot)$ at step $k_i \ge 2$
(trivially $g_{i,u}(1)= \alpha_{i,u}$) as,
\begin{equation}\label{eq: non-over marginal gain}\textstyle
g_{i,u}(k_i) = (\bm{\alpha_i})_{-u}^{\top}\bm{T_i(u)}^{k_i-2}(\bm{p_i})_u.
\end{equation}
Now, we can show $g_{i,u}(k_i)$ is non-negative because any element of
$(\bm{\alpha_i})_{-u}^{\top}$, $\bm{T_i(u)}$ and $(\bm{p_i})_u$ are non-negative, which means
$P_{i,u}(k_i)$ is non-decreasing with respect to step $k_i$.
\end{proof}

\textbf{An example showing $g_{i,u}(k_i)$ is not monotone:}
Consider a path $P_3$ with three nodes as the $i$-th layer $\cG_i(\cV_i, \cE_i)$, where $\cV_i=\{u,v,w\}$ and $\cE_i=\{(u,v),(v,u), (v,w), (w,v)\}$. If the $W_i$ always starts from the left-most node $u$ and chooses the right-most node $w$ as our target node. Then, $g_{i,w}(1)=g_{i,w}(2)=0$ but $g_{i,w}(3)>0$ since at least three steps are needed to visit the node $w$, which shows that $g_{i,w}(k_i)$ is not always non-increasing.
\subsection{Starting From the Stationary Distribution}
\thmStationary*
\begin{proof}
	Consider any layer $L_i$ with transition probability matrix $\bm{P_i}$, if we start from the stationary distribution $\bm{\pi_i}$, the probability that node $u \in \cV_i$ is ever visited in the first $k_i$ steps is $P_{i,u}(k_i)=(\bm{\pi_i})_{-u}^{\top} (\bm{\Gamma_i})_{uk_i} + \pi_{i,u}$, where $(\bm{\Gamma_i})_{uk_i}=(\bm{I} + \bm{T_i(u)} + ... + \bm{T_i(u)}^{k-2})(\bm{p_i})_u$, $\bm{T_i(u)}=(\bm{P_i})_{-u,-u}$, $(\bm{p_i})_u = (\bm{P_i}{[\cdot, u]})_{-u}$. 
	Then the marginal gain for node $u$ is $g_{i,u}(k_i) = P_{i,u}(k_i) - P_{i,u}(k_i-1) = (\bm{\pi_i})_{-u}^{\top}\bm{T_i(u)}^{k_i-2}(\bm{p_i})_u$ for $k_i \ge 2$ and $g_{i,u}(k_i)=\pi_{i,u}$ when $k_i=1$. 
	
	Define the margin of the marginal gain as $\Delta_{i,u}(k_i)=g_{i,u}(k_i+1) - g_{i,u}(k_i)$.
	When $k=1$, $\Delta(u,1) = g_{i,u}(2) - g_{i,u}(1) = (\bm{\pi_i})_{-u}^{\top}\cdot(\bm{p_i})_u - \pi_{i,u} = \pi_{i,u} - \pi_{i,u} \bm{P}{[u,u]} - \pi_{i,u} =  -\pi_{i,u} \bm{P}{[u,u]} \le 0$. 
	When $k_i \ge 2$,   $\Delta_{i,u}(k_i)=g_{i,u}(k_i+1)-g_{i,u}(k_i)=(\bm{\pi_i})_{-u}(\bm{T_i(u)}^{k_i-1}-\bm{T_i(u)}^{k_i-2})(\bm{p_i})_u$. 
	Because $\bm{\pi_i}^{\top} \bm{P_i} = \bm{\pi_i}^{\top} $, we have $(\bm{\pi_i})_{-u}^{\top} - (\bm{\pi_i})_{-u}^{\top} \bm{P_i}(u) = \pi_{i,u} (\bm{q_i})_u^{\top}$, where $(\bm{q_i})_u^{\top}=(\bm{P_i}{[u, \cdot]})_{-u}$. 
	Thus, $\Delta_{i,u}(k_i) = -(\pi_{i,u}) (\bm{q_i})_u^{\top} \bm{T_i(u)}^{k_i-2} (\bm{p_i})_u \le 0$ because any element in $\bm{\pi_i}$, $(\bm{q_i})_u^{\top}$, $(\bm{p_i})_u$ and $\bm{P_i(u)}$ is non-negative.
\end{proof}

More interestingly, the stationary distribution $\bm{\pi}_i$ is the \textit{only} starting distribution for $\bP_i$ such that any $u \in \cV_i, k_i \in \Z_{>0}$, $g_{i,u}(k_i+1) - g_{i,u}(k_i) \le 0$ when $\bP_i$ is ergodic and there are no self loops in $\cG_i$.
 
\begin{proof}
With a little abuse of the notation, we use $\bP$ to denote the transition probability matrix $\bP_i \in \R^{n \times n}$ and let $\bP_{i,j}$ be the element in the $i$-th row and the $j$-th column.
Since $\bP$ is ergodic, according to Theorem 54 in~\cite{serfozo2009basics}, there exists a unique and positive stationary distribution $\bm{\pi}=(\pi_1, ..., \pi_n)$, i.e., $\bm{\pi}\bP =\bm{\pi}$ and $\bm{\pi} > \bm{0}$.
Any starting distribution $\bm{\alpha}$ can be represented by $\bm{\pi} + \bm{\epsilon}$, where $\bm{\epsilon}=(\epsilon_1, ..., \epsilon_n)$ is a perturbation vector, and $-\pi_j \le \epsilon_j \le 1-\pi_j, j\in [n]$.
We have the following equation for the margin of marginal gains for node $u$ in the first two steps,
\begin{align*}
\Delta_u = g_{i,u}(1) - g_{i,u}(2) &= (\pi_u + \epsilon_u) - \sum_{j \neq u}(\pi_j + \epsilon_j)\bP_{j,u} \\
&={\pi}_u\bP_{uu} + \epsilon_u - \sum_{j\neq u}\epsilon_j
\bP_{j,u}.\\
\end{align*}
Since $\cG$ has no self loops, i.e., $\bP_{u,u} = 0$,
we have $\Delta_u =  \epsilon_u - \sum_{j\neq u}\epsilon_i\bP_{j,u}$,
and we can verify that $\sum_{u \in [n]}\Delta_u = \sum_{u \in [n]}(1-\sum_{j \neq u}\bP_{u,j})\epsilon_u = 0 $.
Therefore, we have to guarantee $\Delta_u = 0$ for all $u$, otherwise there will exist a node $u$ such that $\Delta_u < 0$.
To ensure $\Delta_u=0$, we need to ensure $\bm{\epsilon}\bP = \bm{\epsilon}$ by rephrasing the equations $\Delta_u=0$ for all $u$.
Again, since there exists a unique and positive stationary distribution for $\bP$ and $\sum_{u \in [n]}\epsilon_u = 0$,
we can derive $\bm{\epsilon} = \beta \bm{\pi}$, where $\beta$ has to be $0$. 
Therefore, combined with Lemma~\ref{lem: stationary}, the stationary distribution is the only starting distribution such that for any $\bP_i$, any $u \in \cV_i, k_i \in \Z_{>0}$, $g_{i,u}(k_i+1) - g_{i,u}(k_i) \le 0$.
\end{proof}

\section{Proofs of Offline Optimization for Overlapping \mulane}\label{apdx:offline_over_proof}
\subsection{Starting from the arbitrary distribution}\label{appendix: offline proof arbitrary distribution}
\thmSubmodular*
\begin{proof}
By definition, we need to show $r_{\cG, \bm{\alpha}, \bsig}( \bm{x \wedge y}) + r_{\cG, \bm{\alpha}, \bsig}( \bm{x \vee y}) \le r_{\cG, \bm{\alpha}, \bsig}( \bm{x}) + r_{\cG, \bm{\alpha}, \bsig}( \bm{y})$ for any $\bm{x},\bm{y} \in \Z^m_{\ge 0}$, and $r_{\cG, \bm{\alpha}, \bsig}( \bm{x}) \le r_{\cG, \bm{\alpha}, \bsig}( \bm{y})$ if $\bm{x} \le \bm{y}$.

(\textbf{Monotonicity.}) By Eq.~(\ref{eq: over reward}),
$r_{\cG, \bm{\alpha}, \bsig}( \bm{k})=\sum_{v\in\cV}\sigma_v(1-\Pi_{i\in[m]}(1-P_{i,v}(k_i)))$.
Since $P_{j,v}(x_j) \le P_{j,v}(x_j+1)$, for any $j \in [m], v\in \cV$, $x_j \in \bbZ_{\ge 0}$, we have
\begin{align*}
&r_{\cG, \bm{\alpha}, \bsig}( \bm{x}+\bch_j)-r_{\cG, \bm{\alpha}, \bsig}( \bm{x})\\
&=\sum_{v\in\cV}\sigma_v((\prod_{i\neq j}\left(1-P_{i,v}(x_i)))\left(P_{j,v}(x_j+1)-P_{j,v}(k_j)\right)\right) \\
&\ge 0,
\end{align*}
for any $\bx \in \bbZ_{\ge 0}^m, j \in [m]$.
Then we can use the above inequality repeatedly to show that $r_{\cG, \bm{\alpha}, \bsig}( \bm{x}) \le r_{\cG, \bm{\alpha}, \bsig}( \bm{y})$ when $\bm{x} \le \bm{y}$.

(\textbf{Submodularity.}) For submodular property, it is sufficient to prove $(1-\Pi_{i\in[m]}(1-P_{i,v}(k_i)))$ is submodular for any $v \in \cV$, because a positive weighted sum of submodular function is still submodular.
We will rely on the following lemma to prove the submodularity of $r_{\cG, \bm{\alpha}, \bsig}( \bm{k})$.
\begin{lemma}\label{lem: lattice submodular}
Function $f: \bbZ^m_{\ge 0}\rightarrow \R$ is submodular if and only if
\begin{equation}
\begin{split}\label{ieq: lem lattice submodular}
&f(\bm{x}+\bm{\chi}_i) - f(\bm{x}) \ge f(\bm{x} + \bm{\chi}_j +\bm{\chi}_i) - f(\bm{x} + \bm{\chi}_j),\quad \\
&\text{for any $\bm{x} \in \Z^m_{\ge 0}$ and $i\neq j$.}
\end{split}
\end{equation}
\end{lemma}

\begin{proof}[Proof of Lemma~\ref{lem: lattice submodular}]

\textbf{(If part.)}
We first prove, if inequality~(\ref{ieq: lem lattice submodular}) holds, the following inequality holds, 
\begin{equation}\label{ieq: DR-submodular}
f(\bx + \bch_i) - f(\bx) \ge f(\by + \bch_i) - f(\by) 
\end{equation}
for any $\bx \le \by$ and $i \in [m]$ such that $x_i = y_i$.

Let $I_0 = \{i\in[m]: x_i = y_i\}, I_1 = \{i \in [m]: x_i < y_i\}$.
For any $\bx \le \by$, we denote the elements in $I_1$ by $i_1, ..., i_s$ and write $\by = \bx + \sum_{j=1}^{s} \alpha_j \bch_{i_j} $, where $\alpha_j=y_j-x_j$. 
For any $i \in I_0$, we have
\begin{align}\label{ieq: submodular ratio}
&f(\bx + \bch_i) - f(\bx) \nonumber \\
&\ge f(\bx + \bch_{i_1} + \bch_i) - f(\bx + \bch_{i_1}) \nonumber\\
& \ge f(\bx+2\bch_{i_1}+\bch_i) - f(\bx+2\bch_{i_1}) \nonumber\\
&\ge ... \nonumber\\
&\ge f(\bx+\alpha_{i_1}\bch_{i_1}+\bch_i) - f(\bx+\alpha_{i_1}\bch_{i_1})\nonumber\\
& \ge f(\bx+\alpha_{i_1}\bch_{i_1}+\bch_{i_2}+\bch_i) - f(\bx+\alpha_{i_1}\bch_{i_1}+\bch_{i_2}) \nonumber\\
&\ge f(\bx+\alpha_{i_1}\bch_{i_1}+\alpha_{i_2}\bch_{i_2}+\bch_i) - f(\bx+\alpha_{i_1}\bch_{i_1}+\alpha_{i_2}\bch_{i_2}) \nonumber\nonumber\\
&\ge ... \nonumber\\
&\ge f(\bx + \sum_{j=1}^{s} \alpha_j \bch_{i_j}+\bch_i) - f(\bx + \sum_{j=1}^{s} \alpha_j \bch_{i_j})\nonumber\\
&= f(\by+\bch_i) - f(\by).
\end{align}

Then for any $i \in I_0$ and $a \in \bbZ_{\ge 0}$, we have
\begin{align}\label{ieq: lattice submodular property}
&f(\bx+a\bch_i) - f(\bx) \nonumber\\
&= \sum_{j=1}^{a}\left(f(\bx+j\bch_i) - f(\bx+(j-1)\bch_i)\right) \nonumber\\
&\ge  \sum_{j=1}^{a}\left(f(\by+j\bch_i) - f(\by+(j-1)\bch_i)\right) \nonumber\\
&\ge f(\by+a\bch_i) - f(\by).
\end{align}
The first inequality holds because of Inequality~(\ref{ieq: submodular ratio}), the fact $\bx+(j-1)\bch_i \le \by+(j-1)\bch_i$ and $x_i + j-1 = y_i + j-1$.

Then for any $\bx,\by \in \bbZ_{\ge 0}^m$, let $I_2=\{i\in[m]: x_i > y_i\}=\{i_1, ..., i_s\}$. We have  
\begin{align*}
&f(\bx) - f(\bx \wedge \by) \\
&= \sum_{l=1}^s\Bigg(f\Bigg(\bx \wedge \by + \sum_{j=1}^l{(x_{i_j}-y_{i_j})\bch_{i_j}}\Bigg) \\
&- f\Bigg(\bx \wedge \by + \sum_{j=1}^{l-1}{(x_{i_j}-y_{i_j})\bch_{i_j}}\Bigg)\Bigg) \\
&\ge \sum_{l=1}^s\Bigg(f\Bigg(\by + \sum_{j=1}^l{(x_{i_j}-y_{i_j})\bch_{i_j}}\Bigg) \\
&- f\Bigg(\by + \sum_{j=1}^{l-1}{(x_{i_j}-y_{i_j})\bch_{i_j}}\Bigg)\Bigg) \\
&=f(\bx \vee \by) - f(\by)  .
\end{align*}
The inequality is derived from Inequality~(\ref{ieq: lattice submodular property}) because $ \by + \sum_{j=1}^{l-1}{(x_{i_j}-y_{i_j})\bch_{i_j}} \ge \bx \wedge \by + \sum_{j=1}^{l-1}{(x_{i_j}-y_{i_j})\bch_{i_j}}$ for any $0 \le l \le s$ and $(\bm{x}\wedge \bm{y})_{i_l} = \by_{il}$, which concludes the if part.

\textbf{(Only if part.)}
Assume $f$ is submodular, let $\ba = \bx + \bch_i, \bb = \bx + \bch_j, i\neq j$, we have $\ba \vee \bb=\bx+\bch_i+\bch_j$, $\ba \wedge \bb = \bx$.
$ f(\bm{x}+\bm{\chi}_i) - f(\bm{x}) = f(\ba)- f(\ba \wedge \bb) \ge f(\ba \vee \bb) - f(\bb) = f(\bm{x} + \bm{\chi}_j +\bm{\chi}_i) - f(\bm{x} + \bm{\chi}_j)$.

\end{proof}

Then, by Lemma~\ref{lem: lattice submodular} and the explicit formula of the reward function given by Eq.~(\ref{eq: over reward}), we can prove $g(\bx, v) - g(\bx+\bch_l, v) \ge g(\bx + \bch_j, v)- g(\bx+\bch_l+\bch_j, v)$ for any $\bx \in \bbZ_{\ge 0}^m$, $v\in \cV$ and $l\neq j \in [m]$, where $g(\bm{x}, v)=\Pi_{i\in[m]}(1-P_{i,v}(x_i))$.
This holds due to the fact that the left hand side equals to $\left(\Pi_{i \in [m]\setminus \{j,l\}}(1-P_{i,v}(x_i))\right)(1-P_{j,v}(k_j))(P_{l,v}(k_l+1)-P_{l,v}(k_l))$ and the right hand side equals to $\left(\Pi_{i \in [m]\setminus \{j,l\}}(1-P_{i,v}(x_i))\right)(1-P_{j,v}(k_j+1))(P_{l,v}(k_l+1)-P_{l,v}(k_l))$, and
the left hand side is larger or equal to the right hand side because $(1-P_{j,v}(k_j)) \ge (1-P_{j,v}(k_j+1))$.
By summation over all nodes $v\in \cV$ with node weights $\sigma_v \in [0,1]$, we can prove the reward function is submodular.
\end{proof}

\thmOverSol*

\begin{proof}
For theoretical analysis, we first give a modified version of Alg.~\ref{alg: layer traversal greedy} in Alg.~\ref{alg: equivalent layer traversal greedy}. 
Both algorithms provide the same solution $\bk$ given the same problem instance $(\cG, \bm{\alpha}, \bsig, B, \bc)$.
To see this fact, Alg.~\ref{alg: equivalent layer traversal greedy} considers invalid tentative allocations $(i^*, b^*)$ (adding it will exceed the total budget constraint $B$) and remove them in line~\ref{line: equivalent alg remove} of Alg.~\ref{alg: equivalent layer traversal greedy}, while Alg.~\ref{alg: layer traversal greedy} only considers valid allocations by directly removing invalid allocations in advance in line~\ref{line: remove} of Alg.~\ref{alg: layer traversal greedy}. 

With a little abuse of the notation, we use $r(\bm{k})$ to represent the reward $r_{\cG, \bm{\alpha}, \bsig}( \bm{k})$ of a given problem instance. 
Let $\bm{k^*} \in \Z_{\ge 0}^m$ denote the optimal budget allocation and $\bm{k}^j\in \Z_{\ge 0}^m$ denote the budget allocation before entering the $j$-th iteration of the while loop (line~\ref{line: equivalent alg while begin}-\ref{line: equivalent alg remove}) in Alg.~\ref{alg: equivalent layer traversal greedy}.
After entering the $j$-th iteration, the algorithm tries to extend the current budget allocation $\bm{k}^j$ by choosing the pair $(i^*, b^*)$ in line~\ref{line: equivalent alg largest margin}, which we denote as $(i^j, b^j)$. 
Let $s$ be the first iteration we can not extend the current solution, i.e., $\bm{k}^s=\bm{k}^{s+1}$ and $\bm{k}^j< \bm{k}^{j+1}$ for $j = 1, ..., s-1$.
If we can always extend the current solution, we set $s$ to be $B+1$.
For analysis, we temporarily add $(i^s, b^s)$ (in the algorithm, this pair is removed by line~\ref{line: equivalent alg remove} in the $s$-th iteration) to form a "virtual" budget allocation $\bk^{s+1}$=$\bk^{s}+b^s\bch_{i^s}$. Let $\bk_g \in \Z_{\ge 0}^m$ denote the solution returned by Alg.~\ref{alg: equivalent layer traversal greedy}.

\begin{algorithm}[t]
\caption{Equivalent Budget Effective Greedy Algorithm (BEG) for the Overlapping MuLaNE.}\label{alg: equivalent layer traversal greedy}
\begin{algorithmic}[1]
	    \INPUT{Network $\cG$, starting distributions $\bm{\alpha}$, node weights $\bsig$, budget $B$, constraints $\bc$.}
	    \OUTPUT{Budget allocation $\bk$.}
	    \STATE Compute visiting probabilities $(P_{i,u}(b))_{i\in[m], u\in\cV, b \in [c_i]}$ according to Eq.~(\ref{eq: simple form}).
	    	\label{line: equivalent alg computePiub}
	    \STATE $\bk \leftarrow$ BEG($(P_{i,u}(b))_{i\in[m], u\in\cV, b \in [c_i]}$, $\bsig$, $B$, $\bc$).\label{line: equivalent alg return}
	    \algrule
		\FUNCTION{BEG($(P_{i,u}(b))_{i\in[m], u\in\cV, b \in [c_i]}$, $\bsig$, $B$, $\bc$)}
		\STATE Let $\bm{k}\defeq(k_1, ..., k_m) \leftarrow \bm{0}$, $K\leftarrow B$.
		\STATE Let $\cQ \leftarrow \{(i,b_i) \, | \, i \in [m], 1\le b_i \le  c_i\}$.
		\WHILE {$K > 0$ and $\cQ \neq \emptyset$ \label{line: equivalent alg while begin}}
	 \STATE $(i^*,b^*) \leftarrow \argmax_{(i,b)\in \cQ} \delta(i,b,\bk)/b $ \label{line: equivalent alg largest margin}\algorithmiccomment{Eq.~(\ref{eq: marginal gain delta})}
		\IF {$b^* \le K$}
		 \STATE $k_{i^*} \leftarrow k_{i^*} + b^*$, $K \leftarrow K - b^*$.
		 \STATE Modify pairs $(i^*,b) \in \cQ$ to $(i^*, b - b^*)$.
		 \STATE Remove paris $(i^*,b) \in \cQ$ such that $b \le 0$.
		
		\ELSE 
		\STATE Remove $(i^*, b^*)$ from $\cQ$.\label{line: equivalent alg remove}
		\ENDIF
		\ENDWHILE
		\FOR{$i \in [m]$\label{line: equivalent alg for begin}}
		\STATE \textbf{if} $r_{\cG, \bm{\alpha}, \bsig}( c_i\bm{\chi}_i) > r_{\cG, \bm{\alpha}, \bsig}( \bm{k})$, \textbf{then} $\bk \leftarrow c_i \bch_i.$ \label{line: equivalent alg for end}
        \ENDFOR
		\textbf{return} $\bk\defeq(k_1, ..., k_m)$.
		\ENDFUNCTION
\end{algorithmic}
\end{algorithm}

We first introduce lemmas describing two important properties given by the submodularity over the integer lattice.

\begin{lemma}\cite{soma2014optimal}.\label{lem: property2}
	Let $f: \bbZ_{\ge 0}^m \rightarrow \bbR$ be a submodular function. For any $\bx, \by \in \bbZ_{\ge 0}^m $, we have,
	\begin{equation}
	\begin{split}
	&f(\bx \vee \by)\\
	&\le f(\bx) + \sum_{i \in \text{supp}^+(\by - \bx)} \pl f\left(\bx+(y_i - x_i)\bch_i\right) - f(\bx) \pr .
	\end{split}
	\end{equation}
\end{lemma}

\begin{lemma}\cite{soma2014optimal}.
	Let $f: \bbZ_{\ge 0}^m \rightarrow \bbR$ be a monotone submodular function. For any $\bx, \by \in \bbZ_{\ge 0}^m $ with $\bx \le \by$ and $i \in [m]$ we have,
	\begin{equation}
	f(\bx \vee k\bch_i) - f(\bx)\ge f(\by \vee k\bch_i) -f(\by).
	\end{equation}
\end{lemma}

Then, we have the following lemma.
\begin{lemma} \label{lem: greedy 1}
For $j=1, ..., s$, 
\begin{equation}\label{ieq: layer greedy 1}
r(\bm{k}^{j+1})\ge(1-\frac{b^j}{B})r(\bm{k}^j) + \frac{b^j}{B}r(\bm{k^*}).
\end{equation}
\end{lemma}
\begin{proof}[Proof of Lemma~\ref{lem: greedy 1}]
This is because
\begin{align}
&r(\bk^*) \nonumber\\
&\le r(\bk^*\vee \bk^j) \nonumber\\
&\le r(\bk^j) + \sum_{i \in \text{supp}^+(\bk^*-\bk^j)}\left(r(\bk^j +  (k^*_i - k^j_i)\bch_i) - r(\bk^j)\right)\nonumber\\
&= r(\bk^j) + \sum_{i \in \text{supp}^+(\bk^*-\bk^j)}\left(r(\bk^j + \alpha_i\bch_i) - r(\bk^j)\right)\tag{Let $\alpha_i=k^*_i-k^j_i$}\nonumber\\
&\le r(\bk^j) + \sum_{i \in \text{supp}^+(\bk^*-\bk^j)}\left(\alpha_i\frac{r(\bk^{j+1}) - r(\bk^j)}{b^j}\right)\nonumber\\
&\le r(\bk^j) + B\left(\frac{r(\bk^{j+1}) - r(\bk^j)}{b^j}\right).\label{ieq: layer greedy 1 plus}
\end{align}
The second inequality comes from Lemma \ref{lem: property2}, the third inequality holds because of the greedy procedure in line~\ref{line: equivalent alg largest margin} and the last inequality holds because $\sum_{i \in \text{supp}^+(\bk^*-\bk^j)}\alpha_i \le B$. 
By rearranging terms, Inequality~(\ref{ieq: layer greedy 1}) holds.
\end{proof}

Next, We can prove the following lemma.
\begin{lemma}\label{lem: greedy 2}
For $l = 1, ..., s$, 
\begin{equation}
\begin{split}
    r(\bk^{l+1}) &\ge r(\bm{k}^1)\Pi_{j=1}^l\left(1-\frac{b^{j}}{B}\right)\\
    &+r(\bk^*)\left(1-\Pi_{j=1}^l\left(1-\frac{b^{j}}{B}\right)\right).
\end{split}
\end{equation}
\end{lemma}
\begin{proof}[Proof of Lemma~\ref{lem: greedy 2}]
We can prove this lemma by induction on $l$. 
When $l=1$, the lemma holds due to Lemma~\ref{lem: greedy 1}.
Assume that the lemma holds for $l-1$, we have the following inequality holds,

\begin{align}\label{ieq: layer greedy 2}
&r(\bk^{l+1}) \nonumber\\
&\ge (1-\frac{b^l}{B})r(\bk^l) + \frac{b^l}{B}r(\bk^*) \nonumber \\
&\ge (1-\frac{b^l}{B})r(\bm{k}^1)\Pi_{j=1}^{l-1}\left(1-\frac{b^{j}}{B}\right) \nonumber\\ 
&+(1-\frac{b^l}{B})r(\bk^*)\left(1-\Pi_{j=1}^{l-1}\left(1-\frac{b^{j}}{B}\right)\right)+\frac{b^l}{B}r(\bk^*)\nonumber\\
&= r(\bm{k}^1)\Pi_{j=1}^l\left(1-\frac{b^{j}}{B}\right)
    +r(\bk^*)\left(1-\Pi_{j=1}^l\left(1-\frac{b^{j}}{B}\right)\right),
\end{align}
where the first inequality is due to Lemma~\ref{lem: greedy 1} by setting $j=l$, and the second inequality is by the assumption for $l-1$.
By induction, Lemma~\ref{lem: greedy 2} holds.

\end{proof}

We then consider the following cases.

\textbf{Case 1.}
Suppose the total budget used for $\bk^s$ is larger or equal to $\eta B$, i.e., $\sum_{j=1}^{s-1}b^j \ge \eta B$, where $\eta \in [0,1]$. 

We have the following inequality.
\begin{equation}\label{ieq: greedy case 1}
r(\bk^{s}) \ge (1-e^{-\eta})r(\bk^*).
\end{equation}

This is due to Lem.~\ref{lem: greedy 2} by setting $l={s-1}$, combined with the fact that $\bk^1=\bm{0}$ and $\Pi_{j=1}^{s-1}\left(1-\frac{b^{j}}{B}\right)\le e^{-\eta}$. 
The later fact holds because, 
\begin{align*}
\log\left(\Pi_{j=1}^{s-1}\left(1-\frac{b^{j}}{B}\right)\right)&=({s-1})\sum_{j=1}^{s-1}\frac{1}{{s-1}}\log\left(1-\frac{b^{j}}{B}\right)\\
&\le ({s-1})\log\left(1-\frac{1}{{s-1}}\sum_{j=1}^{s-1}\frac{b^{j}}{B}\right)\\
&\le ({s-1})\log(1-\frac{\eta}{{s-1}}),
\end{align*}
where the first inequality holds because of the Jensen's Inequality~\cite{raginsky2013concentration} and the second inequality holds because  $\sum_{j=1}^{s-1}\frac{b^{j}}{B} \ge \eta$. Then we can easily check $\Pi_{j=1}^{s-1}\left(1-\frac{b^{j}}{B}\right)\le (1-\frac{\eta}{s-1})^{s-1} \le e^{-\eta}$.

\textbf{Case 2.}
Suppose the total budget $\sum_{j=1}^{s-1}b^j \le \eta B$. Then, we have $b^s > (1-\eta)B$.
We can prove the following inequality holds,
\begin{equation}\label{ieq: greedy case 2}
    r(\bk_g) \ge (1-\frac{1}{2-\eta})r(\bk^*).
\end{equation}

This is due to Inequality~(\ref{ieq: layer greedy 1 plus}), we have 
\begin{align*}
    r(\bk^*) &\le r(\bk^s) + B\left(\frac{r(\bk^{s+1}) - r(\bk^s)}{b^s}\right) \\
    & \le  r(\bk^s) + \left(\frac{r(\bk^{s+1}) - r(\bk^s)}{1-\eta}\right) \\
    & \le (1+ \frac{1}{1-\eta}) r(\bk_g),
\end{align*}
where the second inequality is due to $b^s > (1-\eta)B$ and the last equality is due to the fact $r(\bk^{s+1})-r(\bk^s) \le r(\bk_g)$, $r(\bk^s) \le r(\bk_g)$.
To see the above fact, we assume without loss of generality the pair $(i^s, b^s)$ improves $\bk^s$ towards the optimal budget allocation, i.e., $k^{s+1}_{i_s}=b^s+k^s_{i^s} \le k^*_{i^s}$. 
Otherwise, if $b^s+k^s_{i^s} > k^*_{i^s}$, we can safely delete $(i^s, b^s)$ in the queue $\cQ$ and does not affect the greedy solution, the optimal solution and the analysis.
Let $r_{max}=\max_{i\in[m]}r(c_i\bch_i)$, we have $r(\bk^{s+1}) - r(\bk^s) \le r(\bk^{s} \vee k^*_{i^s}\bch_{i^s}) - r(\bk^s)\le r(k^*_{i^s}\bch_{i^s}) - r(\bm{0})\le r_{\max} \le r(\bk_g)$.
Also, we can obtain that $r(\bk^s) \le r(\bk_g)$ since $\bk^s \le \bk_g$.
By rearranging the terms, Inequality~(\ref{ieq: greedy case 2}) holds.

Combining Inequality~(\ref{ieq: greedy case 1}) and (\ref{ieq: greedy case 2}), we have
\begin{align*}
    r(\bk_g) &\ge \min_{\eta \in [0,1]}\max{(1-e^{-\eta}, 1-\frac{1}{2-\eta})}r(\bk^*) \\
    &\ge (1-e^{-\eta})r(\bk^*),
\end{align*}
where $\eta$ is the solution for equation $e^\eta = 2-\eta$.

\end{proof}

\subsection{Starting From the Stationary Distribution}
\thmDRSubmodular*
\begin{proof}
By definition, we need to show $r_{\cG, \bm{\pi}, \bsig}( \bm{y} + \bm{\chi_j}) - r_{\cG, \bm{\pi}, \bsig}( \bm{y}) \le r_{\cG, \bm{\pi}, \bsig}( \bm{x}+\bm{\chi_j)} - r_{\cG, \bm{\pi}, \bsig}( \bm{x})$, and $r_{\cG, \bm{\pi}, \bsig}( \bm{x}) \le r_{\cG, \bm{\pi}, \bsig}( \bm{y})$ for any $\bm{x} \le \bm{y}$, $j\in [m]$.

We first introduce the following lemma to help us to prove the DR-submodularity.
\begin{lemma}\label{lem: DR submodular}
	Function $f$ is DR-submodular if and only if 
	\begin{equation}
	\begin{split}
	 &f(\bm{x}+\bm{\chi}_i) - f(\bm{x}) \ge f(\bm{x} + \bm{\chi}_j +\bm{\chi}_i) - f(\bm{x} + \bm{\chi}_j), \quad \\
	 &\text{for any $\bm{x} \in \Z^m_{\ge 0}$.}
	\end{split}
	\end{equation}
\end{lemma}
\begin{proof}[Proof of Lemma~\ref{lem: DR submodular}]
\textbf{(If part.)} We can easily check this direction holds by using the similar argument for Inequality~(\ref{ieq: submodular ratio}), where the only difference is we consider $i \in [m]$ instead of $i \in I_0$.

\textbf{(Only if part.) } We can set $\by \defeq \bx' + \bch_j$, $\bx \defeq \bx' + \bch_i$, for any $i,j\in[m], \bx' \in \bbZ_{\ge 0}^m$, and use Inequality~(\ref{ieq: DR-submodular}) to show the only if part holds.
\end{proof}
	Since $r_{\cG, \bm{\alpha}, \bsig}( \bk)$ is submodular for arbitrary starting distributions $\bm{\alpha}$, Inequality~(\ref{ieq: lem lattice submodular}) holds.  
	Then consider any layer $j \in [m]$ and budget allocation $\bm{x} \in \Z_{\ge 0}^m$, it is sufficient to show  $r_{\cG, \bm{\pi}, \bsig}( \bm{x}+2\bm{\chi_j)} - r_{\cG, \bm{\pi}, \bsig}( \bm{x}+\bm{\chi_j)} \le r_{\cG, \bm{\pi}, \bsig}( \bm{x} + \bm{\chi_j}) - r_{\cG, \bm{\pi}, \bsig}( \bm{x})$. 
	Since 
	the left hand side minus the right hand side equals to $\sum_{v\in\cV}\sigma_v[\left(\Pi_{i\in[m], i\neq j}\left(1-P_{i,v}\left(x_i\right)\right)\right) \cdot (P_{j,v}(x_j) - 2P_{j,u}(x_j+1) + P_{j,v}(x_j+2) )]$, 
	we only need to show $\sum_{v\in\cV}\sigma_v[\left(\Pi_{i\in[m], i\neq j}\left(1-P_{i,v}\left(x_i\right)\right)\right)$$\cdot (g_{j,v}(x_j+2)-g_{j,v}(x_j+1))] \le 0$. 
	The above inequality holds because of Lemma \ref{lem: stationary}.
\end{proof}

\thmOverStationarySol*
\begin{proof}
We can observe that line~\ref{line: equivalent alg largest margin} always select the pair $(i^*, b^*)$ with $b^*=1$ because $r_{\cG, \bm{\alpha}, \bsig}( \bm{k}+b\bm{\chi}_i) - r_{\cG, \bm{\alpha}, \bsig}( \bm{k}+(b-1)\bm{\chi}_i) \le r_{\cG, \bm{\alpha}, \bsig}( \bm{k}+\bm{\chi}_i) - r_{\cG, \bm{\alpha}, \bsig}( \bm{k})$ for arbitrary $\bk$, $i$ and $b$.
Thus, we have $s=B+1$ and by the similar argument for Inequality~(\ref{ieq: greedy case 1}), we have $r(\bk_g)=r(\bk^{B+1})\ge (1-1/e)r(\bk^*)$, which completes the proof.
\end{proof}

\subsection{Efficiently Evaluating the Reward Function}\label{appendix: efficient eval}
One key issue to derive the budget allocation is to efficiently evaluate the reward function $r_{\cG, \bm{\alpha}, \bsig}( \bm{k})$ and its marginal gains $\delta(i,b,\bk)$.
Since we need to repetitively use the visiting probabilities $P_{i,u}(k_i)$, 
we pre-calculate $P_{i,u}(k_i)$ in $O(m\norm{\bc}_{\infty}n_{max}^3)$ time in our algorithms based on Eq.~(\ref{eq: simple form}), 
where $n_{max}=\max_{i}|\cV_i|$.
For the overlapping case, given the current budget allocation $\bk$, 
we maintain a value $p_u=\Pi_{i \in [m]}{\prts{1-P_{i, u}(k_i)}}$ for each node $u \in \cV$.
The marginal gain $\delta(i,b,\bk)=\sum_{u \in \cV}\sigma_u \brkt{p_u\prts{1-\frac{1-P_{i,u}(k_i+b)}{1-P_{i,u}(k_i)}}}/b$ can be evaluated in $O(n_{max})$ time.
Then, we update all $p_u = p_u\prts{\frac{1-P_{i,u}(k_i+b)}{1-P_{i,u}(k_i)}}$ in $O(n_{max})$ after we allocate $b$ more budgets to layer $i$.
Therefore, we can use $O(n_{max})$ in total to evaluate $\delta(i,b, \bk)$.
In practice, lazy evaluation~\cite{krause2008efficient} and parallel computing can be used to further accelerate our algorithm. 

\section{Budget Effective Greedy Algorithm With Partial Enumeration and Its Analysis}\label{appendix: enum begreedy}
\subsection{Algorithm}\label{appendix: enum begreedy alg}
The algorithm is shown in Alg.~\ref{alg: begreedy enum}.

\begin{algorithm}[t]
\caption{Budget Effective Greedy Algorithm with Partial Enumeration (BEGE).}\label{alg: begreedy enum}
\begin{algorithmic}[1]
		\INPUT{Graph $\cG$, starting distributions $\bm{\alpha}$, node weights $\bsig$, budget $B$, constraints $\bc$.}
		\OUTPUT{Budget allocation $\bk$.}
	    \STATE Compute visiting probabilities $(P_{i,u}(b))_{i\in[m], u\in\cV, b \in [c_i]}$ according to Eq.~(\ref{eq: over reward}).
	    \STATE $\bk \leftarrow$ BEGE($(P_{i,u}(b))_{i\in[m], u\in\cV, b \in [c_i]}$, $\bsig$, $B$, $\bc$).
	    \algrule
		\FUNCTION{BEGE($(P_{i,u}(b))_{i\in[m], u\in\cV, b \in [c_i]}$, $\bsig$, $B$, $\bc$)}
			\STATE Let $\bk_{max} \leftarrow \bm{0}$.
			\STATE $\cS \leftarrow \{\bk=(k_1, ..., k_m)|0\le k_i \le c_i, \sum_{i \in [m]}k_i \le B, \sum_{i \in [m]}\mathbb{I}\{k_i > 0 \} \le 3\}$.
			\algorithmiccomment{$\cS$ contains all partial solutions which allocate partial budgets to at most three layers}
		\FOR{$\bk \in \cS$}
		    \STATE $K\leftarrow B-\sum_{i \in [m]}k_i$.
		    \STATE Let $\cQ \leftarrow \{(i,b_i) \, | \, i \in [m], 1\le b_i \le  c_i-k_i\}$.
		    \WHILE {$K > 0$ and $Q \neq \emptyset$ \label{line: enum while begin}}
		        \STATE $(i^*,b^*) \leftarrow \argmax_{(i,b)\in \cQ} \delta(i,b,\bk)/b $. \label{eq: enum largest margin} \algorithmiccomment{Eq.~(\ref{eq: marginal gain delta})}
		        \IF {$b^* \le K$}
		            \STATE $k_{i^*} \leftarrow k_{i^*} + b^*$, $K \leftarrow K - b^*$.
		            \STATE Modify all pairs $(i^*,b) \in \cQ$ to $(i^*, b - b^*)$.
		            \STATE Remove all pairs $(i^*,b) \in \cQ$ such that $b \le 0$.
		        \ELSE \STATE Remove $(i^*, b^*)$ from $\cQ$.
		            \label{line: enum remove}
		        \ENDIF
	    	\ENDWHILE  
	    	\STATE \textbf{if} $r_{\cG, \bm{\alpha}, \bsig}( \bm{k}) > r_{\cG, \bm{\alpha}, \bsig}( \bm{k_{max}})$, \textbf{then} $\bm{k}_{max} \leftarrow \bk$.
		\ENDFOR
		\textbf{return} $\bk_{max}\defeq(k_1, ..., k_m)$.
		\ENDFUNCTION
		\end{algorithmic}
\end{algorithm}

\subsection{Analysis}\label{appendix: enum begreedy analysis}
\begin{theorem}
The Algorithm~\ref{alg: begreedy enum} obtains a $(1-1/e)$-approximate solution to the overlapping \mulane{} with arbitrary starting distributions.
\end{theorem}

\begin{proof}
Suppose that we start from a partial solution $\bk^1 \in \bbZ_{\ge 0}^m \neq \bm{0}$.
Let us first reorder the optimal solution $\bk^*$ according to a non-increasing marginal gain ordering.
Namely, the marginal gain of the pair $(s^*_1, b^*_1)$ with respect to the empty pair is the highest among all other pairs, 
the marginal gain of the pair $(s^*_2, b^*_2)$ with respect to the solution consisting of pair $(s^*_1, b^*_1)$ is the highest among all remaining pairs, and so on.
To be concrete, $\bk^* = \{(s^*_1, b^*_1), ..., (s^*_m, b^*_m)\}$, where $s^*_i \in [m] \setminus \{ s^*_1, ..., s^*_{i-1}\}$ is selected to maximize the following equation:
\begin{equation}
\left(r\left(\sum_{j=1}^{i-1}b^*_j\bch_{s^*_j} + {b^*}_{s^*_i}\bch_l \right) -  r\left(\sum_{j=1}^{i-1}b^*_j\bch_{s^*_j}\right)\right).
\end{equation} 

Then, we try to bound the "virtual" marginal gain $\Delta^s = r(\bk^{s+1}) - r(\bk^s)$, and recall $s$ is the first iteration we can not extend the current solution.
Without loss of generality, we assume the virtual pair $(i^s, b^s)$ improves $\bk^s$ towards the optimal budget allocation, i.e., $b^s+k^s_{i^s} \le k^*_{i^s}$. 
Otherwise, if $b^s+k^s_{i^s} > k^*_{i^s}$, we can safely delete $(i^s, b^s)$ in the queue $\cQ$ and does not affect the greedy solution, the optimal solution and the analysis.

If we start from the initial solution $\bk^1$ such that $\bk^1$ matches $(s^*_1, b^*_1), (s^*_2, b^*_2), (s^*_3, b^*_3)$, i.e., ${\bk^1}_{s^*_i} = b^*_i$ for $i=1,2,3$, we have $\Delta^s = r(\bk^{s+1}) - r(\bk^s) \le r(\bk^{s} \vee k^*_{i^s}\bch_{i^s}) - r(\bk^s) \le r(k^*_{i^s}\bch_{i^s}) - r(\bm{0}) \le r({b^*_1}\bch_{s^*_1})$.
Moreover, $\Delta^s = r(\bk^{s+1}) - r(\bk^s) \le r(\bk^{s} \vee k^*_{i^s}\bch_{i^s}) - r(\bk^s) = r(\bk^{s} \vee b^*_1\bch_{s^*_1} \vee k^*_{i^s}\bch_{i^s}) - r(\bk^s \vee b^*_1\bch_{s^*_1}) \le r( b^*_1\bch_{s^*_1} \vee k^*_{i^s}\bch_{i^s}) - r( b^*_1\bch_{s^*_1}) \le r( b^*_1\bch_{s^*_1} +  b^*_2\bch_{s^*_2}) - r( b^*_1\bch_{s^*_1}).$
Similarly, we have $\Delta^s \le r( b^*_1\bch_{s^*_1} +  b^*_2\bch_{s^*_2} +  b^*_3\bch_{s^*_3}) - r( b^*_1\bch_{s^*_1} +  b^*_2\bch_{s^*_2}).$
By adding above inequalities, we have $\Delta^s \le r(\bk^1)/3$.

Now, we can use Lemma~\ref{lem: greedy 2} by setting $l=s$ and the fact $\sum_{j=1}^s{b^j} \ge B$ to show $r(\bk^{s+1}) \ge r(\bk^1) + (1-1/e)(r(
\bk^*) - r(\bk^1))$.
Combining $\Delta^s \le r(\bk^1)/3$, we have $r(\bk^{s}) \ge (1-1/3)r(\bk^1) + (1-1/e)(r(
\bk^*) - r(\bk^1)) \ge (1-1/e)r(\bk^*)$, which completes the proof.
\end{proof}

The time complexity of Alg.~\ref{alg: begreedy enum} is $O(B^4m^4\norm{\bc}_{\infty}n_{max}+\norm{\bc}_{\infty}mn_{max}^3)$. 
This is because the number of all partial solutions is $|\cS|=O(B^3m^3)$, and for any partial enumeration $\bk \in \cS$, the time complexity is the same order as the Alg.~\ref{alg: layer traversal greedy}, i.e., $O(B\norm{\bc}_{\infty}mn_{max})$, where $O(n_{max})$ is the time to evaluate $\delta(i,b,\bk)$ as discussed before.

\section{Algorithms and Analysis of Offline Optimization for Non-overlapping \mulane}\label{appendix: offline non over analysis}

\subsection{Starting From the Stationary Distribution}
According to Eq.~(\ref{eq: non-over reward}) and~(\ref{eq: marginal gain}), 
the reward function can be rewritten as 
\begin{equation}
r_{\cG, \bm{\alpha}, \bsig}( \bm{k})=\sum_{i\in[m]}\sum_{b \in [k_i]}g_i{}(b),
\end{equation}
where $g_i(b)$ represents the \textit{layer-level marginal gain} when we allocate one more budget (from $b-1$ to $b$) to layer $i$, which is given by
\begin{equation}\label{eq: layer-wise marginal}
    g_i(b)=\sum_{v \in \cV_i}\sigma_v(P_{i,v}(b)-P_{i,v}(b-1)).
\end{equation}
From the definition of the reward function, we have two observations.
First, budgets allocated to a layer will not affect the reward of other layers because layers are non-overlapping.
Second, the layer-level marginal gain $g_i(b)$ for the $i$-th layer is non-increasing with with respect to budget $b$.
Based on these two observation, we propose the myopic greedy algorithm (Alg.~\ref{alg: non-over greedy}), which is a slight modification of Alg.~\ref{alg: pure greedy}. 
Alg.~\ref{alg: non-over greedy} takes the multi-layered network $\cG=(\cG_1, ..., \cG_m)$, starting distributions $(\bm{\alpha}_1, ..., \bm{\alpha}_m)$, node weights $\bsig=(\sigma_1, ..., \sigma_{|\cV|})$ and total budget $B$ as inputs. 
It consists of $B$ rounds to compute the optimal budget allocation $\bk^*$.
In each round, line~\ref{line: non over MG greedy select} in Alg.~\ref{alg: non-over greedy} selects the layer $i^*$ with the largest layer-level marginal gain and allocate one budget to layer $i^*$ until total $B$ budgets are used up.

\begin{algorithm}[h]
	\caption{Myopic Greedy Algorithm for the Non-overlapping MuLaNE}\label{alg: non-over greedy}
	\begin{algorithmic}[1]
	\INPUT{Network $\cG$, starting distributions $\bm{\alpha}$, budget $B$, constraints $\bc$.}
	\OUTPUT{Budget allocation $\bk$.}
	    \STATE Compute visiting probabilities $(P_{i,u}(b))_{i\in[m], u\in\cV, b \in [c_i]}$ according to Eq.~(\ref{eq: non-over reward}).
	    \STATE Compute layer-level marginal gain $g_{i}(b)_{i \in [m], b \in [c_i]}$ according to Eq.~(\ref{eq: layer-wise marginal}).
	    \STATE $\bk \leftarrow$ MG($(g_i(b))_{i\in[m], b \in [c_i]}$, $B$, $\bc$)
	    \algrule
		\FUNCTION {MG($(g_{i}(b))_{i \in [m], b \in [c_i]}$, $B$, $\bc$)}
		\STATE Let $\bm{k}\defeq\{k_1, ..., k_m\} \leftarrow \bm{0}$, $K \leftarrow B$.
        \WHILE{$K  > 0$} 
            \STATE $i^* \leftarrow \argmax_{i \in [m], k_i + 1 \le c_i} g_i(k_i+1).$  
            \algorithmiccomment{$O(\log m)$ using the priority queue}
            \label{line: non over MG greedy select}
            \STATE $k_{i^*} \leftarrow k_{i^*} + 1$, $K \leftarrow K - 1$.
        \ENDWHILE
        \textbf{return} {$\bk=(k_1, ..., k_m)$.}
        \ENDFUNCTION
\end{algorithmic}
\end{algorithm}

\begin{restatable}{theorem}{thmNonOverGreedySol}
    The Algorithm~\ref{alg: non-over greedy} obtains the optimal budget allocation to the non-overlapping \mulane{} with the stationary starting distributions.
\end{restatable}

\begin{proof}
    	Define a two dimensional array $\bm{M} \in \R^{c_1+c_2+...+c_m}$, where $(i,j)$-th entry is the j-th step layer-level marginal gain for layer $i$, i.e., $\bm{M}{[i,j]} = \sum_{v \in \cV_i}g_{i,v}(j)$, for $i\in[m], j \in [c_i]$. Given the budget allocation $\bk = \{k_1, ..., k_m\}$, the expected reward $r_{\cG, \bm{\alpha}, \bsig}( \bk)$ can be written as the sum of elements in $\bm{M}$, i.e., $r_{\cG, \bm{\alpha}, \bsig}( \bk)=\sum_{i=1}^{m}\sum_{j=1}^{k_i}\bm{M}{[i,j]}$. Because $g_{i,v}(k_i)$ is non-increasing with respect to $k_i$, the element in each row or the layer-level marginal gain is non-increasing. Hence, the greedy method at step $s$ choose the $s$-th largest element in $\bm{M}$. At step $s=B$, the greedy policy selects all $B$ largest elements and the corresponding budget allocation maximizes the expected reward $r_{\cG, \bm{\alpha}, \bsig}( \bk)$. 
\end{proof}

Alg.~\ref{alg: non-over greedy} uses $O(\norm{\bc}_{\infty}mn_{max}^3)$ to compute visiting probabilities $P_{i,u}(k_i)$ and $O(\norm{\bc}_{\infty}mn_{max})$ to compute layer-wise marginal gains $g_i(b)$.
Then $\delta(i,1,\bk)$ can be evaluated in $O(\log m)$ using the priority queue, which is repeated for $B$ iterations. 
Therefore, the time complexity of Alg.~\ref{alg: non-over greedy} is $O(B\log{m}+\norm{\bc}_{\infty}mn_{max}^3)$.

\subsection{Starting From Arbitrary Distributions}
When random explorers start from any arbitrary distribution, Alg.~\ref{alg: non-over greedy} can not obtain the optimal solution.
This is because the layer-level marginal gain $g_i(b)$ is \textit{not} non-increasing with respect to $b$.
So we have to adopt a more general technique, dynamic programming (DP), to solve the budget allocation problem.
The key idea is to keep a DP budget allocation table $A \in \R^{(m+1)\times(B+1)}$, where the $(i,b)$-th entry $A[i,b]$ saves the optimal budget allocation by allocating $b$ budgets to the first $i$ layers. 
Another DP table $V \in R^{(m+1)\times(B+1)}$ saves the value of the optimal reward, where the $(i,b)$-th entry $V[i,b]$ corresponds to the optimal reward when setting $A[i,b]$ as the budget allocation.
In the $i'$-th outer loop and $b'$-th inner loop in Alg~\ref{alg: dynamic programming}, we have already obtained the optimal budget allocation for $A[i,b]$ with $i=0, ..., i'$, $b \in [B]$, which help us to find the optimal amount of budget $j^*$ to $(i'+1)$-th layer (in line~\ref{line: alg2 dp}), and thus we can obtain the $A[i'+1,b']$ and $V[i'+1,b']$ accordingly.
\begin{theorem}
Algorithm~\ref{alg: dynamic programming} obtains the optimal budget allocation to the non-overlapping \mulane{} with an arbitrary starting distribution.
\end{theorem}

Alg.~\ref{alg: dynamic programming} uses $O(\norm{\bc}_{\infty}mn_{max}^3)$ to compute visiting probabilities $P_{i,u}(k_i)$ and $O(\norm{\bc}_{\infty}mn_{max})$ to compute layer-wise marginal gains $g_i(b)$.
In the DP procedure, we can also pre-calculate all rewards $r_{\cG, \bm{\alpha}, \bsig}(j\bm{\chi_{i+1}})$ in $O(\norm{\bc}_{\infty}m)$.
Then DP procedure takes $O(\norm{\bc}_{\infty})$ to find $j^*$, which is repeated for $Bm$ iterations. 
Therefore, the Alg.~\ref{alg: dynamic programming} has the time complexity of $O(B\norm{\bc}_{\infty}m+\norm{\bc}_{\infty}mn_{max}^3)$.

\begin{algorithm}[t]
	\caption{Dynamic Programming (DP) Algorithm for the Non-overlapping MuLaNE}\label{alg: dynamic programming}
	\begin{algorithmic}[1]
	\INPUT{Network $\cG$, starting distributions $\bm{\alpha}$, node weights $\bsig$, budget $B$, constraints $\bc$.}
	\OUTPUT{Budget allocation $\bk$.}
	    \STATE Compute visiting probabilities $(P_{i,u}(b))_{i\in[m], u\in\cV, b \in [c_i]}$ according to Eq.~(\ref{eq: non-over reward}).
	    \STATE Compute layer-level marginal gain $(g_{i}(b))_{i \in [m], b \in [c_i]}$ according to Eq.~(\ref{eq: layer-wise marginal}).
	    \STATE $\bk \leftarrow$ DP($(g_i(b))_{i\in[m], b \in [c_i]}$, $B$, $\bc$)
	    \algrule
		\FUNCTION{DP($(g_{i}(b))_{i\in[m], b \in c_i}$, $B$, $\bc$)}
		\FOR{$i \in 0 \cup [m], b \in 0 \cup [B]$.}
		 \STATE Set $A[i,b] \leftarrow \bm{0}$, $V[i,b] \leftarrow 0$.
		 \ENDFOR
		\FOR {$i \leftarrow 0$ \textbf{to} $m-1$\label{line: alg2 outer}}
		\FOR {$b \in [B]$\label{line: alg2 inner}}
		    \STATE $j^* \leftarrow \argmax_{j \in 0\cup[c_{i+1}]}((V[i, b-j] + r_{\cG, \bm{\alpha}, \bsig}(j\bm{\chi_{i+1}})).$ \label{line: alg2 dp}
		    \STATE $A[i+1, b] \leftarrow  A[i,b-j^*] + j^*\bm{\chi_{i+1}}$.
		    \STATE $V[i+1, b] \leftarrow V[i, b-j^*] + r_{\cG, \bm{\alpha}, \bsig}(j^*\bm{\chi_{i+1}}).$
		    \ENDFOR
		\ENDFOR
		\KwRet{ $A[m,B]$}.
		\ENDFUNCTION
	\end{algorithmic}
\end{algorithm}

\section{Analysis of Online Learning for Overlapping \mulane}\label{appendix: over online analysis}
\subsection{Proof of 1-Norm Bounded Smoothness.}
\condOverMono*
\begin{proof}
    According to Eq.~(\ref{eq: over reward}), and since $\mu_{i,u,b} \le \mu'_{i,u,b}$, for $(i,u,b)\in\cA$, we have $(1-\Pi_{i\in[m]}(1-\mu_{i,u,k_i})) \le (1-\Pi_{i\in[m]}(1-\mu'_{i,u,k_i}))$.
    Therefore for any $\sigma_u \le \sigma'_u$, we have $\sigma_u(1-\Pi_{i\in[m]}(1-\mu_{i,u,k_i})) \le \sigma'_u(1-\Pi_{i\in[m]}(1-\mu'_{i,u,k_i}))$
    
    By summing up both sides over $u\in \cV$, we have $r_{\bk}(\bm{\mu}, \bsig)  \le r_{\bk}(\bm{\mu'}, \bsig')$.
\end{proof}

\condOverOneNorm*
\begin{proof}
	The left-hand side:
	\begin{align*}
	&\abs{r_{\bm{\mu}, \bsig}(\bm{k}) - r_{\bm{\mu'}, \bsig'}(\bm{k})}\\ 
	&= \Big|\sum_{u\in\cV}\big[\sigma_u\left(1-\Pi_{i\in[m]}\left(1-\mu_{i,u,k_i}\right)\right) \\
	&- \sigma_u'\left(1-\Pi_{i\in[m]}\left(1-\mu'_{i,u,k_i}\right)\right)\big]\Big|\\
	&\le \sum_{u\in\cV} \abs{\sigma_u(1-\Pi_{i \in [m]} q_{i,u}) - \sigma_u'(1-\Pi_{i \in [m]}q'_{i,u})}\\
	&= \sum_{u\in\cV} |\sigma_u(1-\Pi_{i \in [m]} q_{i,u}) - \sigma_u(1-\Pi_{i \in [m]} q'_{i,u}) \\
	&+ \sigma_u(1-\Pi_{i \in [m]} q'_{i,u}) - \sigma_u'(1-\Pi_{i \in [m]}q'_{i,u})|\\
	&\le \sum_{u\in\cV} \sigma_u|\Pi_{i \in [m]} q_{i,u} -\Pi_{i \in [m]} q'_{i,u}| \\ 
	&+ \sum_{u\in\cV}|\sigma_u - \sigma_u'|(1-\Pi_{i \in [m]}q'_{i,u}),
	\end{align*} 
	where $q_{i,u} = (1-\mu_{i,u,k_i}), q'_{i,u} = (1-\mu'_{i,u,k_i})$.
	
	For the first term, it can be derived that,
	\begin{align*}
	&\sigma_u\abs{\Pi_{i \in [m]} q_{i,u} -\Pi_{i \in [m]} q'_{i,u}} \\
	&= \sigma_u\abs{ \sum_{i=1}^m \prts{\prts{\Pi_{j=1}^{i-1}{q'_{j,u}}} (q_{i,u} - q'_{i,u})\prts{\Pi_{j=i+1}^m q_{j,u}}} }\\
	& \le  \sigma_u\sum_{i=1}^m \abs{ \prts{\prts{\Pi_{j=1}^{i-1}{q'_{j,u}}} (q_{i,u} - q'_{i,u})\prts{\Pi_{j=i+1}^m q_{j,u}}} }\\
	&\le  \sigma_u\sum_{i=1}^m \abs{q_{i,u} - q'_{i,u}} \\
	&= \sigma_u\sum_{i \in [m]} \abs{\mu_{i,u,k_i} - \mu'_{i,u,k_i}},
	\end{align*}
	where last inequality holds because $q_{j,u}, q'_{j,u} \le 1$ and $\sigma_u \in [0,1]$, for any $j\in[m], u \in \cV$.
	
	For the second term, we have 
	\begin{align*}
	    &|\sigma_u - \sigma_u'|(1-\Pi_{i \in [m]}q'_{i,u})\\
	    &=|\sigma_u - \sigma_u'|(1-\Pi_{i \in [m]}(1-\mu'_{i,u,k_i}))\\
	    &\le |\sigma_u - \sigma_u'| \sum_{i\in[m]}{\mu'_{i,u,k_i}},
	\end{align*}
	where the inequality holds because of the Weierstrass product inequality.
	Plugging the above two terms into the previous inequality, 1-Norm Bounded Smoothness is satisfied.
\end{proof}

\subsection{Regret Analysis for Overlapping \mulane}\label{appendix: over regret analysis}
In this section, we give the regret analysis for overlapping \mulane.
\subsubsection{Facts} \label{app: fact1}

We utilize the following tail bound in our analysis.
\begin{fact}[Hoeffding's Inequality~\cite{dubhashi2009concentration}.]\label{fact: hoef}
	Let $Y_1, ..., Y_n$ be independent and identically distributed (i.e., i.i.d) random variables with common suppport $[0,1]$ and mean $\mu$. Let $Z = Y_1 + ... + Y_n.$ Then for all $\delta$,
	\begin{equation*}
	Pr\{|Z-n\mu|\ge \delta\} \le 2 e^{-2\delta^2/n}.
	\end{equation*}
\end{fact}

\subsubsection{Proof Details}\label{appendix: over proof details}

Let $\cA$ denote the set containing all base arms, i.e.,  $\cA = \{(i, u, b) | i \in [m], u \in \cV, b \in [c_i]\}$.
We add subscript $t$ to denote the value of a variable at the end of round $t \in [T]$, e.g. $\bar{\sigma}_t, \hat{\mu}_{i,u,b,t}$, where $T$ is the total number of rounds.
For example, $T_{i,u,b,t}$ denotes the total number of times that arm $(i,u,b)$ is played at the end of round $t$.
Let us first introduce a definition of an unlikely event, where $\hat{\mu}_{i,u,b,t-1}$ is not accurate as expected.

\begin{definition}
We say that the sampling is nice at the beginning of round t, if for every arm $(i,u,b) \in \cA$, $|\hat{\mu}_{i,u,b,t-1} - \mu_{i,u,b}| < \rho_{i,u,b,t}$, where $\rho_{i,u,b,t} = \sqrt{\frac{3\ln t}{2T_{i,u,b, t-1}}}$. Let $\mathcal{N}_t$ be such event.
\end{definition}

\begin{lemma} \label{lem: over nice event}
	For each round $t \ge 1$, $Pr\{\neg N_t\} \le 2|\cA|t^{-2}$
\end{lemma}

\begin{proof}
	For each round $t \ge 1$, we have 
	\begin{align*}
	&Pr\{\neg \mathcal{N}_t\} \\
	&= Pr\{\exists (i,u,b)\in  \cA, |\hat{\mu}_{i,u,b,t-1} - \mu_{i,u,b}| \ge \sqrt{\frac{3\ln t}{2T_{i,u,b, t-1}}}\} \\
	&\le \sum_{(i,u,b)\in  \cA}Pr\{ |\hat{\mu}_{i,u,b,t-1} - \mu_{i,u,b}| \ge \sqrt{\frac{3\ln t}{2T_{i,u,b, t-1}}}\} \\
	&=\sum_{(i,u,b)\in  \cA}\sum_{k=1}^{t-1}Pr\{ T_{i,u,b,t-1} = k \\
	&\wedge |\hat{\mu}_{i,u,b,t-1} -\mu_{i,u,b}| \ge \sqrt{\frac{3\ln t}{2T_{i,u,b, t-1}}}\} \\
	&\le \sum_{(i,u,b)\in  \cA}\sum_{k=1}^{(t-1)} \frac{2}{t^{3}} \le 2|\cA|t^{-2}. 
	\end{align*}
	
Given $T_{i,u,b,t-1} = k$, $\hat{\mu}_{i,u,b,t-1}$ is the average of $k$ i.i.d. random variables $Y_{i,u,b}^{[1]}, ..., Y_{i,u,b}^{[k]}$, where $Y_{i,u,b}^{[j]}$ is the Bernoulli random variable ( computed in Alg.~\ref{alg: over online} line~\ref{line: over bernoulli rv}) when the arm $(i,u,b)$ is played for the $j$-th time during the execution.  
That is, $\hat{\mu}_{i,u,b,t-1}=\sum_{j=1}^k Y_{i,u,b}^{[j]}/k$. Note that the independence of $Y_{i,u,b}^{[1]}, ..., Y_{i,u,b}^{[k]}$ comes from the fact we select a new starting position following the starting distribution $\bm{\alpha}_i$ at the beginning of each round $t$. 
And the last inequality uses the Hoeffding Inequaility (Fact~\ref{fact: hoef}).
\end{proof}

Then we can use monotonicity and 1-norm bounded smoothness properties (Property~\ref{cond: over mono} and Property~\ref{cond: over 1-norm}) to bound the reward gap $\bm{\Delta_{k_t}}=\xi     r_{\bm{\mu}, \bsig}(\bm{k^*}) - r_{\bm{\mu}, \bsig}(\bm{k_t})$ between the optimal action $\bm{k^*}$ and the action $\bm{k_t}=(k_{t,1}, ..., k_{t,m})$ selected by our algorithm $A$. 
To achieve this, we introduce a positive real number $M_{i,u,b}$ for each arm $(i,u,b)$ and define $M_{\bm{k_t}}=\max_{(i,u,b) \in S_t} M_{i,u,b}$, where $S_t=\{(i,u,b) \in \cA | b = k_{t,i}\}$.
Define $|\cE'|= |S_t| =\sum_{i \in [m]}{|\cV|}=m|\cV|$, and let

\begin{equation*}
\kappa_T(M,s) = \begin{cases}
2, & \text{if $s=0$}, \\
3\sqrt{\frac{6\ln T}{s}}, & \text{if $1 \le s \le \ell_T{(M)}$}, \\
0, &\text{if $s \ge \ell_T(M) + 1$},
\end{cases}
\end{equation*}

where
\begin{equation*}
\ell_T(M) = \lfloor \frac{54|\cE'|^2\ln T}{M^2} \rfloor.
\end{equation*}

\begin{lemma}
	If $\{\Delta_{\bm{k}_t} \ge M_{\bm{k_t}} \}$ and $\mathcal{N}_t$ holds, we have
	\begin{equation} \label{lem: over bound_delta}
	\Delta_{\bm{k}_t} \le \sum_{(i,u,b) \in S_t}\kappa_T(M_{i,u,b},T_{i,u,b,t-1}).
	\end{equation}	
\end{lemma}

\begin{proof}

    First, we can observe, for $(i,u,b) \in \cA$, $\mu_{i,u,b}$ is upper bounded by $\bar{\mu}_{i,u,b,t}$, i.e., $\bar{\mu}_{i,u,b,t} \ge \mu_{i,u,b}$, when $\cN_t$ holds.
    This is due to 
    \begin{align*}
        \bar{\mu}_{i,u,b,t} &= \max_{j \in [b]}\tilde{\mu}_{i,u,j,t} \\
        &\ge \tilde{\mu}_{i,u,b,t} \\
        &=\min\{\hat{\mu}_{i,u,b,t-1}+\rho_{i,u,b,t}, 1\}\\
        &\ge \mu_{i,u,b} ,
    \end{align*}
    where the last inequality comes from the fact that $|\hat{\mu}_{i,u,b,t-1} - \mu_{i,u,b}| < \rho_{i,u,b,t}$ when $\cN_t$ holds.

	Then, notice that the right hand of Inequality~(\ref{lem: over bound_delta}) is non-negative, it is trivially satisfied when $\Delta_{\bm{k}_t} = 0$. We only need to consider $\Delta_{\bm{k}_t} > 0$. By $\mathcal{N}_t$ and Property~\ref{cond: over mono},
	We have the following inequalities,
	\begin{equation*}
	r_{\bar{\bm{\mu}}_t, \bar{\bsig}}(\bm{k_t}) \ge \xi r_{\bar{\bm{\mu}}_t, \bar{\bsig}}(\bm{k^*}) \ge \xi r_{{\bm{\mu}, \bsig}}(\bm{k^*}) = \Delta_{\bm{k_t}} + r_{\bm{\mu}, \bsig}(\bm{k_t}).
	\end{equation*}
	The first inequality is due to the \textit{oracle} outputs the $(\xi, \beta)$-approximate solution given parameters $\bar{\bm{\mu}}_t$, where $\xi=1-e^{-\eta}, \beta=1$ by definition.
	The second inequality is due to Condition~\ref{cond: over mono} (monotonicity). 
	
	Then, by Condition~\ref{cond: over 1-norm}, we have 
	\begin{align*}
	\Delta_{\bm{k_t}} &\le r_{\bar{\bm{\mu}}_t, \bar{\bsig}}(\bm{k_t}) - r_{{\bm{\mu}, \bsig}}(\bm{k_t}) \\
	&\le \sum_{(i,u,b) \in S_t}(\sigma_u|\bar{{\mu}}_{i,u,b, t} - \mu_{i,u,b}|+|\bar{\sigma}_u-\sigma_u|\bar{{\mu}}_{i,u,b, t})\\
	&\le \sum_{(i,u,b) \in S_t}(|\bar{{\mu}}_{i,u,b, t} - \mu_{i,u,b}|+\rho_{i,u,b,t}),
	\end{align*}
	where the last inequality is due to the following observation:
	$\bar{\sigma}_u\neq \sigma_u$ if and only if $u$ has not been visited so far, so $\hat{{\mu}}_{i,u,b, t-1}=0$ for any $i\in[m], b\in[B]$ and $\bar{{\mu}}_{i,u,b, t}=\max_{j\in[b]}\{\hat{{\mu}}_{i,u,j, t-1}+\rho_{i,u,j,t}\}=\rho_{i,u,b,t}$.
	
	Next, we can bound $\Delta_{\bm{k_t}}$ by bounding $\bar{{\mu}}_{i,u,b, t} - \mu_{i,u,b} + \rho_{i,u,b,t}$ by applying a transformation. As $\{\Delta_{\bm{k}_t} \ge M_{\bm{k_t}} \}$ holds by assumption, so $\sum_{(i,u,b) \in S_t}|\bar{{\mu}}_{i,u,b, t} - \mu_{i,u,b}|+\rho_{i,u,b,t} \ge \Delta_{\bm{k_t}} \ge M_{\bm{k_t}}$. We have 
	\begin{align}\label{ieq: online over bound k}
	\Delta_{\bm{k_t}} &\le \sum_{(i,u,b) \in S_t}(|\bar{{\mu}}_{i,u,b, t} - \mu_{i,u,b}|+\rho_{i,u,b,t}) \nonumber\\
	& \le -M_{\bm{k_t}} + 2\sum_{(i,u,b) \in S_t}|\bar{{\mu}}_{i,u,b, t} - \mu_{i,u,b}|+\rho_{i,u,b,t} \nonumber\\
	& = 2\sum_{(i,u,b) \in S_t}(|\bar{{\mu}}_{i,u,b, t} - \mu_{i,u,b}|+\rho_{i,u,b,t} - \frac{M_{\bm{k_t}}}{2|\cE'|}) \nonumber\\
	& \le 2\sum_{(i,u,b) \in S_t}(|\bar{{\mu}}_{i,u,b, t} - \mu_{i,u,b}|+\rho_{i,u,b,t} - \frac{M_{i,u,b}}{2|\cE'|}) .
	\end{align}
	
	By $\mathcal{N}_t$, for all $(i,u,b) \in \cA$, we have:
	\begin{align*}
	    &|\bar{\mu}_{i,u,b,t} - \mu_{i,u,b}| + \rho_{i,u,b,t} - \frac{M_{i,u,b}}{2|\cE'|} \\
	    &=\tilde{\mu}_{i,u,l,t} - \mu_{i,u,b}+ \rho_{i,u,b,t} - \frac{M_{i,u,b}}{2|\cE'|} \tag{$l = \arg\max_{j \in [b]}\tilde{\mu}_{i,u,j,t}$}\\
	    &\le \tilde{\mu}_{i,u,l,t} - \mu_{i,u,l}+ \rho_{i,u,b,t} - \frac{M_{i,u,b}}{2|\cE'|}\\
	     &\le  \min\{2\rho_{i,u,l,t}+ \rho_{i,u,b,t}, 1\} - \frac{M_{i,u,b}}{2|\cE'|} \\
	    &\le  \min\{3\rho_{i,u,b,t}, 1\} - \frac{M_{i,u,b}}{2|\cE'|} \\
	    &\le  \min \left\{3\sqrt{\frac{3\ln T}{2T_{i,u,b, t-1}}}, 1\right\} - \frac{M_{i,u,b}}{2|\cE'|}.
	\end{align*}
	where the first inequality is due to $\mu_{i,u,j} \le \mu_{i,u,b}$ for $j \le b$,
	the second inequality is due to the event $\cN_t$ and the third inequality holds because $T_{i,u,j, t-1} \ge T_{i,u,b, t-1}$ for $j \le b$.
	
	Now consider the following cases:
	
	\textbf{Case 1.} If $T_{i,u,b, t-1} \le \ell_T(M_{i,u,b})$, we have $\bar{\mu}_{i,u,b,t} - \mu_{i,u,b}+ \rho_{i,u,b,t} - \frac{M_{i,u,b}}{2|\cE'|} \le \min\left\{3       \sqrt{\frac{3\ln T}{2T_{i,u,b, t-1}}}, 1\right\} \le \frac{1}{2} \kappa_T(M_{i,u,b}, T_{i,u,b,t-1})$. 
	
	\textbf{Case 2.} If  $T_{i, b, t-1} \ge \ell_T(M_{i,u,b,t})+1$, then $3\sqrt{\frac{3\ln T}{2T_{i,u,b, t-1}}} \le \frac{M_{i,u,b}} {2|\cE'|}$, so $\bar{\mu}_{i,u,b,t} - \mu_{i,u,b}+ \rho_{i,u,b,t} - \frac{M_{i,u,b}}{2|\cE'|} \le 0 = \frac{1}{2} \kappa_T(M_{i,u,b}, T_{i,u,b,t-1})$. 
	
	Combine these two cases and inequality~(\ref{ieq: online over bound k}), we have the following inequality, which concludes the lemma.
	\begin{equation*}
	\Delta_{\bm{k}_t} \le \sum_{(i,u,b) \in S_t}\kappa_T(M_{i,u,b},T_{i,u,b,t-1}).
	\end{equation*}
\end{proof}

\thmOver*
\begin{proof}[Proof of Theorem~\ref{thm: over regret}]
	By above lemmas, we have
	\begin{align*}
	&\sum_{t=1}^T \mathbb{I}\left({\{\Delta_{\bm{k_t}} \ge M_{\bm{k_t}}\} \wedge \mathcal{N}_t}\right) \Delta_{\bm{k_t}} \\
	&\le \sum_{t=1}^T \sum_{(i,u,b)\in S_t}\kappa_T(M_{i,u,b},T_{i,u,b,t-1}) \\
	&=\sum_{(i,u,b) \in \cA}\sum_{t'\in \{ t| t\in [T], \bm{k}_{t,i} = b \}} \kappa_T(M_{i,u,b},T_{i,u,b,t'-1}) \\
	&\le \sum_{(i,u,b) \in \cA}\sum_{s=0}^{T_{i,u,b,T}} \kappa_T(M_{i,u,b},s)\\
	&\le 2|\cA| + \sum_{(i,u,b) \in \cA}\sum_{s=1}^{\ell_T(M_{i,u,b})} 3\sqrt{6\frac{\ln T}{s}}\\
	&\le 2|\cA| + \sum_{(i,u,b) \in \cA} \int_{s=0}^{\ell_T(M_{i,u,b})}  3\sqrt{6\frac{\ln T}{s}} ds \\
	&\le 2|\cA| + \sum_{(i,u,b) \in \cA}   6\sqrt{6 \ln T \ell_T(M_{i,u,b})} \\
	&\le 2|\cA| + \sum_{(i,u,b) \in \cA} \frac{108|\cE'| \ln T}{M_{i,u,b}}.
	\end{align*}
	For distribution dependent bound, let us set $M_{i,u,b} = \Delta^{i,u,b}_{\min}$ and thus the event \{$\Delta_{\bm{k_t}} \ge M_{\bm{k_t}}$\} always holds, i.e., $\mathbb{I}\{\Delta_{\bm{k_t}} < M_{\bm{k_t}}\}=0$ . We have 
	\begin{equation*}
	\begin{split}
	&\sum_{t=1}^T \mathbb{I}\left(\mathcal{N}_t\ \wedge \{\Delta_{\bm{k_t}} \ge M_{\bm{k_t}}\}\right) \Delta_{\bm{k_t}}\\
	&\le 2|\cA| + \sum_{(i,u,b) \in \cA} \frac{108|\cE'| \ln T}{\Delta_{\min}^{i,u,b}}.	    
	\end{split}
	\end{equation*}
	
	By Lem.~\ref{lem: over nice event}, $Pr\{\neg N_t\} \le 2|\cA|t^{-2}$, then 
	\begin{equation*}
	\E\left[\sum_{t=1}^T \mathbb{I}\left( \neg \mathcal{N}_t\right) \Delta_{\bm{k_t}}\right] \le \sum_{t=1}^T 2|\cA|t^{-2} \Delta_{\max} \le \frac{\pi^2}{3}|\cA| \Delta_{\max}.
	\end{equation*}
	
	By our choice of $M_{i,u,b}$, 
	
	\begin{equation*}
	\sum_{t=1}^T \mathbb{I}\left(\Delta_{\bm{k_t}} < M_{\bm{k_t}}\right) \Delta_{\bm{k_t}} = 0 .
	\end{equation*}
	
	By the definition of the regret,
	
	\begin{align*}
	Reg_{\bm{\mu}}(T) &= \E\left[\sum_{t=1}^T \Delta_{\bm{k_t}}\right] \\
	&\le \E\bigg[\sum_{t=1}^T \bigg(\mathbb{I}\big( \neg \mathcal{N}_t\big) + \mathbb{I}\big(\Delta_{\bm{k_t}} < M_{\bm{k_t}}\big) \\
	&+ \mathbb{I}\big(\mathcal{N}_t\ \wedge \{\Delta_{\bm{k_t}} \ge M_{\bm{k_t}}\}\big)\bigg) \Delta_{\bm{k_t}}\bigg] \\
	&\le \sum_{(i,u,b) \in \cA}\frac{108|\cE'| \ln T}{\Delta_{min}^{i,u,b}} + 2|\cA| + \frac{\pi^2}{3}|\cA|\Delta_{max}\\
	&= \sum_{(i,u,b) \in \cA}\frac{108m|\cV| \ln T}{\Delta_{min}^{i,u,b}} + 2|\cA| + \frac{\pi^2}{3}|\cA|\Delta_{max}.
	\end{align*}
	
	For distribution independent bound, we take $M_{i,u,b} = M =\sqrt{108|\cE'||\cA| \ln T /T}$, then we have
	$\sum_{t=1}^T \mathbb{I}\{\Delta_{\bm{k_t}} \le M_{\bm{k_t}}\} \le TM$.
	
	Then 
	\begin{align*}
	&Reg_{\mu}(T) \\
	&\le \sum_{(i,u,b) \in \cA}\frac{108|\cE'| \ln T}{M_{i,u,b}} + 2|\cA| + \frac{\pi^2}{3}|\cA|\Delta_{max} + TM \\
	&\le \frac{108|\cE'||\cA| \ln T}{M} + 2|\cA| + \frac{\pi^2}{3}|\cA|\Delta_{max} + TM \\
	&= 2\sqrt{108|\cE'| |\cA|T\ln T} + \frac{\pi^2}{3}|\cA|\Delta_{max} + 2|\cA| \\
	&\le 21\sqrt{m|\cV||\cA|T\ln T} + \frac{\pi^2}{3}|\cA|\Delta_{max} + 2|\cA|.
	\end{align*}
\end{proof}

\subsection{Extension for Random Node Weights}\label{appendix: over extend}
\begin{algorithm}[t]
	\caption{CUCB-MAX-R: Combinatorial Upper Confidence Bound algorithm for the overlapping \mulane, with Random node weights}\label{alg: extend over online}
	\resizebox{1.0\columnwidth}{!}{
\begin{minipage}{\columnwidth}
	\begin{algorithmic}[1]
	    \INPUT{Budget $B$, number of layers $m$, number of nodes $|\cV|$, constraints $\bc$, offline oracle BEG.}
		\STATE For each arm $(i,u,b) \in \cA$, $T_{i,u,b} \leftarrow 0$, $\hat{\mu}_{i,u,b} \leftarrow 0$. 
		\STATE {For each node $v \in \cV$, $T'_v \leftarrow 0$, $\hat{\sigma}_v \leftarrow 0$.}\label{line: extend over initial weight}
		\FOR {$t=1,2,3,...,T $}
		    \FOR {$(i,u,b) \in \cA$}
		    \STATE $\rho_{i,u,b} \leftarrow \sqrt{3\ln t/(2T_{i,u,b})}$.  \label{line: extend over online radius}
		    \STATE $\tilde{\mu}_{i,u,b} \leftarrow \min \{\hat{\mu}_{i,u,b} + \rho_{i,u,b}, 1\}$. 
		    \label{line: extend over online ucb}
            \ENDFOR
            \FOR {$v \in \cV$}
		    \STATE $\rho'_{v} \leftarrow \sqrt{3\ln t/(2T'_{v})}$.  \label{line: extend over online radius for nodes}
		    \STATE $\hat{\sigma}_{v} \leftarrow \min \{\hat{\sigma}_{v} + \rho'_{v}, 1\}$. 
		    \label{line: extend over online ucb for nodes}
            \ENDFOR
        \STATE For $(i,u,b)\in \cA$, $\bar{\mu}_{i,u,b} \leftarrow \max_{j \in [b]} {\tilde{\mu}_{i,u,j}}$.		
		\label{line: extend over online validate}\label{alg: extend valid}
		\STATE $\bm{k} \leftarrow$ BEG($(\bar{\mu}_{i,u,b})_{(i,u,b)\in\cA}$, $(\bar{\sigma}_v)_{v \in \cV}$, $B$, $\bc$).
		\label{line: extend over online oracle}
		\STATE Apply budget allocation $\bk$, which gives trajectories $\bm{X}\defeq(X_{i,1}, ..., X_{i,k_i})_{i\in [m]}$ as feedbacks. \label{line: extend over online feedback}
		\STATE { For any visited node $v \in \bigcup_{i\in[m]}\{X_{i,1}, ..., X_{i,k_{i}}\}$, receive its random node weight $\tilde{\sigma}_v$, update $T'_{v} \leftarrow T'_{v} + 1$,  $\hat{\sigma}_v \leftarrow \hat{\sigma}_v  + (\tilde{\sigma}_v - \hat{\sigma}_v )/T'_{v}$}\label{line: extend update weight}
		\STATE For any $(i,u,b) \in \tau \defeq \{(i,u,b) \in \cA \vert \, b \le k_i\}$, \\$Y_{i,u,b} \leftarrow 1$ if $u \in \{X_{i,1}, ..., X_{i, b}\}$ and $0$ otherwise.\label{line: extend over bernoulli rv}
		\STATE For $(i,u,b) \in \tau$, update $T_{i,u,b}$ and $\hat{\mu}_{i,u,b}$:\\ 
		$T_{i,u,b} \leftarrow T_{i,u,b} + 1$,  $\hat{\mu}_{i,u,b} \leftarrow \hat{\mu}_{i,u,b} + (Y_{i,u,b} - \hat{\mu}_{i,u,b})/T_{i,u,b}.$\label{line: extend over online update}
		\ENDFOR
		\end{algorithmic}
		\end{minipage}}
\end{algorithm}
In this section, we extend the deterministic node weight $\sigma_u$ for node $u \in \cV$ to a random weight. 
Specifically, we denote $\tilde{\sigma}_u\in[0,1]$ as the random weight when we first visit $u$ (after the first visit we will not get any reward so on so forth), and denote $\sigma_u\defeq \E[\tilde{\sigma}_u]$.

For the offline setting, the reward function in~Eq.(\ref{eq: over reward}) remains the same since $\tilde{\sigma}_u$ is independent of all other random variables.
For the online setting, since the reward function remains unchanged, property ~\ref{cond: over mono} and property~\ref{cond: over 1-norm} still hold.
However, we need to change the way we maintain the optimistic node weight $\bar{\sigma}_u$.

We present our CUCB-MAX-R algorithm in Alg.~\ref{alg: extend over online}.
Compared with Alg.~\ref{alg: over online}, the major difference is that $\bar{\sigma}_u$ now represent the upper confidence bound (UCB) value of $\sigma_u$.
In line~\ref{line: extend over initial weight}, we denote $T'_v$ as the number of times node $v$ is played so far and $\hat{\sigma}_v$ as the empirical mean of $v$'s weight.
In line~\ref{line: extend over online radius for nodes}, $\rho'_v$ is the confidence radius and $\bar{\sigma}_u$ is the UCB value of node weight $\sigma_u$.
In line~\ref{line: extend update weight}, we update the empirical mean $\hat{\sigma}_v$ for any visited $v$.

\noindent{\textbf{Regret analysis}}.
Let $\cA$ denote the set containing all base arms for visiting probabilities, i.e.,  $\cA = \{(i, u, b) | i \in [m], u \in \cV, b \in [c_i]\}$.
With a bit of abuse of notation, we also use $\cV$ to denote the set of base arms corresponding to the node weights.
Therefore, $\cA \cup \cV$ is the set of all base arms, which is different from the Section~\ref{appendix: over regret analysis} that only has $\cA$ as base arms and does not have base arms for random weights.
Following the notations of ~\cite{wang2017improving}, we define $X_a \in [0,1]$ be the random outcome of base arm $a \in \cA\cup\cV$ and define $D$ as the unknown joint distribution of random outcomes $\bm{X}=(X_a)_{a\in \cA \cup \cV}$ with unknown mean $\mu_a=\E_{\bm{X}\sim D}[X_a]$.
Let $\bk$ be any feasible budget allocation, we denote $p_a^{D, \bk}$ the probability that action $\bk$ triggers arm $a$ (i.e. outcome $X_a$ is observed) when the unknown environment distribution is $D$.
Specifically, $p_a^{D, \bk}=1$ for $a \in \{(i, u, b) | i \in [m], u \in \cV, b = k_i\}$, $p_a^{D,\bk}=1-\prod_{i\in [m]}(1-\mu_{i,a,k_i})$ for $a \in \cV$ and $p_i^{D,\bk}=0$ for the rest of base arms.
Therefore, we can rewrite the inequality in property~\ref{cond: over 1-norm} as 
\begin{align}
    &|r_{\bm{\mu}, \bsig}(\bm{k}) - r_{\bm{\mu'}, \bsig'}(\bm{k})| \notag\\
    &\le \sum_{i \in [m], u \in \cV, b = k_i}(\sigma_u'|\mu_{i,u,b} - \mu'_{i,u, b}|)\notag\\
    &+\sum_{u\in \cV}\abs{\sigma_u-\sigma_u'}\left(1-\prod_{i\in[m]}(1-\mu_{i,u,b})\right)\notag\\
    &\le\sum_{a\in\cA}p_a^{D,\bk}|\mu_a-\mu'_a|,
\end{align}\label{ieq: extend new property 2}
where we abuse the notation to denote $\mu_a$ as $\sigma_a$ and $\mu'_a$ as $\sigma'_a$ if $a \in \cA_2$ in the last inequality.
We can observe that the above rewritten property~\ref{cond: over 1-norm} can be viewed as a special case of the 1-norm TPM bounded smoothness condition in ~\cite{wang2017improving}.

Now we can apply the similar proof in Section B.3 in ~\cite{wang2017improving}, together with how we deal with the $|\bar{\mu}_{i,u,b,t} - \mu_{i,u,b}|$ in Section~\ref{appendix: over proof details} to get a distribution-dependent regret bound.

Let $K=m|\cV|+|\cV|$ be the maximum number of triggered arms, $|\cA \cup \cV|=Bm|\cV|+|\cV|$ be the number of base arms and $\lceil x \rceil_0=\max\{\lceil x \rceil, 0\}$.
For any feasible budget allocation $\bk$ and any base arm $a \in \cA \cup \cV$, the gap $\Delta_{\bk}=\max\{0, \xi r_{\bm{\mu}, \bsig}(\bk^*)-r_{\bm{\mu},\bsig}(\bk)\}$ and we define $\Delta_{\min}^a=\inf_{\bk:p_{a}^{D,\bk}>0, \Delta_{\bk}>0}\Delta_{\bk}$, $\Delta_{\max}^a=\sup_{\bk:p_{a}^{D,\bk}>0, \Delta_{\bk}>0}\Delta_{\bk}$.
As convention, if there is no action $\bk$ such that $p_a^{D,\bk}>0$ and $\Delta_{\bk}>0$, we define $\Delta_{\min}^a=+\infty, \Delta_{\max}^a=0.$
We define $\min_{a\in\cA}\Delta_{\min}^a$ and $\max_{a\in\cA}\Delta_{\max}^a$.
\begin{theorem}
Algorithm~\ref{alg: extend over online} has the following distribution-dependent $(1-e^{-\eta}, 1)$ approximation regret,
\begin{align*}
   &Reg_{\bm{\mu}, \bsig}(T) \le \sum_{a \in \cA \cup \cV}\frac{576K \ln T}{\Delta_{\min}^{a}}\\ 
   &+ 4|\cA \cup \cV| + \sum_{a\in\cA \cup \cV}\left(\lceil\log_2\frac{2K}{\Delta_{\min}^a}\rceil_0+2\right)\frac{\pi^2}{6}\Delta_{\max},
\end{align*}
\end{theorem}
We can also have a distribution-independent bound,
\begin{align*}
   &Reg_{\bm{\mu}, \bsig}(T) \le 12\sqrt{|\cA \cup \cV|KT \ln T}\\ 
   &+ 2|\cA \cup \cV| + \left(\lceil\log_2\frac{T}{18\ln T}\rceil_0+2\right)\frac{\pi^2}{6}\Delta_{\max},
\end{align*}
\noindent\textbf{Remark.}
Compared with the regret bound in Thm.~\ref{thm: over regret} in the main text, we can see the base arms now become $\cA \bigcup \cV$ instead of $\cA$, which is worse than the regret bound of the fixed unknown weight setting.
This is a trade-off since we need more estimation to handle the uncertain random node weight, which incurs additional regrets.
It also shows our non-trivial analysis of the fixed unknown weights carefully merge the error term into base arms $\cA$ and incurs only constant-factor of extra regrets. 

\section{Online Learning for Non-overlapping \mulane}\label{appendix: non over online}
\subsection{Online Learning Algorithm for Non-overlapping \mulane}
For non-overlapping \mulane, we estimate layer-level marginal gains as our parameters, rather than probabilities $P_{i,u}(b)$ and node weights $\bsig$.
By doing so, we can largely simply our analysis since we no longer need to handle the unknown weights and the extra constraint of $P_{i,u}(b)$ as in the overlapping setting.
Concretely, we maintain a set of base arms $\cA = \{(i,b) | i \in [m], b \in [c_i]\}$, and let $|\cA|=\sum_{i \in [m]}c_i$ be the total number of these arms.
For each base arm $(i,b) \in \cA$., denote $\mu_{i,b}$ as the truth value of each base arm, i.e., $\mu_{i,b}=\sum_{u \in \cV}\sigma_u(P_{i,u}(b)-P_{i,u}(b-1))$.
Based on the definition of base arms, we can see that we do not need to explicitly estimate the node weights $\bsig$ or probabilities $P_{i,u}(b)$, so we will use $r{\bm{\mu}}({\bm{k}})$ to denote the reward function $r_{\cG, \bm{\alpha}, \bsig}(\bk)$.

We present our algorithm in Alg.~\ref{alg: non-over online}, which has three differences compared with the Alg.~\ref{alg: over online} for overlapping \mulane.
First, the expected reward of action $\bm{k}$ is $r_{\bm{\mu}}(\bm{k})=\sum_{i=1}^{m}\sum_{b=1}^{k_i}\mu_{i,b}$, which is a linear reward function and satisfies two new Monotonicity and 1-Norm Bounded Smoothness conditions (see Condition~\ref{cond: mono} and~\ref{cond: 1-norm}).
Second, we use the Dynamic Programming method (Alg.~\ref{alg: dynamic programming}) as our Oracle, which does not require the parameters $\mu_{i,b}$ to be increasing.
Moreover, the offline oracle always outputs \textit{optimal} solutions.
Third, we use a new random variable $Y_{i,b}$, which indicates the gain of weights which depends on whether a new node is visited by random walker from the $i$-th layer exactly at the $b$-th step, i.e., $Y_{i,b} = \sum_{u\in \bigcup_{j=1}^b \{X_{i,j}\}} \sigma_u - \sum_{u \in \bigcup_{j=1}^{b-1} \{X_{i,j}\}}\sigma_u$ for base arms $(i,b)$ such that $b \le k_i$. 

\begin{condition}[Monotonicity] \label{cond: mono}
The reward $r_{\bm{\mu}}(\bm{k})$ is monotonically increasing, i.e., for any budget allocation $\bm{k}$ and any two vectors $\bm{\mu} = (\mu_{i,b})_{(i,b) \in \cA}$, $\bm{\mu'}=(\mu'_{i,b})_{(i,b) \in \cA}$, we have $r_{\bm{\mu}}(\bm{k}) \le r_{\bm{\mu'}}(\bm{k})$, if $\mu_{i,b} \le \mu'_{i,b}$, for $(i,b) \in \cA$.
\end{condition}

\begin{condition}[1-Norm Bounded Smoothness]\label{cond: 1-norm}
The reward function $r_{\bm{\mu}}(\bm{k})$ satisfies the 1-norm bounded smoothness, i.e., for any budget allocation $\bm{k}$ and any two vectors $\bm{\mu} = (\mu_{i,b})_{(i,b)\in \cA}$, $\bm{\mu'}=(\mu'_{i,b})_{(i,b)\in \cA}$, we have $|r_{\bm{\mu}}(\bm{k}) - r_{\bm{\mu'}}(\bm{k})| \le \sum_{b \le k_i}|\mu_{i,b} - \mu'_{i,b}|$.
\end{condition}

We define the reward gap $\Delta_{\bm{k}}=r_{\bm{\mu}}(\bm{k^*}) - r_{\bm{\mu}}(\bm{k})$ for all feasible action $\bm{k}$ satisfying $\sum_{i=1}^{m}k_i = B$, $0 \le k_i \le c_i $. 
For each base arm $(i,b)$, we define  $\Delta_{min}^{i,b}=\min_{\Delta_{\textbf{k}}>0, k_i \ge b}\Delta_{\textbf{k}}$ and $\Delta_{max}^{i,b}=\max_{\Delta_{\textbf{k}}>0, k_i \ge b}\Delta_{\textbf{k}}$. 
As a convention, if there is no action with $k_i \ge b$ such that $\Delta_{\bm{k}} > 0$, 
we define $\Delta_{\min}^{i,b}=\infty$ and $\Delta_{\max}^{i,b}=0$. 
Also, we define $\Delta_{\min} = \min_{(i,b) \in \cA}\Delta_{\min}^{i,b}$ and $\Delta_{\max} = \max_{(i,b) \in \cA}\Delta_{\max}^{i,b}$.
\begin{restatable}{theorem}{thmNonOVer} \label{thm: non over regret bound}
	CUCB-MG
	has the following distribution-dependent regret bound,
	\begin{equation*}\textstyle
	Reg_{\mu}(T) \le \sum_{(i,b) \in \cA}\frac{48 B \ln T}{\Delta_{\min}^{i,b}} + 2|\cA| + \frac{\pi^2}{3}|\cA|\Delta_{\max}.
	\end{equation*}
\end{restatable}

\begin{algorithm}[t]
	\caption{CUCB-MG: Combinatorial Upper Confidence Bound (CUCB) algorithm for non-overlapping \mulane, using
		Marginal Gains as the base arms}\label{alg: non-over online}
	\begin{algorithmic}[1]
	    \INPUT{Budget $B$, number of layers $m$, offline oracle DP.}
		\STATE For each $(i,b) \in \cA$, $T_{i,b} \leftarrow 0$, $\hat{\mu}_{i,b} \leftarrow 1$.\; 
		\FOR {$t=1,2,3,...,T $}
		\FOR{$(i,b)\in \cA$}
		\STATE $\rho_{i,b} \leftarrow \sqrt{3\ln t/(2T_{i,b})}$.  
		\STATE $\hat{\mu}_{i,b} = \min \{\hat{\mu}_{i,b} + \rho_{i,b}, 1\}$. 
		\ENDFOR
		\STATE $\bm{k} \leftarrow$ DP{$((\bar{\mu}_{i,b})_{i,b}$, $B$, $\bc$})
		\STATE Apply budget allocation $\bk$, which gives trajectories $\bm{X}\defeq(X_{i,1}, ..., X_{i,k_i})_{i\in[m]}$ as feedbacks. \label{line: non over online feedback}
		\STATE For $(i,b) \in \tau \defeq \{(i,b)\in \cA \lvert \, b \le k_i\}$, \\
		$Y_{i,b} = \sum_{u\in \bigcup_{j=1}^b \{X_{i,j}\}} \sigma_u - \sum_{u \in \bigcup_{j=1}^{b-1} \{X_{i,j}\}}\sigma_u$.
		\STATE For $(i,b) \in \tau$, update $T_{i,b}$ and $\hat{\mu}_{i,b}$: 
		\STATE $T_{i,b} = T_{i,b} + 1$, $\hat{\mu}_{i,b} = \hat{\mu}_{i,b} + (Y_{i,b} - \hat{\mu}_{i,b})/T_{i,b}.$
		\ENDFOR
		\end{algorithmic}
\end{algorithm}

\subsection{Regret Analysis for Non-overlapping \mulane}
In this section, we give the regret analysis for CUCB-MG, which applies the standard CUCB algorithm \cite{wang2017improving} to this setting.
\subsubsection{Proof Details}

Let $\cA$ denote the set containing all base arms, i.e.,  $\cA = \{(i, b) | i \in [m], b \in [c_i]\}$.
We add subscript $t$ to denote the value of a variable at the end of $t$. For example, $T_{i,b,t}$ denotes the total times of arm $(i,b) \in \cA$ is played at the end of round $t$.
Let us first introduce a definition of an unlikely event that $\hat{\mu}_{i,b,t-1}$ is not accurate as expected.

\begin{definition}
We say that the sampling is nice at the beginning of round t, if for every arm $(i,b) \in \cA$, $|\hat{\mu}_{i,b,t-1} - \mu_{i,b}| < \rho_{i,b,t}$, where $\rho_{i,b,t} = \sqrt{\frac{3\ln t}{2T_{i,b, t-1}}}$. Let $\mathcal{N}_t$ be such event.
\end{definition}

\begin{lemma} \label{lem: non over nice event}
	For each round $t \ge 1$, $Pr\{\neg N_t\} \le 2|\cA|t^{-2}$
\end{lemma}

\begin{proof}
	For each round $t \ge 1$, we have 
	\begin{align*}
	&Pr\{\neg \mathcal{N}_t\} \\
	&= Pr\{\exists (i,b)\in  \cA, |\hat{\mu}_{i,b,t-1} - \mu_{i,b}| \ge \sqrt{\frac{3\ln t}{2T_{i,b, t-1}}}\} \\
	&\le \sum_{(i,b)\in  \cA}Pr\{ |\hat{\mu}_{i,b,t-1} - \mu_{i,b}| \ge \sqrt{\frac{3\ln t}{2T_{i,b, t-1}}}\} \\
	&=\sum_{(i,b)\in  \cA}\sum_{k=1}^{t-1}Pr\{ T_{i,b,t-1} = k \\
	&\wedge |\hat{\mu}_{i,b,t-1} - \mu_{i,b}| \ge \sqrt{\frac{3\ln t}{2T_{i,b, t-1}}}\} \\
	&\le \sum_{(i,b)\in  \cA}\sum_{k=1}^{(t-1)} \frac{2}{t^{3}} \le 2|\cA|t^{-2}. \qquad
	\end{align*}
	
	When $T_{i,b,t-1} = k$, $\hat{\mu}_{i,b,t-1}$ is the average of $k$ i.i.d. random variables $Y_{i,b}^{[1]}, ..., Y_{i,b}^{[k]}$, where $Y_{i,b}^{[j]}$ is the Bernoulli random variable of arm $(i,b)$ when it is played for the $j$-th time during the execution. 
	That is, $\hat{\mu}_{i,b,t-1}=\sum_{j=1}^k Y_{i,b}^{[j]}/k$. 
	Note that the independence of $Y_{i,b}^{[1]}, ..., Y_{i,b}^{[k]}$ comes from the fact we select a new starting position following the starting distribution $\bm{\alpha}_i$ at the beginning of each round $t$. The last inequality uses the Hoeffding Inequaility (Fact~\ref{fact: hoef}).
\end{proof}

Then we can use monotonicity and 1-norm bounded smoothness propterty (Condition~\ref{cond: mono} and Condition~\ref{cond: 1-norm})to bound the reward gap $\bm{\Delta_{k_t}}=r_{\bm{\mu}}(\bm{k^*}) - r_{\bm{\mu}}(\bm{k_t})$ between the optimal action $\bm{k^*}$ and the action $\bm{k_t}=(k_{t,1}, ..., k_{t,m})$ selected by our algorithm $A$. 
To achieve this, we introduce a positive real number $M_{i,b}$ for each arm $(i,b)$ and define $M_{\bm{k_t}}=\max_{(i,b) \in S_t} M_{i,b}$, where $S_t=\{(i,b) \in \cA | b \le k_{t,i}\}$ and $|S_t|=B$. Define, 

\begin{equation*}
\kappa_T(M,s) = \begin{cases}
2, & \text{if $s=0$}, \\
2\sqrt{\frac{6\ln T}{s}}, & \text{if $1 \le s \le \ell_T{(M)}$}, \\
0, &\text{if $s \ge \ell_T(M) + 1$},
\end{cases}
\end{equation*}

where
\begin{equation*}
\ell_T(M) = \lfloor \frac{24 B^2\ln T}{M^2} \rfloor.
\end{equation*}

\begin{lemma}
	If $\{\Delta_{\bm{k}_t} \ge M_{\bm{k_t}} \}$ and $\mathcal{N}_t$ holds, we have
	\begin{equation} \label{lem:bound_delta}
	\Delta_{\bm{k}_t} \le \sum_{(i,b) \in S_t}\kappa_T(M_{i,b},T_{i,b,t-1}).
	\end{equation}	
\end{lemma}

\begin{proof}

    First, we can observe $\bar{\mu}_{i,b,t} \ge \mu_{i,b}$, for $(i,b) \in \cA$, when $\cN_t$ holds.
    This comes from the fact $\bar{\mu}_{i,b,t}=\min\{\hat{\mu}_{i,b,t-1}+\rho_{i,b,t}, 1\}$ and $|\hat{\mu}_{i,b,t-1} - \mu_{i,b}| < \rho_{i,b,t}$.

	Notice that the right hand of Inequality~(\ref{lem:bound_delta}) is non-negative, it is trivially satisfied when $\Delta_{\bm{k}_t} = 0$. We only need to consider $\Delta_{\bm{k}_t} > 0$. By $\mathcal{N}_t$ and Condition~\ref{cond: mono},
	We have the following inequalities,
	\begin{equation*}
	r_{\bar{\bm{\mu}}_t}(\bm{k_t}) \ge r_{\bar{\bm{\mu}}_t}(\bm{k^*}) \ge r_{{\bm{\mu}}}(\bm{k^*}) = \Delta_{\bm{k_t}} + r_{\bm{\mu}}(\bm{k_t}).
	\end{equation*}
	The first inequality is due to the \textit{oracle} outputs the optimal solution given parameters $\bar{\bm{\mu}}_t$ ,the second inequality is due to Condition~\ref{cond: mono} (monotonicity). 
	
	Then, by Condition~\ref{cond: 1-norm}, we have 
	\begin{equation*}
	\Delta_{\bm{k_t}} \le r_{\bar{\bm{\mu}}_t}(\bm{k_t}) - r_{{\bm{\mu}}}(\bm{k_t}) \le \sum_{(i,b) \in S_t}|\bar{{\mu}}_{i,b, t} - \mu_{i,b}|.
	\end{equation*}
	Next, we can bound $\Delta_{\bm{k_t}}$ by bounding $\bar{{\mu}}_{i,b, t} - \mu_{i,b}$ by applying a transformation. As $\{\Delta_{\bm{k}_t} \ge M_{\bm{k_t}} \}$ holds by assumption, so $\sum_{(i,b) \in S_t}|\bar{{\mu}}_{i,b, t} - \mu_{i,b}| \ge \Delta_{\bm{k_t}} \ge M_{\bm{k_t}}$. We have 
	\begin{align}\label{ieq: online non-over bound k}
	\Delta_{\bm{k_t}} &\le \sum_{(i,b) \in S_t}|\bar{{\mu}}_{i,b, t} - \mu_{i,b}| \nonumber\\
	& \le -M_{\bm{k_t}} + 2\sum_{(i,b) \in S_t}|\bar{{\mu}}_{i,b, t} - \mu_{i,b}| \nonumber\\
	& = 2\sum_{(i,b) \in S_t}(|\bar{{\mu}}_{i,b, t} - \mu_{i,b}| - \frac{M_{\bm{k_t}}}{2B}) \nonumber\\
	& \le 2\sum_{(i,b) \in S_t}(|\bar{{\mu}}_{i,b, t} - \mu_{i,b}| - \frac{M_{i,b}}{2B}) .
	\end{align}
	
	By $\mathcal{N}_t$, we have:
	\begin{align*}
	    |\bar{\mu}_{i,b,t} - \mu_{i,b}| - \frac{M_{i,b}}{2B} &\le  \min\{2\rho_{i,b,t}, 1\} - \frac{M_{i,b}}{2B} \\
	    &\le  \min \left\{2\sqrt{\frac{3\ln T}{2T_{i,b, t-1}}}, 1\right\} - \frac{M_{i,b}}{2B}.
	\end{align*}

	Now consider the following cases:
	
	\textbf{Case 1.} If $T_{i,b, t-1} \le \ell_T(M_{i,b})$, we have $\bar{\mu}_{i,b,t} - \mu_{i,b} - \frac{M_{i,b}}{2B} \le \min\left\{2\sqrt{\frac{3\ln T}{2T_{i,b, t-1}}}, 1\right\} \le \frac{1}{2} \kappa_T(M_{i,b}, T_{i,b,t-1})$. 
	
	\textbf{Case 2.} If  $T_{i, b, t-1} \ge \ell_T(M_{i,b,t})+1$, then $2\sqrt{\frac{3\ln T}{2T_{i,b, t-1}}} \le \frac{M_{i,b}} {2B}$, so $\bar{\mu}_{i,b,t} - \mu_{i,b} - \frac{M_{i,b}}{2B} \le 0 = \frac{1}{2} \kappa_T(M_{i,b}, T_{i,b,t-1})$. 
	
	Combine these two cases and inequality~(\ref{ieq: online non-over bound k}), we have the following inequality, which concludes the lemma.
	\begin{equation*}
	\Delta_{\bm{k}_t} \le \sum_{(i,b) \in S_t}\kappa_T(M_{i,b},T_{i,b,t-1}).
	\end{equation*}
\end{proof}

\thmNonOVer*
\begin{proof}[Proof of Theorem~\ref{thm: non over regret bound}]
	By above lemmas, we have
	\begin{align*}
	&\sum_{t=1}^T \mathbb{I}\left({\{\Delta_{\bm{k_t}} \ge M_{\bm{k_t}}\} \wedge \mathcal{N}_t}\right) \Delta_{\bm{k_t}} \\
	&\le \sum_{t=1}^T \sum_{(i,b) :(i,b) \in S_t}\kappa_T(M_{i,b},T_{i,b,t-1}) \\
	&=\sum_{(i,b) \in \cA}\sum_{t'\in \{ t| t\in [T], k_{t,i} \ge b \}} \kappa_T(M_{i,b},T_{i,b,t'-1}) \\
	&\le \sum_{(i,b) \in \cA}\sum_{s=0}^{T_{i,b,T}} \kappa_T(M_{i,b},s)\\
	&\le 2|\cA| + \sum_{(i,b) \in \cA}\sum_{s=1}^{\ell_T(M_{i,b})} 2\sqrt{6\frac{\ln T}{s}}\\
	&\le 2|\cA| + \sum_{(i,b) \in \cA} \int_{s=0}^{\ell_T(M_{i,b})}  2\sqrt{6\frac{\ln T}{s}} ds \\
	&\le 2|\cA| + \sum_{(i,b) \in \cA}   4\sqrt{6 \ln T \ell_T(M_{i,b})} \\
	&\le 2|\cA| + \sum_{(i,b) \in \cA} \frac{48 B \ln T}{M_{i,b}}.
	\end{align*}
	Let us set $M_{i,b} = \Delta^{i,b}_{\min}$ and thus the event \{$\Delta_{\bm{k_t}} \ge M_{\bm{k_t}}$\} always holds. We have 
	\begin{equation*}
	\begin{split}
	&\sum_{t=1}^T \mathbb{I}\left(\mathcal{N}_t\ \wedge \{\Delta_{\bm{k_t}} \ge M_{\bm{k_t}}\}\right) \Delta_{\bm{k_t}} \\
	&\le 2|\cA| + \sum_{(i,b) \in \cA} \frac{48 B \ln T}{\Delta_{\min}^{i,b}}.
	\end{split}
	\end{equation*}
	
	By Lem.~\ref{lem: non over nice event}, $Pr\{\neg N_t\} \le 2|\cA|t^{-2}$, then we have 
	\begin{equation*}
	\E\left[\sum_{t=1}^T \mathbb{I}\left( \neg \mathcal{N}_t\right) \Delta_{\bm{k_t}}\right] \le \sum_{t=1}^T 2|\cA|t^{-2} \Delta_{\max} \le \frac{\pi^2}{3}|\cA| \Delta_{\max}.
	\end{equation*}
	
	By our choice of $M_{i,b}$, 
	
	\begin{equation*}
	\sum_{t=1}^T \mathbb{I}\left(\Delta_{\bm{k_t}} < M_{\bm{k_t}}\right) \Delta_{\bm{k_t}} = 0 .
	\end{equation*}
	
	By the definition of the regret,
	
	\begin{align*}
	Reg_{\bm{\mu}}(T) &= \E\left[\sum_{t=1}^T \Delta_{\bm{k_t}}\right] \\
	&\le \E\bigg[\sum_{t=1}^T \bigg(\mathbb{I}\big( \neg \mathcal{N}_t\big) + \mathbb{I}\big(\Delta_{\bm{k_t}} < M_{\bm{k_t}}\big) \\
	&+ \mathbb{I}\big(\mathcal{N}_t\ \wedge \{\Delta_{\bm{k_t}} \ge M_{\bm{k_t}}\}\big)\bigg) \Delta_{\bm{k_t}}\bigg] \\
	&\le \sum_{(i,b) \in \cA}\frac{48 B \ln T}{\Delta_{min}^{i,b}} + 2|\cA| + \frac{\pi^2}{3}|\cA|\Delta_{max}.
	\end{align*}
	
	For distribution independent bound, we take $M_{i,b} = M =\sqrt{48 B|\cA| \ln T /T}$, then we have
	$\sum_{t=1}^T \mathbb{I}\{\Delta_{\bm{k_t}} \le M_{\bm{k_t}}\} \le TM$.
	
	Then 
	\begin{align*}
	Reg_{\mu}(T) &\le \sum_{(i,b) \in \cA}\frac{48 B \ln T}{M_{i,b}} + 2|\cA| + \frac{\pi^2}{3}|\cA|\Delta_{max} + TM \\
	&\le \frac{48 B|\cA| \ln T}{M} + 2|\cA| + \frac{\pi^2}{3}|\cA|\Delta_{max} + TM \\
	&= 2\sqrt{48 B |\cA|T\ln T} + \frac{\pi^2}{3}|\cA|\Delta_{max} + 2|\cA| \\
	&\le 14\sqrt{B|\cA|T\ln T} + \frac{\pi^2}{3}|\cA|\Delta_{max} + 2|\cA|.
	\end{align*}
\end{proof}

\section{Supplemental Experiments}\label{apdx:exp}
\subsection{Experimental Results for Stationary Distributions}\label{appendix: expriment for stationary distributions}
In this section, we show experimental results with explorers starting from stationary distributions in Fig.~\ref{fig: algorithms for stationary distributions}. Specifically, each explorer $W_i$ starts from the node $u$ in layer $L_i$ with probability proportional to the out-degree of $u$.
For offline overlapping case, the total weights of unique nodes visited for BEG and MG are the same, and outperforms two baselines, which accords with our theoretical analysis.
Similar to the Sec.~\ref{sec: experiments}, BEG and MG are empirically close to the optimal solution.
For offline non-overlapping case, DP and MG give us the same \textit{optimal} solution.
For the online learning algorithms, note that we still use BEG rather than MG as the offline oracle, this is due to the UCB values do not have the diminishing return properties.
The same applies for the non-overlapping case, where DP (instead of MG) is used as the offline oracle.
The trend of the curves and the analysis are similar to that in the Sec.~\ref{sec: experiments}.
\begin{figure}[t]
	\centering
	\begin{subfigure}[b]{0.235\textwidth}
		\centering
		\includegraphics[width=\textwidth]{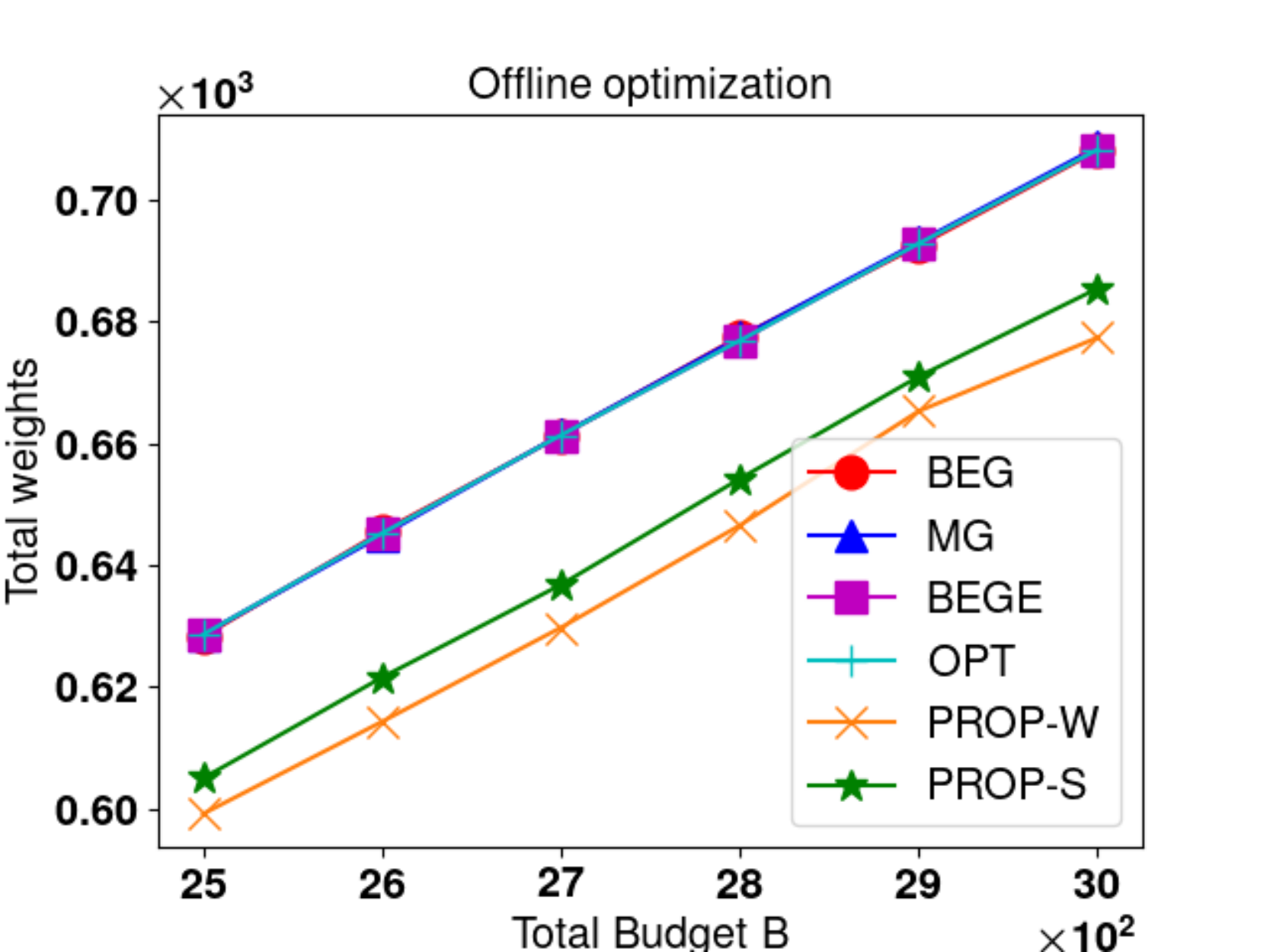}
		\caption{Offline, overlapping.}
		\label{fig: offline over stationary}
	\end{subfigure}
	\hfill
	\begin{subfigure}[b]{0.235\textwidth}
		\centering
		\includegraphics[width=\textwidth]{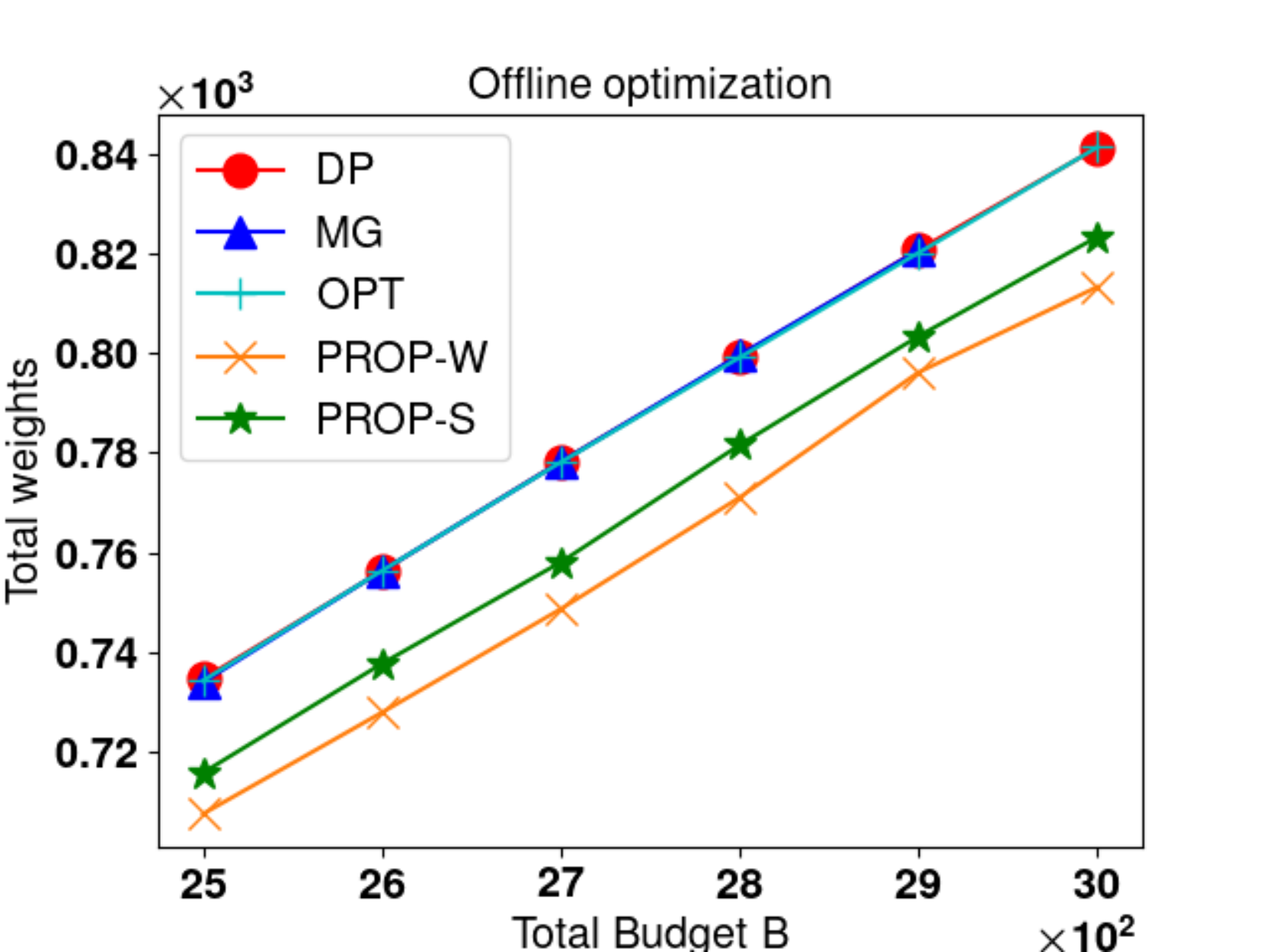}
		\caption{Offline, non-overlapping.}
		\label{fig: offline non over stationary}
	\end{subfigure}
	\begin{subfigure}[b]{0.235\textwidth}
		\centering
		\includegraphics[width=\textwidth]{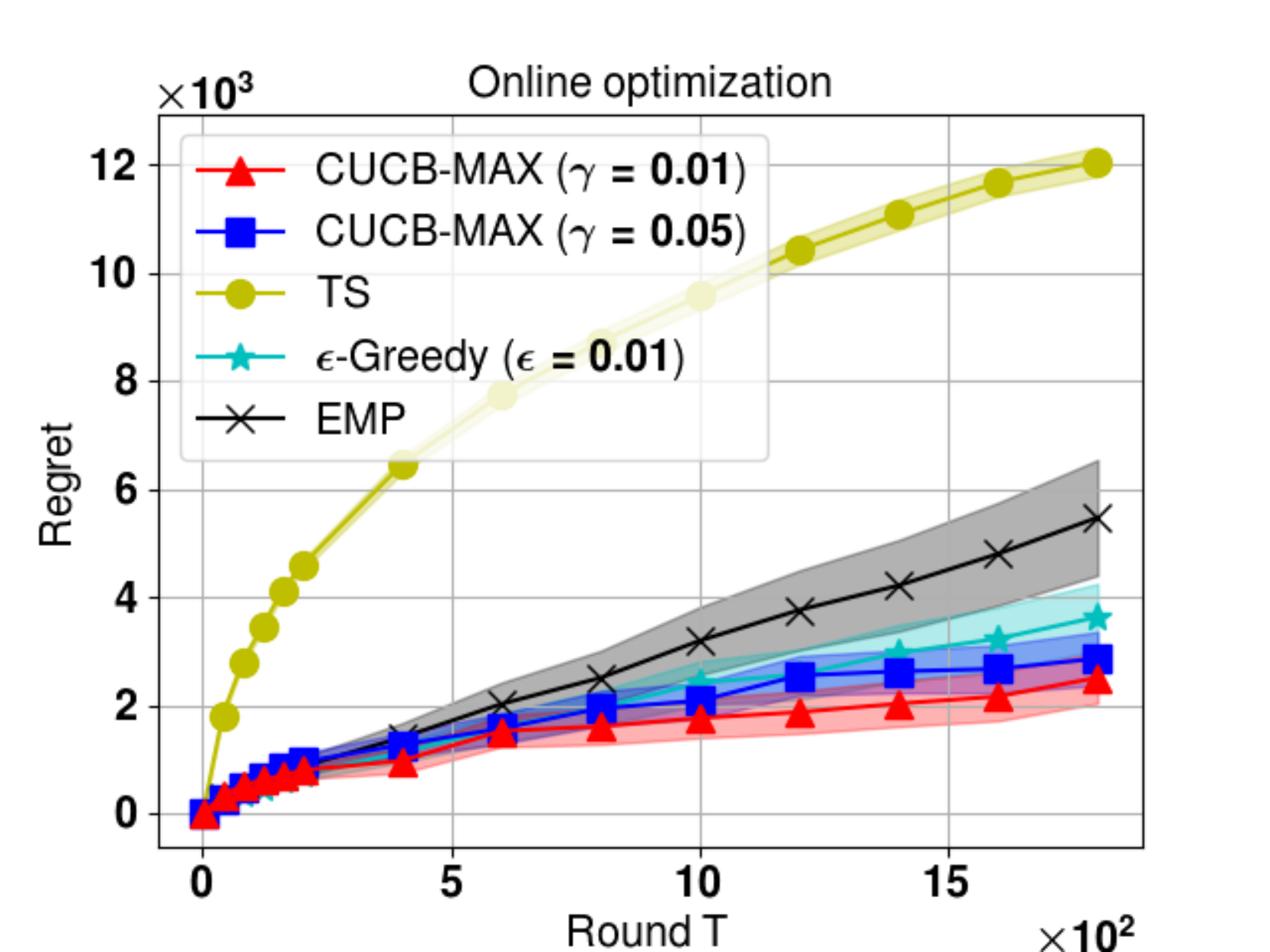}
		\caption{Online, overlapping.}
		\label{fig: online over stationary 3000}
	\end{subfigure}
	\hfill
	\begin{subfigure}[b]{0.235\textwidth}
		\centering
		\includegraphics[width=\textwidth]{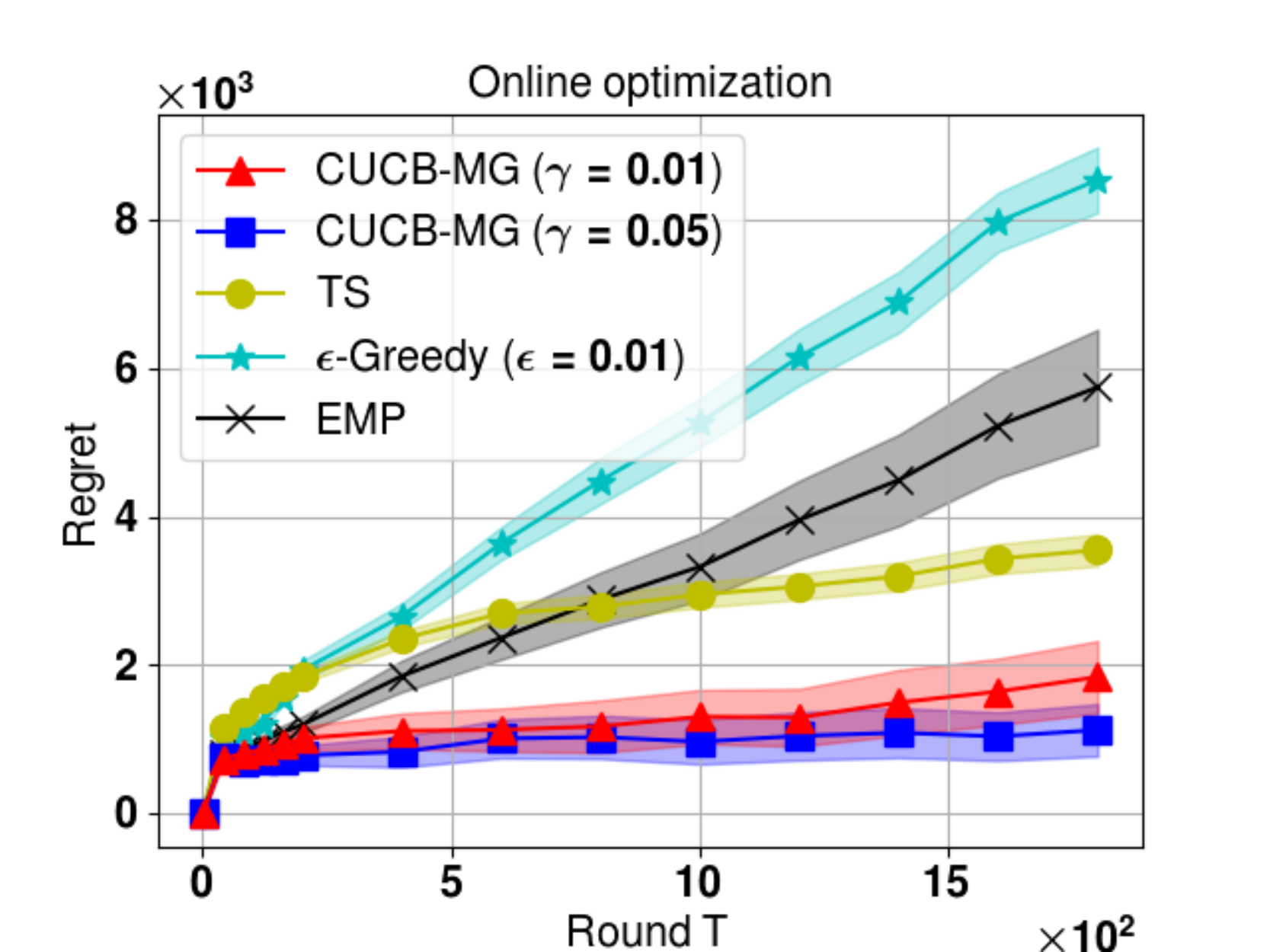}
		\caption{Online, non-overlapping}
		\label{fig: online non over stationary 3000}
	\end{subfigure}
	\caption{Explorers start with stationary distributions. Above: total weights of unique nodes visited given by different offline algorithms. Below: regret for different online algorithms when $B=3000$.}\label{fig: algorithms for stationary distributions}
\end{figure}

\subsection{Experimental Results for Online Learning Algorithms with Different Budgets $B$}\label{appendix: expriment for diff budgets}
As shown in Fig.~\ref{fig: online algorithms 2000.} and Fig.~\ref{fig: online algorithms 4000}, although there are some fluctuations, the trend of the curves and the analysis are almost the same as that when $B=3000$, which shows our experimental results are consistent for different budgets.  
\begin{figure}[t]
     \centering
     \begin{subfigure}[b]{0.235\textwidth}
         \centering
         \includegraphics[width=\textwidth]{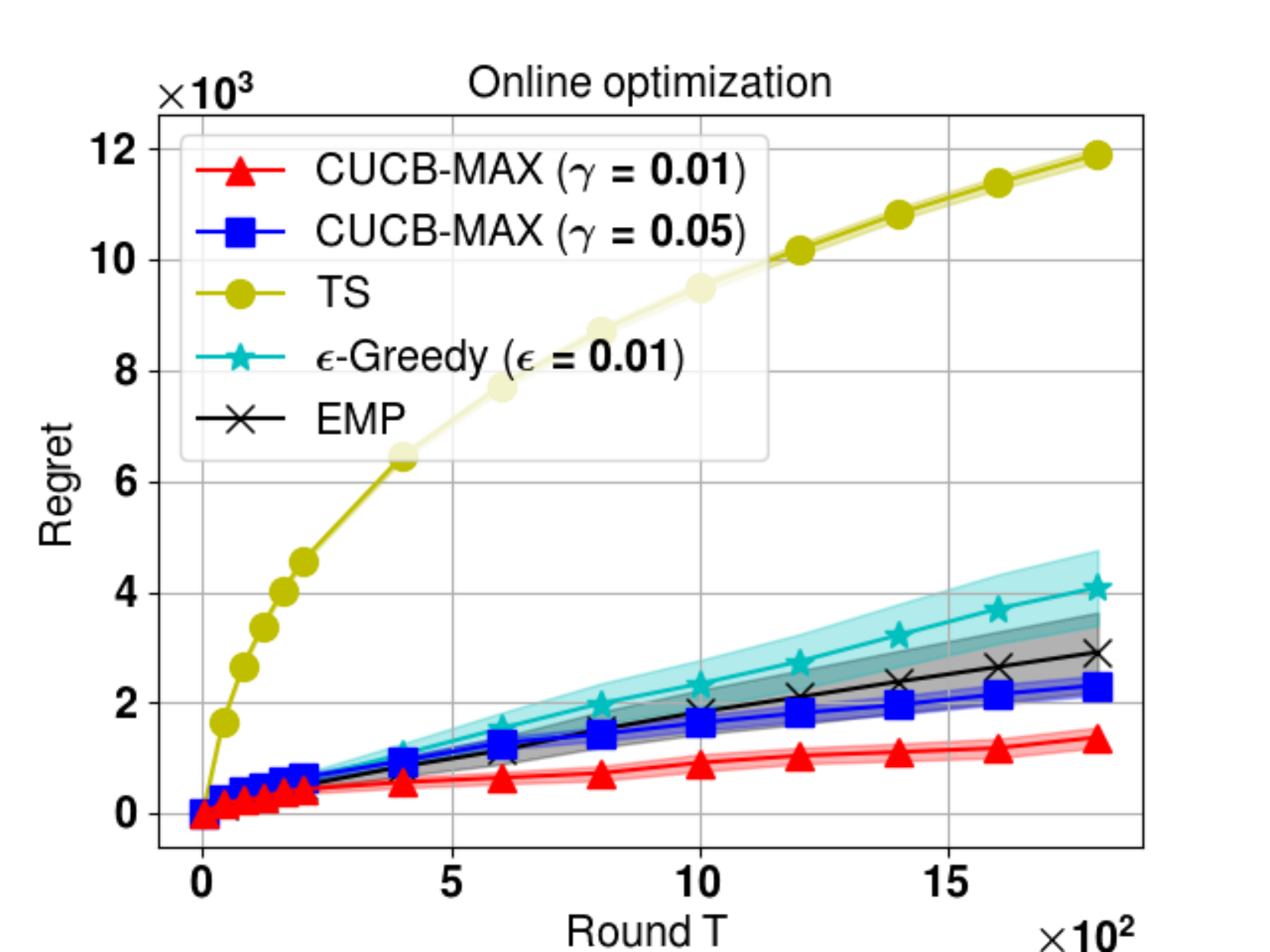}
         \caption{Overlapping, fix node.}
         \label{fig: online over fix 2000}
     \end{subfigure}
     \hfill
     \begin{subfigure}[b]{0.235\textwidth}
         \centering
         \includegraphics[width=\textwidth]{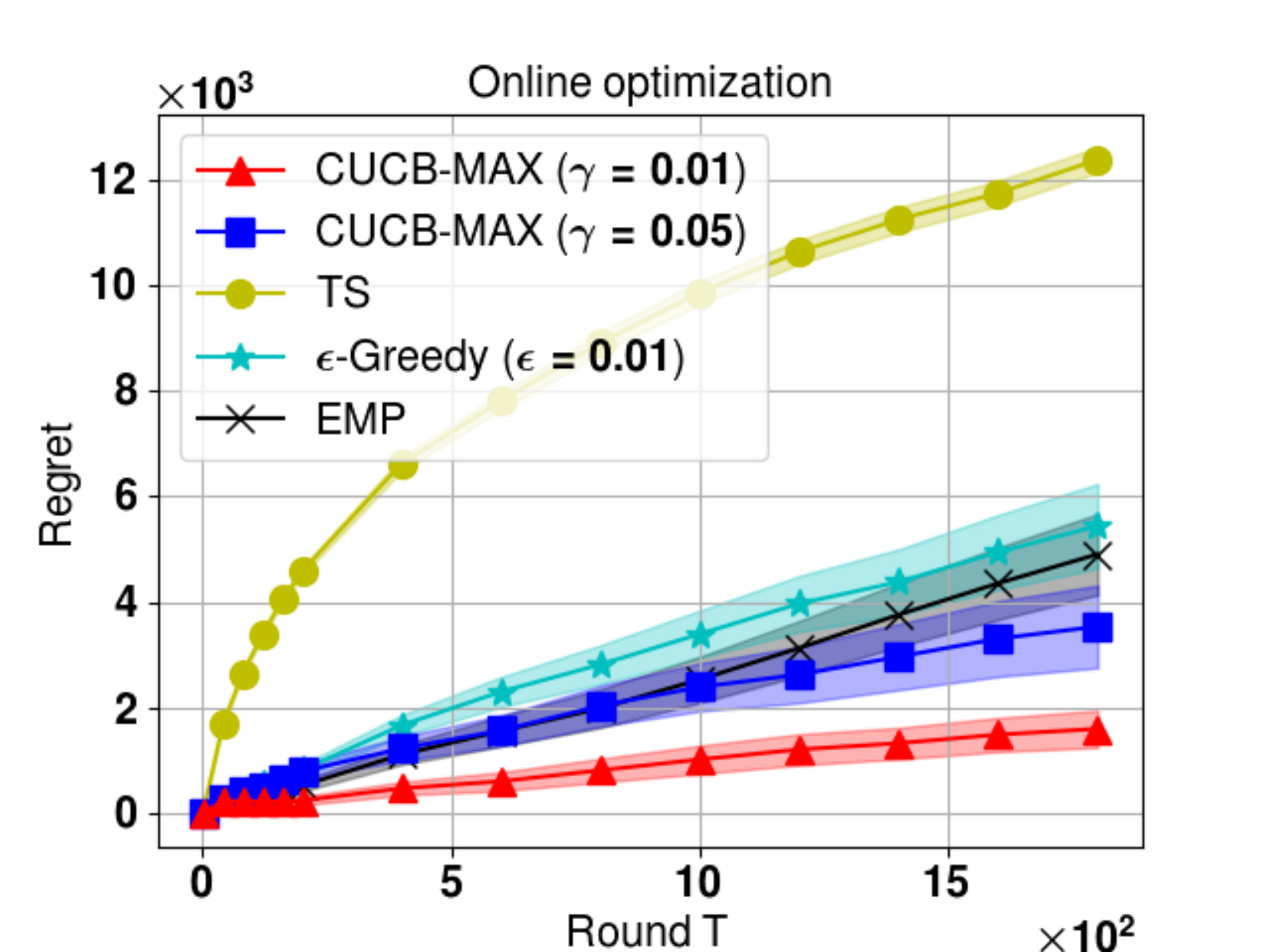}
         \caption{Overlapping, stationary.}
         \label{fig: online over stationary 2000}
     \end{subfigure}
     \begin{subfigure}[b]{0.235\textwidth}
         \centering
         \includegraphics[width=\textwidth]{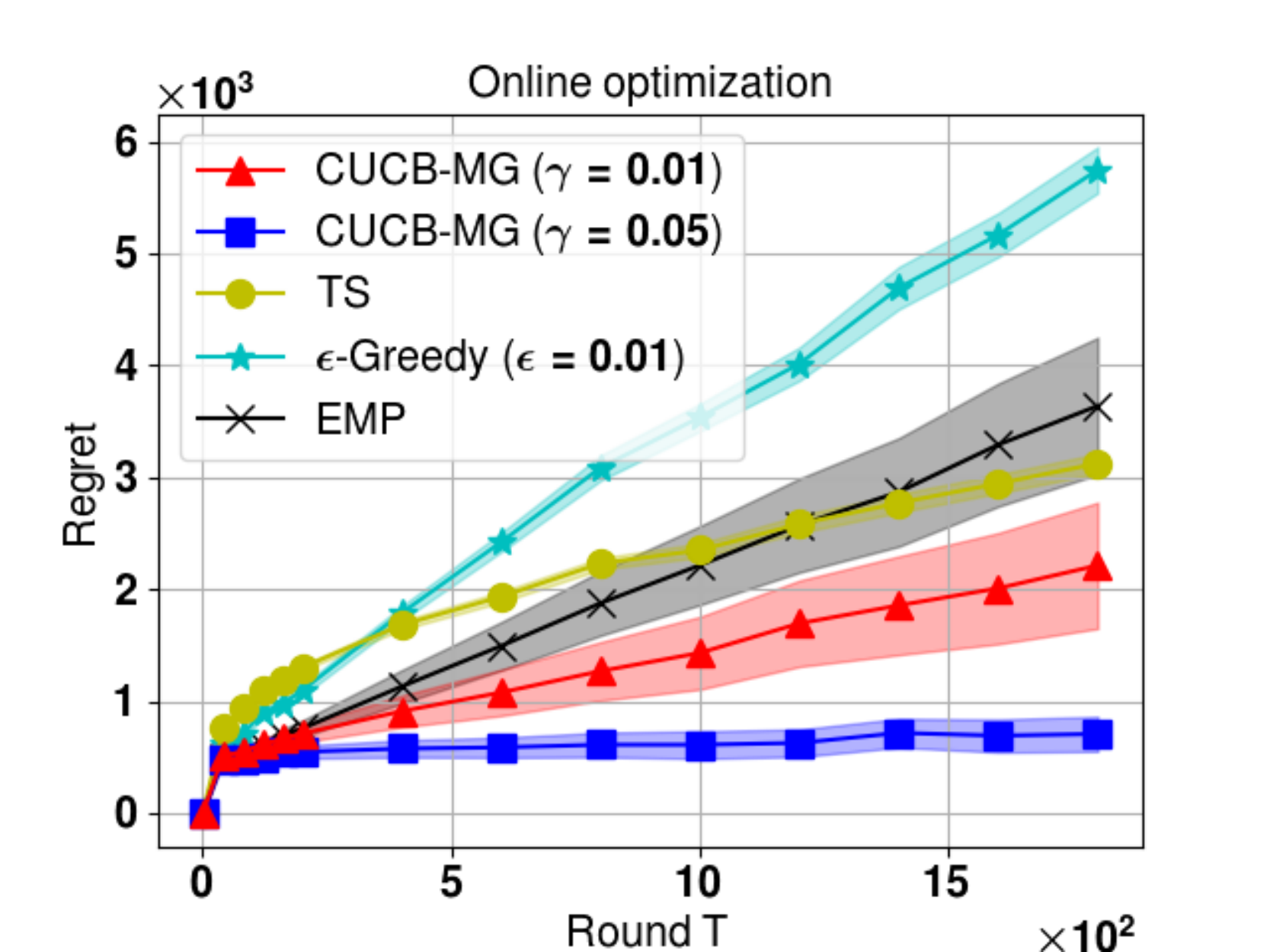}
         \caption{Non-overlapping, fix node.}
         \label{fig: online non over fix 2000}
     \end{subfigure}
     \hfill
     \begin{subfigure}[b]{0.235\textwidth}
         \centering
         \includegraphics[width=\textwidth]{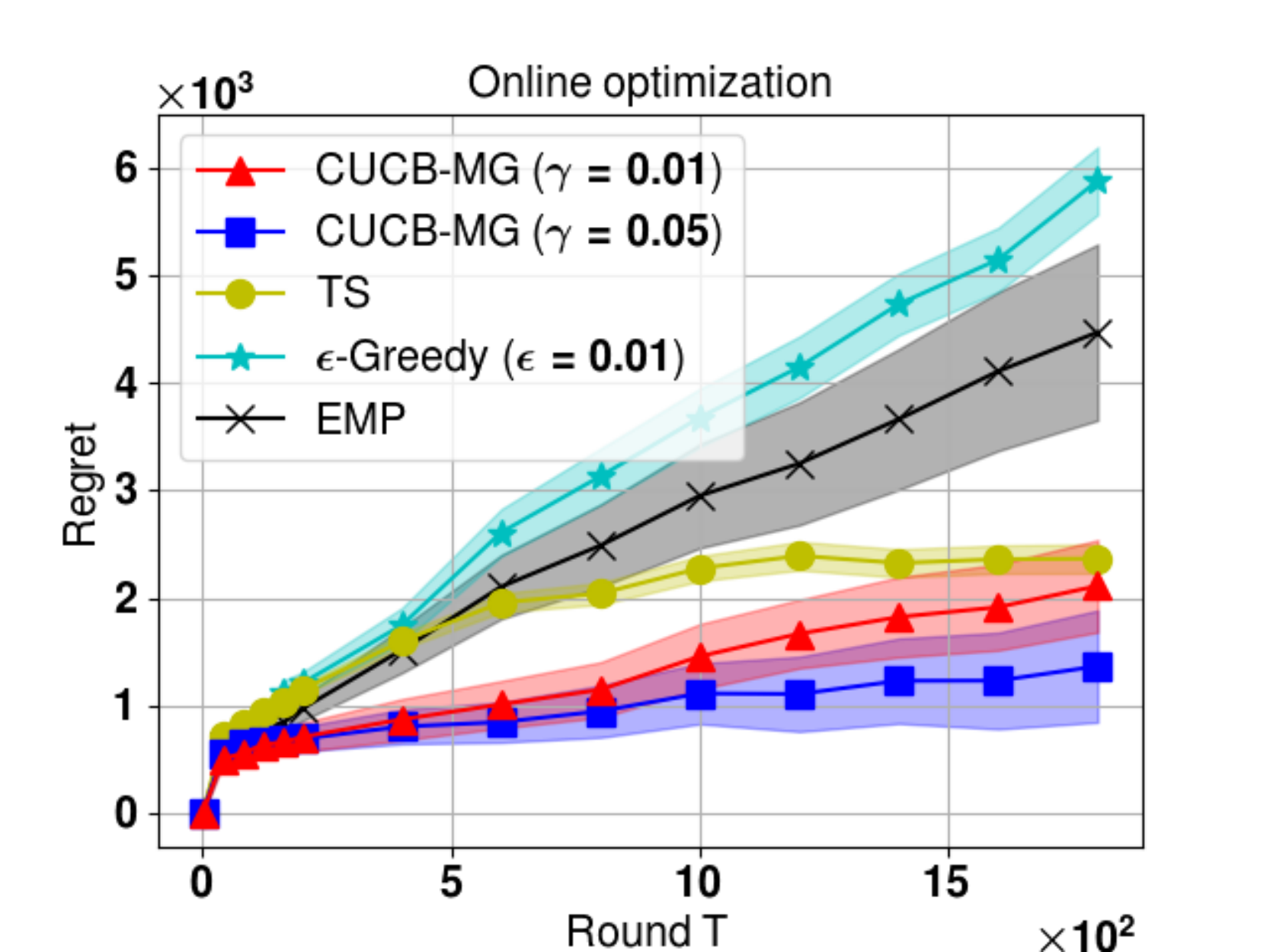}
         \caption{Non-overlapping, stationary.}
         \label{fig: online non over stationary 2000}
     \end{subfigure}
        \caption{Regret for different online algorithms, when total budget $B=2000$.}\label{fig: online algorithms 2000.}
\end{figure}

\begin{figure}[t]
     \centering
     \begin{subfigure}[b]{0.235\textwidth}
         \centering
         \includegraphics[width=\textwidth]{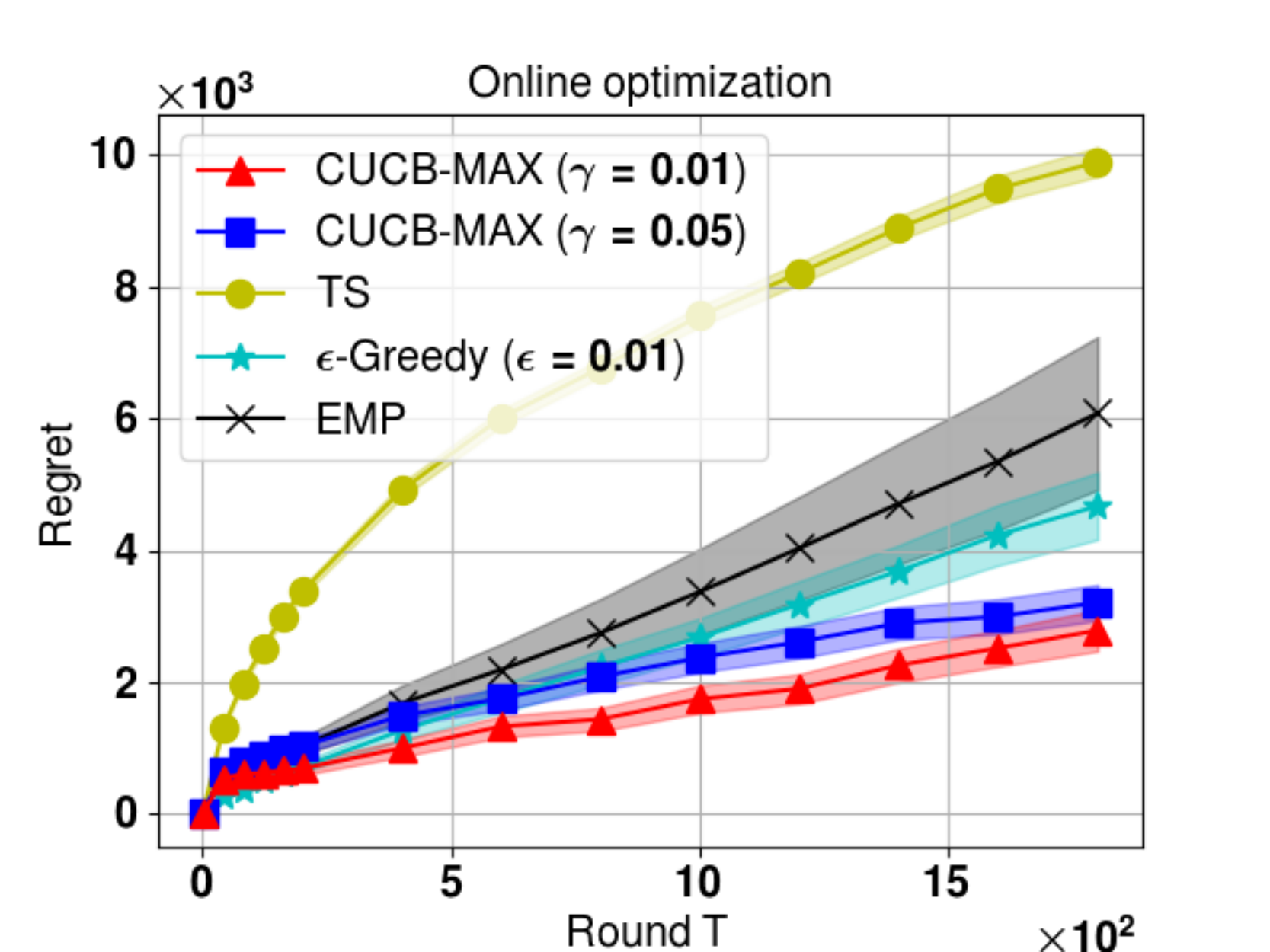}
         \caption{Overlapping, fix node.}
         \label{fig: online over fix 4000}
     \end{subfigure}
     \hfill
     \begin{subfigure}[b]{0.235\textwidth}
         \centering
         \includegraphics[width=\textwidth]{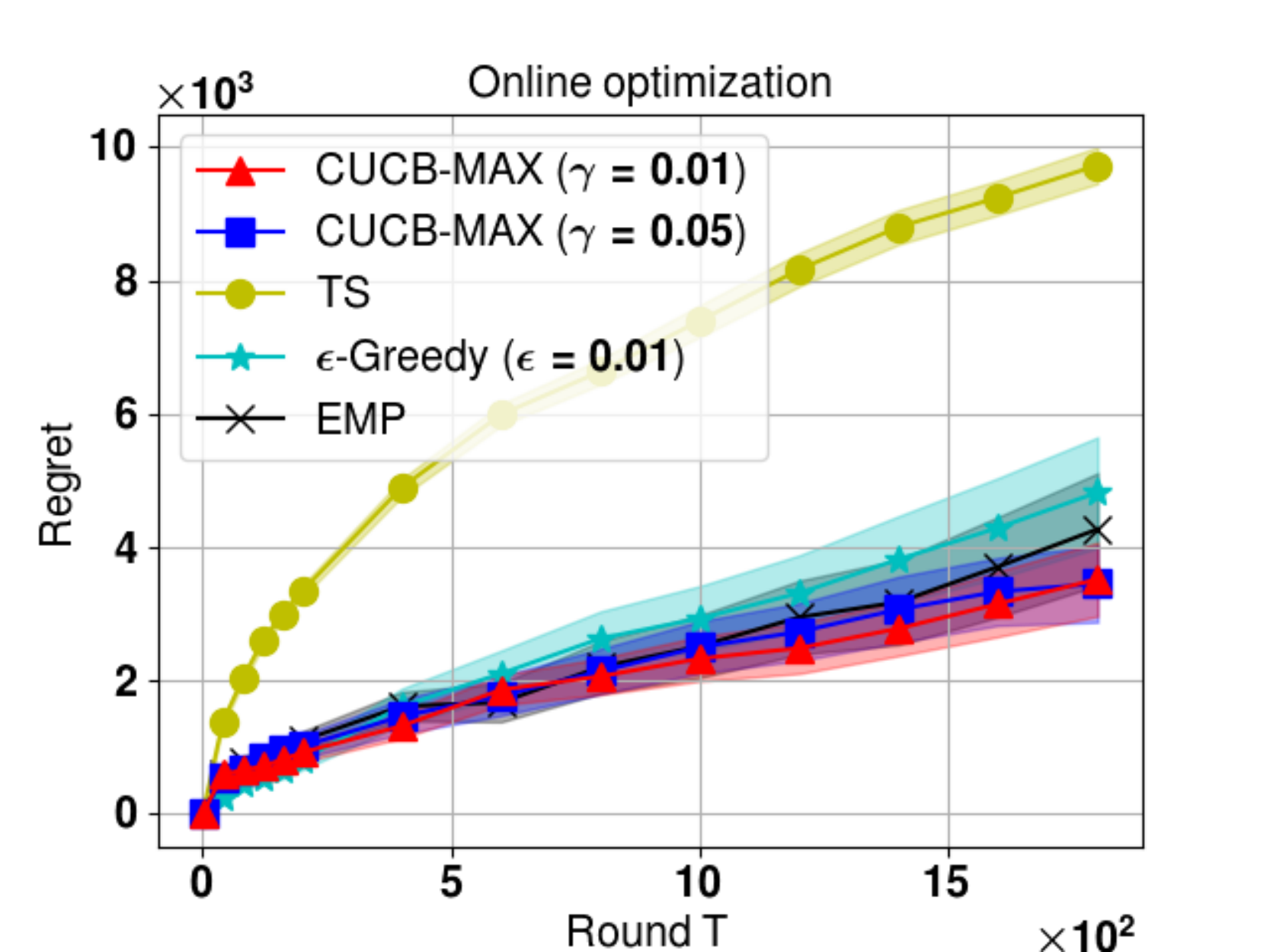}
         \caption{Overlapping, stationary.}
         \label{fig: online over stationary 4000}
     \end{subfigure}
     \begin{subfigure}[b]{0.235\textwidth}
         \centering
         \includegraphics[width=\textwidth]{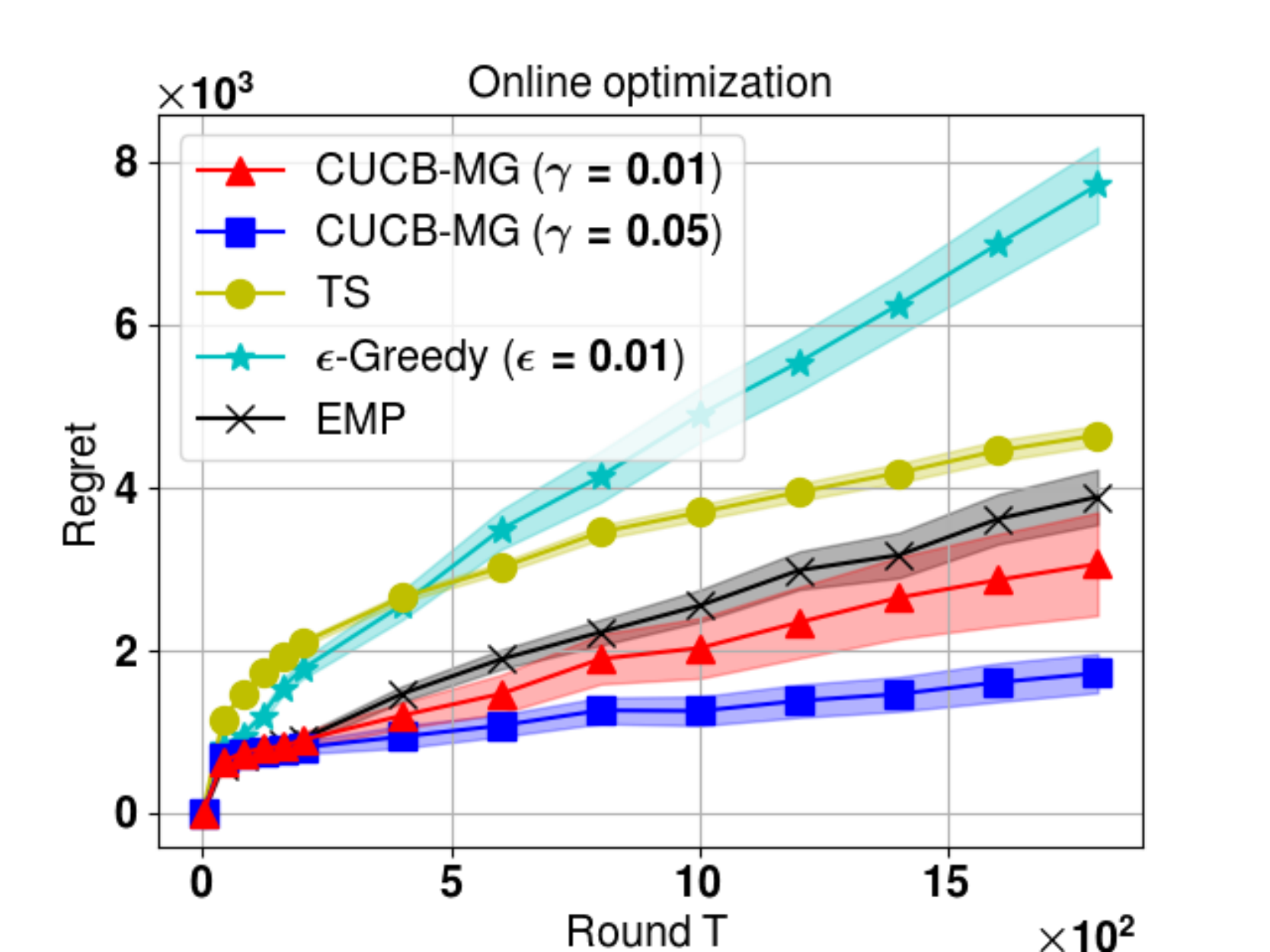}
         \caption{Non-overlapping, fix node.}
         \label{fig: online non over fix 4000}
     \end{subfigure}
     \hfill
     \begin{subfigure}[b]{0.235\textwidth}
         \centering
         \includegraphics[width=\textwidth]{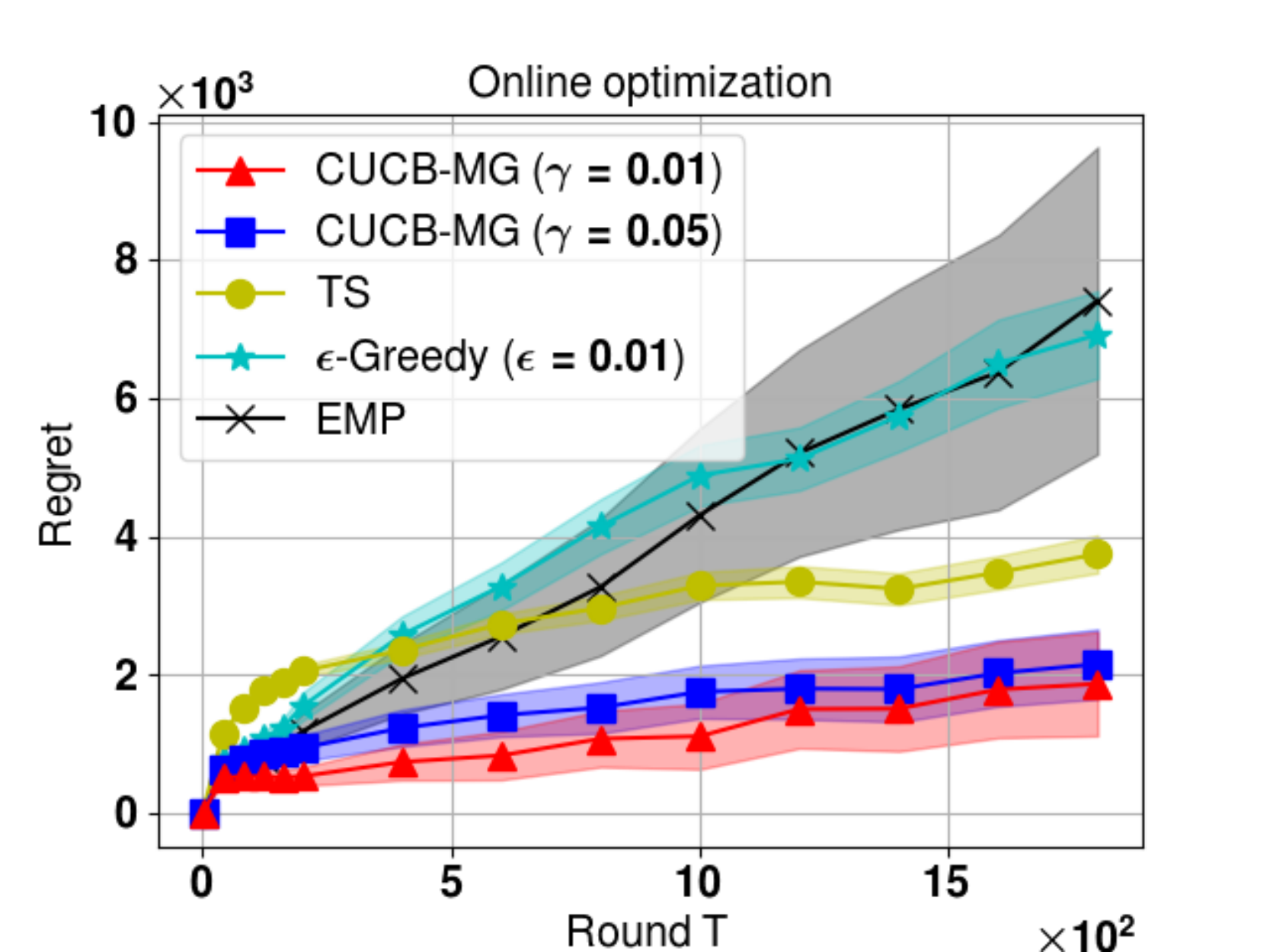}
         \caption{Non-overlapping, stationary.}
         \label{fig: online non over stationary 4000}
     \end{subfigure}
        \caption{Regret for different online algorithms, when total budget $B=4000$.}\label{fig: online algorithms 4000}
\end{figure}
}
\end{document}